%% file: main_paper_neurips_2026.tex
\documentclass{article}

\PassOptionsToPackage{numbers, compress}{natbib}

\usepackage[preprint]{neurips_2026}

\usepackage[utf8]{inputenc} %
\usepackage[T1]{fontenc}    %
\usepackage{hyperref}       %
\usepackage{url}            %
\usepackage{booktabs}       %
\usepackage{amsfonts}       %
\usepackage{nicefrac}       %
\usepackage{microtype}      %
\usepackage{xcolor}         %

 \usepackage[normalem]{ulem}

\usepackage{algorithm}
\usepackage{algorithmic}
\usepackage{amsfonts}       %
\usepackage{amsmath}
\usepackage{amsthm}
\usepackage{amssymb}
\usepackage{graphicx}
\usepackage{subcaption} %
\usepackage{mathtools} %
\usepackage{dsfont} %
\usepackage{array} %
\usepackage{longtable} %

\include{pramble}

\title{Non-Asymptotic Best Policy Identification Guarantees in Online Reinforcement Learning}

\author{%
  Joseph Lazzaro
    \\
  Imperial College London\\
  \texttt{jl22718@ic.ac.uk} \\
  \And
  Alessio Russo \\
  Boston University\\
  \texttt{arusso2@bu.edu} \\
  \AND
  Aldo Pacchiano\\
Boston University\\
Broad Institute of MIT and Harvard\\
  \texttt{pacchian@bu.edu} \\
}

\begin{document}

\maketitle

\begin{abstract}
 In this work we study the Best Policy Identification (BPI) problem in online, tabular Reinforcement Learning. This is an active sequential hypothesis testing problem in which the learner's objective is to identify an optimal policy in a Markov Decision Process (MDP) with high confidence, while minimizing the expected sample complexity to do so. We consider an online setting with deterministic rewards, where the agent must strategically navigate through the MDP in order to effectively explore. Previous works in the literature have provided asymptotically optimal methods for BPI, such as the Navigate and Stop (NaS) algorithm and its variants, however existing analysis remains asymptotic. In this work, we fill that gap by providing the first non-asymptotic sample complexity guarantees for NaS, showing that its sample complexity depends not only on the characteristic time, but also on the connectivity of the underlying MDP, the curvature of the optimal characteristic time, and other instance-dependent quantities. We identify these additional attributes and make explicit their contributions to the overall sample complexity.  
\end{abstract}

\section{Introduction}
In Reinforcement Learning \citep{sutton2018reinforcement}, a fundamental question is how efficiently one can identify an optimal control strategy when interacting with an unknown stochastic environment. While most classical work focuses on regret minimization \citep{06a276ef-3e35-33d0-912a-626bc0841b5e,NEURIPS2018_d693d554,boone2025regret}, regret is not always the most relevant performance criterion. In many practical scenarios, including robotics prototyping, model selection for planning, and simulation-based decision support, the reward collected during learning is irrelevant; rather, the goal is to \emph{gather information as quickly as possible and stop once a correct decision can be made with high confidence}. This motivates the study of \emph{active pure exploration} in Markov decision processes (MDPs) and, in particular, the problem of \emph{Best Policy Identification} (BPI) \citep{al2021navigating}.

In BPI, the learner interacts with a discounted MDP and sequentially chooses actions, observes transitions and rewards, and aims to identify an optimal policy with probability at least $1 - \delta$ while minimizing the number of interaction steps. This problem was formalized in early PAC-RL work and has recently received renewed attention, culminating in the first instance-dependent lower bounds and algorithms for discounted MDPs under navigation constraints \citep{al2021navigating, russo2023model,russomulti,russo2025adaptive,taupin2023best}. Crucially, unlike the generative model setting, where an algorithm may sample arbitrary state--action pairs at will, the online navigation setting restricts sampling to a single evolving system trajectory, making the structure of optimal exploration policies substantially richer.

The paper of \citet{al2021navigating} investigates BPI under navigation constraints and provides:
(i) an instance-specific lower bound on sample complexity, formalized as a game between the algorithm and nature,
(ii) a sampling rule ensuring that the long-run empirical frequencies match an optimal oracle allocation, and
(iii) an asymptotic sample-complexity analysis in the limit $\delta \to 0$.
Their \emph{Navigate-and-Stop} (NaS) algorithm achieves asymptotic optimality up to constant factors and constitutes the first instance-specific guarantee for the online BPI setting. However, current results for online BPI remain entirely \emph{asymptotic}: the theory guarantees optimal scaling as $\delta \to 0$ but provides no understanding of finite-confidence behavior, the effect of stopping thresholds, or the transient instability of sampling distributions when data are scarce.

The need for non-asymptotic analysis has been recently established in pure exploration bandits. In the classical Best-Arm Identification (BAI) problem, Track-and-Stop (TAS) is known to be asymptotically optimal \citep{garivier2016optimal}. Yet, despite a decade of work, the finite-$\delta$ properties of TAS long remained elusive. Game theoretic approaches with optimistic extensions of TAS were proposed to provide non-asymptotic guarantees \citep{degenne2019non}. Very recent progress \citep{poiani2025non} has also illustrated that with careful technical approaches to the analysis of TaS, it is possible to provide non-asymptotic sample complexity guarantees without needing to modify the original algorithm. These developments in BAI raise a natural and important question for BPI:

\emph{Can we characterize, and control, the non-asymptotic sample complexity of the Navigate-and-Stop algorithm for MDPs?}

Indeed, while bandit pure-exploration methods now admit both asymptotic and finite-confidence guarantees \citep{poiani2025non, degenne2019non}, no such guarantees are known for BPI. The online MDP setting is far more challenging, combining bandit-style change-of-measure arguments with (i) navigation constraints, (ii) non-homogeneous Markov dynamics induced by adaptive policies, and (iii) the need to maintain forced exploration so as to avoid state-space collapse. As a result, controlling finite-time deviations of empirical transition kernels and stopping statistics is substantially more delicate than in bandits.

\paragraph{Contributions.}
In this paper, we close this gap by providing the \emph{first non-asymptotic sample-complexity guarantees for Best Policy Identification in discounted MDPs under navigation constraints}:

\textbf{1.} First, we prove the first finite-confidence sample-complexity upper bound for a NaS-type algorithm in single-trajectory discounted tabular MDPs. Under our assumptions, the algorithm is $\delta$-correct and satisfies an implicit non-asymptotic bound of the form
\[
\mathbb{E}[\tau] \leq \inf\{t:b(t,\delta)\leq (t-\sqrt{t})T(M)^{-1}-\ell(t,M)\}+O(1),
\]
where $\ell(t,M)=o(t)$ is an error term driven by the finite-confidence costs of estimating the transition probabilities, navigating the MDP under a single trajectory, and tracking the optimal state-action allocation induced by the estimated model.

\textbf{2.}  Under linear sharpness of the characteristic time, we obtain an explicit finite-confidence expansion
\[
\mathbb{E}[\tau]\leq B_\delta+\tilde O\left(\sqrt{B_\delta}+T(M) \sum_i A_i B_\delta^{r_i}\right),\quad B_\delta=T(M)\log(1/\delta),\quad r_i<1.
\]
The constants expose the dependence on connectivity, mixing, forced exploration, and allocation geometry.

\textbf{3.}  Third, we develop the technical ingredients needed for this result: finite-time tracking of empirical visitation frequencies under adaptive non-homogeneous Markov dynamics, stability of optimal-allocation selectors under transition perturbations, and a sharpness analysis showing that the objective can be arbitrarily flat in continuum alternative classes.

\section{Problem Setting}

In this section, we introduce the problem setting, concepts and notation used throughout the paper.

\textbf{Markov Decision Processes (MDPs).} 
We consider discounted MDPs denoted by $M=(\statespace, \actionspace,P,r,\gamma)$, where $\statespace$ is the finite state space (of size $S=|\statespace|$), $\actionspace$ is the finite action space (of size $A=|\actionspace|$), $P:\statespace\times \actionspace \to \Delta(\statespace)$ is the transition function, which maps state-action pairs to a distribution over states, $r: \statespace\times \actionspace\to\mathbb{R}$ defines a mapping from state-action pairs to deterministic real rewards,  and $\gamma\in(0,1)$ is the discount factor. 
For a  reward $r$, we define the discounted value of a Markov policy $\pi:\statespace \to \Delta(\actionspace)$ to be $V^\pi(s)=\mathbb{E}^\pi[\sum_{t\geq0}\gamma^t r(s_t,a_t)|s_0=s]$, where $s_{t+1}\sim P(\cdot|s_t,a_t)$ and $a_t\sim \pi(\cdot|s_t)$. 
We also let $V^\star(s) = \max_\pi V^\pi(s)$ be the optimal value in $s$ and, for MDPs with a unique optimal policy, we indicate it by  $\pi^\star_o(s) = \arg\max_\pi V^\pi(s)$. Similarly, we also define the action-value function $Q^\pi(s,a)=r(s,a)+\gamma \mathbb{E}_{s'\sim P(\cdot|s,a)}[V^\pi(s')]$. For an MDP $M$, we indicate by $\Pi^\star(M)=\{\pi \in \Pi_d: \|V^\star - V^\pi\|_\infty=0\}$ the set of optimal policies in $M$ and $\Pi_d$ is the set of deterministic policies.

\textbf{Best Policy Identification with Fixed Confidence}
Best Policy Identification consists of finding an optimal policy in an MDP using the least number of samples with a given confidence level for a specific reward function \citep{al2021adaptive,al2021navigating}.
 The BPI objective can be formalized in the $\delta$-Probably-Correct ($\delta$-PC) framework, where the learner interacts with the MDP until they can output an optimal policy with confidence $1-\delta$. Formally, an algorithm in our $\delta$-PC setting consists of: (i) an exploration strategy (a.k.a. sampling rule), (ii) a stopping rule $\tau$ and (iii) an estimated optimal policy $\hat\pi_{\tau}$. At each time step $t$, the algorithm observes a state $s_t$, it selects an action $a_t$, and then observes the next state $s_{t+1} \sim P(\cdot|s_t,a_t)$. We denote by $\mathbb{P}$ (resp. $\mathbb{E}$) the probability law (resp. expectation) of the process $(Z_t)_t$, where $Z_t=(s_t,a_t)$. We indicate by ${\cal F}_t=\sigma(\{Z_0,\dots, Z_t\})$ the $\sigma$-algebra generated by the random observations made under the algorithm up to time $t$ in the \emph{exploration} phase. For such algorithm, we define $\tau$ to be a stopping rule w.r.t. the filtration $({\cal F}_t)_{t\geq 1}$. This stopping rule $\tau$ simply states when to stop the exploration phase. At the stopping time $\tau$ the algorithm outputs an estimate $\hat\pi_{\tau}$ of the best using the collected data. We focus on $\delta$-PC algorithms, according to the following definition.

\begin{definition}[$\delta$-PC Algorithm]
    We  say that an algorithm  is $\delta$-PC if, for any MDP $M$, it outputs (in finite time) an optimal policy with probability at-least $1-\delta$, i.e.,  $\mathbb{P}[\tau<\infty,\; \hat\pi_{\tau}\in  \Pi^\star(M)]\geq 1-\delta$.
\end{definition}

Let $N_{t}(s,a)$ denote the number of visits to state-action pair $(s,a)$ up to time $t$. The expected sample complexity of an algorithm $\mathcal{A}$ can be written as $\mathbb{E}[\tau] = \sum_{s,a} \mathbb{E}[N_{\tau}(s,a)]$. It is this sample complexity that the learner hopes to minimize.

\subsection{Sample Complexity Lower Bound}\label{subsec: lower bound}
For BPI, the following lower bound Theorem \ref{thm:lowerbound} has been shown to hold for all $\delta$-PC algorithms in \citep{al2021navigating}, where they also provided asymptotic guarantees on their algorithm NaS. Before stating the bound, we first recall two structural constraints on the learning problem. These are not assumptions imposed by our analysis, but rather necessary conditions inherited from the information-theoretic lower bound that any $\delta$-PC algorithm must satisfy. The \emph{information constraints} arise from change-of-measure arguments: to guarantee correctness, the algorithm must collect enough statistical evidence to distinguish the true MDP from every alternative with a different optimal policy. The \emph{navigation constraints} are specific to the online setting: since observations are generated along a single trajectory, the expected visitation counts must satisfy the flow conservation conditions of the induced Markov chain. Together, these two constraints characterize the feasible set of sampling allocations for any $\delta$-PC algorithm.

\textbf{(i) Information constraints. }
The algorithm must collect enough evidence to distinguish the true MDP $M$ from any \emph{alternative} MDP $M'$ with a different optimal policy. Define the alternative set
\begin{equation}\label{eqn_alt_set_definition}
    \mathrm{Alt}(M) = \{ M' \in \mathcal{M} : M \ll M',  \Pi^\star(M) \cap \Pi^\star(M') = \emptyset \},
\end{equation}
 where $\Pi^\star(M)$ is the set of optimal policies for MDP $M$. Then, for any $\delta$-PC algorithm and any $M' \in \mathrm{Alt}(M)$,
\[
\sum_{(s,a)} \mathbb{E}_M[N_{s,a}(\tau_\delta)] \,
\KL_{M|M'}(s,a)
\ge \mathrm{kl}(\delta, 1-\delta),
\]
where we define these KL terms to be the KL divergence between the two respective transition functions at each state action pair, $\KL_{M|M'}(s,a) =\KL(P_M(\cdot|s,a)||P_{M'}(\cdot|s,a))$  
and $\mathrm{kl}$ is the binary divergence. We also define the KL divergence between two transitions as the sum of the KL divergences of the transitions at each state-action pair $\KL(M,M') = \sum_{s,a}\KL_{M|M'}(s,a)$.

\textbf{(ii) Navigation constraints. }
Since observations arise from a single controlled trajectory, expected visitation counts must satisfy stationarity conditions of the induced Markov chain.  
Specifically, for all $s \in \mathcal{S}$,
\[
\bigg| \mathbb{E}_M[N_{\tau_\delta}(s)] - \sum_{s',a'} p_M(s | s',a') \, \mathbb{E}_M[N_{\tau_\delta}(s',a')] \bigg| \le 1.
\]
As $\delta \to 0$, the normalized counts converge to a stationary distribution satisfying exact equality.

We can now state the asymptotic lower bound on the sample complexity of any $\delta$-PC algorithm for the BPI problem. The construction follows the classical bandit lower bound of \citet{garivier2016optimal}, adapted to structured settings. We begin by defining the set of \emph{navigation-constrained allocations}:
\[
\Omega(M) = \bigg\{ \omega \in \Delta(S \times A) :
\omega(s) = \sum_{s',a'} p_M(s| s',a') \omega(s',a'),\;\forall s\in \statespace \bigg\},
\]
where $\omega(s)=\sum_a \omega(s,a)$. We can then state the lower bound as follows.

\begin{theorem}[Asymptotic Lower Bound]
\label{thm:lowerbound}
For any $\delta$-PC algorithm $\mathcal{A}$ and any $M \in \mathcal{M}$,
\[
\liminf_{\delta \to 0}
\frac{\mathbb{E}_M[\tau_\delta]}{\log(1/\delta)} 
\ge T(M),
\]
where
\[
\frac{1}{T(M)} 
= 
\sup_{\omega \in \Omega(M)}
\inf_{M' \in \mathrm{Alt}(M)}
\sum_{(s,a)} \omega(s,a)\,
\KL_{M|M'}(s,a).
\]
\end{theorem}

This lower bound can be interpreted as a two-player zero-sum game: the learner selects an allocation $\omega \in \Omega(M)$, while nature chooses the hardest alternative $M'$. Compared to the classical bandit setting, the key distinction is that $\Omega(M)$ enforces feasibility via the MDP dynamics, restricting allocations to stationary distributions. Consequently, the achievable rate is no better, and typically strictly worse, than in the generative model setting, where arbitrary sampling is allowed.

For convenience, define
\[
T(\omega, M)^{-1} = \inf_{M' \in \mathrm{Alt}(M)}\sum_{(s,a)} \omega(s,a)\,
\KL_{M|M'}(s,a).
\]
Let $\Omega^\star(M) = \argmin_{\omega \in \Omega(M)} T(\omega, M)$ denote the set of optimal allocations, and define a selector
\[
\omega^*(M) = \argmin_{\omega \in \Omega^\star(M)} \|\omega\|_2^2.
\]
We omit dependence on $M$ when clear. Note that $\Omega^*(M)$ is a convex set (this is proven in Lemma~\ref{aux_lemma_optimising_set_convex}) and thus $\omega^*(M)$ is well defined.
\begin{remark}[Non-uniqueness]
The optimal allocation set $\Omega^\star(M)$ need not be a singleton. Prior work often avoids this issue through convex relaxations or approximations, which can yield a unique target allocation but may sacrifice optimality \citep[e.g.][]{al2021navigating,russomulti,russo2025adaptive}. We instead work directly with the navigation-constrained allocation problem, allowing $\Omega^\star(M)$ to be set-valued. This affects both the analysis of NaS and the sampling rule used to track an optimal allocation. Non-singleton examples are discussed in Theorem~\ref{theorem: arbitrary_sharpness}.
\end{remark}

\subsection{Assumptions}\label{subsection:assumptions}
In this work, we study the behavior of the sample complexity $\mathbb{E}_M[\tau_\delta]$ and, in particular, provide non-asymptotic sample-complexity guarantees for a modified version of Navigate and Stop \citep{al2021navigating}. To do so, we impose the following assumptions and discuss them below.

\begin{assumption} \label{assumption_comminicating_aperiodic}
(I) $M\in {\cal M}$, where ${\cal M}$ is a set of communicating MDPs \citep{puterman2014markov} with a unique optimal policy ; (II) $M$ is aperiodic under a policy that assigns a positive probability to all actions in every state.
\end{assumption}

\begin{assumption} \label{assumption_P_in_G_bounded_ratio}
    Assume that there exists a known reference measure $\mu$ and constant $L>1$ such that the true transition measure is in a set $P\in G_\mu(L)$,
    where we define
    \begin{equation}
        G_\mu (L) = \left\{Q \in \Delta(S)^{S\times A}  \; : \; \mu(\cdot|s,a) \ll Q(\cdot|s,a), \quad \frac{\mu(\cdot|s,a)}{Q(\cdot|s,a)} \in [1/L,L] \quad \forall (s,a) \in S\times A \right\}.
    \end{equation}
\end{assumption}

\begin{assumption} \label{assumption_w_0_condition}
    Consider $d(s,a)=V_{P}^{\pi^\star}(s)-Q_{P}^{\pi^\star}(s,a)$. For all $s,a\neq \pi^\star(s)$ we have  \[
   \gamma\left(\|V_P^{\pi^\star}\|_\infty- \mathbb{E}_{s'\sim P(s,a)}[V^{\pi^\star}(s')]\right) > d(s,a),
    \]
    and for all $s$ we have that
    \[
     \gamma (1-\gamma)\left( \mathbb{E}_{s'\sim P(s,\pi^\star(s))}[V^{\pi^\star}(s')]- \min_{s''}V^{\pi^\star}(s'')\right)> \min_{a\neq\pi^\star(s)}d(s,a),
    \]
\end{assumption}

Assumption \ref{assumption_comminicating_aperiodic} is standard in the BPI literature \citep{al2021navigating,russo2023model,russomulti,russo2025adaptive}. The communicating requirement is essentially unavoidable: if some states are unreachable from others, the learner may be unable to gather the evidence needed to distinguish the true MDP from alternatives with a different optimal policy, regardless of the exploration strategy employed. Aperiodicity under the uniform-action policy is a mild technical condition that ensures the Markov chain mixes at a geometric rate, which is needed for concentration of empirical visitation frequencies along a single trajectory. The unique optimal policy requirement is also not restrictive, in particular we prove in Appendix \ref{lebesgue_measure_zero_appendix_section} that the set of MDPs with multiple optimal policies has Lebesgue measure 0 under a prior over ${\cal M}$ that is absolutely continuous w.r.t. the Lebesgue measure. We believe this new result may be of independent interest in other works, particularly Bayesian sequential testing problems. 

Assumption \ref{assumption_P_in_G_bounded_ratio} is analogous to boundedness assumptions commonly made in the bandit best-arm identification literature \citep[e.g.][]{degenne2019non,poiani2025non}, where they are typically introduced to ensure that key quantities such as divergence measures remain uniformly bounded and well-behaved, which in turn enables the use of concentration arguments and leads to finite-confidence guarantees in the bandit setting. Assumption \ref{assumption_P_in_G_bounded_ratio} provides the learner with knowledge of which transitions have zero probability (via the reference measure $\mu$) and imposes a bounded density ratio on the remaining transitions. Indeed, this is similarly important in the BPI setting as ensuring that functions like $M \mapsto \KL_{M|M'}(s,a)$ are well-behaved allow us to control the non-asymptotic sample complexity. Note that the set $G_\mu(L)$ is convex (see Lemma \ref{lemma_G_mu_is_convex}). 

Finally, the last Assumption 3 is introduced to avoid the awkward case where the set of alternatives may be empty for some state-action pairs, which leads to optimal visitation strategies that may not assign positive mass to some pairs. This issue can be avoided by considering stochastic rewards, as in \citep{al2021navigating}, or by using Assumption ~\ref{assumption_w_0_condition}, which is sufficient to avoid these degenerate cases, and we prove this fact  in Lemma ~\ref{lemma_auxiliary_positive_omegas}. Therefore, we use this assumption to avoid notation cluttering, as our main goal is to focus on non-asymptotic guarantees that depend on the underlying dynamics of the system.

\section{Navigate-and-Stop Algorithm}\label{section_algorithm}

In this section we describe the Navigate-and-Stop (NaS) algorithm \citep{al2021navigating} for which we provide non-asymptotic analysis. The algorithm is motivated by the instance-specific lower bound on sample complexity, which we present in Section~\ref{subsec: lower bound}. This lower bound reveals that for any MDP $M$, there exists an optimal state-action visitation frequency that minimizes the sample complexity over all $\delta$-PC algorithms. Since this frequency depends on the unknown model $M$, NaS follows an estimate-then-track strategy: at each round, it estimates the transition kernel from data, computes the optimal visitation frequency with respect to the current estimate, and plays according to a policy whose stationary distribution matches this frequency. To ensure sufficient coverage for consistent estimation, this policy is mixed with a forcing component that guarantees every state-action pair is visited. This type of strategy is known to yield asymptotically optimal sample complexity; in this paper we go further by establishing non-asymptotic guarantees, as detailed in Section~\ref{section_upper_bounds}. Our presentation follows \citep{al2021navigating} with two modifications to Algorithm~\ref{algo:nas}: we handle potential non-uniqueness of optimal allocations explicitly, and we adjust the forcing schedule for reasons arising in the finite-time analysis.

\subsection{Algorithm}
NaS is made of three components: (1) a sampling rule, that dictates how to choose the next action; (2) a stopping rule, that dictates when to stop the data collection process; (3) a recommender rule, that outputs the estimated best policy. All of these components in NaS require the algorithm to maintain an empirical estimate of the model $M_t$ (i.e., an estimate of the transition function). We write out NaS in Algorithm \ref{algo:nas} and go through the details and modifications from  \citep{al2021navigating} below.

\begin{algorithm}[t]
	\caption{\nas{} (Navigate-and-Stop Algorithm)}
    \label{algo:nas}
    \small
	\begin{algorithmic}[1]
		\REQUIRE Confidence $\delta$; exploration terms $(\alpha,\beta,\nu)$, reference measure $\mu$
		\STATE Initialize  counter $t\gets 1$, $N_t(s,a)\gets 0$ for all $(s,a)\in \statespace\times \actionspace$.
        \STATE Set exploration term $\epsilon_t = 1/\max\{1+\nu,N_t(s_t)\}^\alpha$ and observe $s_1$. 
        \WHILE{$tT( N_t/t; M_t)^{-1} < b(N_t,\delta)$} \label{lst:line:stopping_rule}
            \STATE Compute $\Omega_t^\star(M_t)= \argmin_{\omega \in\Omega(M_t)} T(\omega; M_t)$ and let $\omega_t^\star= \argmin_{\omega \in\Omega^*_t(M_t)} \|\omega\|_2$. \label{lst:line:computation_omega}
            \STATE Set $\pi_t(a|s_t) = (1-\epsilon_t)\pi_t^\star(a|s_t) + \epsilon_t \pi_{f,t}(a|s_t)$,  where $\pi_t^\star(a|s) = \omega_t^\star(s,a)/\sum_{a'}\omega_t^\star(s,a')$. \label{lst:line:final_policy}
            \STATE Play $a_t\sim \pi_t(\cdot|s_t)$ and observe $s_{t+1}\sim P(\cdot|s_t,a_t)$.
            \STATE Update $N_t(s_t,a_t), M_t$ and set $t\gets t+1$.
        \ENDWHILE
        \ENSURE an optimal policy $\hat \pi_{r}^\star$. 
	\end{algorithmic}
\end{algorithm}

\textbf{Empirical estimates of the MDP. }
As the learner explores the MDP, we will maintain an estimate $M_t$ for the MDP in every round $t \in \mathbb{N}_{+}$ as follows. Define the counts $N_t(s,a,s')=\sum_{j=1}^t \mathds{1}\{(s_j,a_j,s_{j+1})=(s,a,s')\}$ and $N_t(s,a) = \sum_{s'\in \mathcal{S}} N_t(s,a,s')$. We define the empirical transition function $P_t(s'|s,a) = \frac{N_t(s,a,s')}{N_t(s,a)}$
if $N_t(s,a)>0$, otherwise we define $P_t(s'|s,a) \equiv 1/S$.
 We initialize as in line 1 and update in every round as in line 7 of Algorithm \ref{algo:nas}. We also define the projection of this estimator onto the space $G_\mu(L)$ from Assumption \ref{assumption_P_in_G_bounded_ratio} as
\begin{equation} \label{eqn_projection_defintiion}
        P_t^\Pi = \argmin_{Q\in G_\mu(L)} \sum_{s,a}\|P_t(\cdot|s,a) - Q(\cdot|s,a)\|_1,
    \end{equation}
     letting $M_t$ be the approximated MDP with this projected transition matrix. We will oftentimes use the notation for the MDP and the transition matrices interchangeably throughout analysis, as we are assuming rewards are known and deterministic.

\textbf{Sampling Rule. }
As outlined by the lower bounds discussed in Section \ref{subsec: lower bound}, the optimal policy should try and follow a state-action allocation in $\Omega^*(M)$,
as this will yield a sample-efficient exploration policy. However, as the model \(M\) is unknown to the learner, we cannot directly compute an optimal allocation in \(\Omega^\star(M)\). Instead, we use the current estimate of the model \(M_t\) in place of the true model \(M\) to compute the approximated set of optimal allocations and denote $\omega^*_t$ as the approximated optimal allocation with the smallest 2-norm, namely
\[
\Omega_t^\star (M_t) = \arginf_{\omega \in \Omega(M_t)} T(\omega; M_t), \qquad \omega_t^\star= \argmin_{\omega \in\Omega^*_t(M_t)} \|\omega\|_2.
\]
We then choose our actions using $\pi_t^*(a|s)$, a policy which has stationary distribution equal to $\omega^*_t$,
$$\pi_t^\star(a|s) = \frac{\omega_t^\star(s,a)}{\sum_{a'}\omega_t^\star(s,a')}.$$
However, if we were to play actions which enforce this allocation, this would only be efficient with respect to the estimated model, $M_t$. Hence, we would like to construct our sampling strategy to ensure that the estimated model $M_t$ converges to the true $M$. To that aim, we mix the computed allocation with a \emph{forcing policy} $\pi_f$, ensuring that each state-action pair is sampled sufficiently (and infinitely often in the limit). This is done in lines 4-6 of algorithm \ref{algo:nas}. Motivated by \citep{russomulti}, we use the forcing policy $\pi_{f,t}(\cdot | s) = \mathrm{softmax}(-\beta_t(s) N_t(s, \cdot))$
with
\[
\beta_t(s) = \frac{\beta \log(N_t(s))}{\max_{a}[N_t(s,a)] - \min_{b}[N_t(s,b)]}, \qquad \beta \in (0,1)
\]
and $\mathrm{softmax}(x_i) = e^{x_i} / \sum_j e^{x_j}$ for a vector $x$.
This choice encourages selecting under-sampled actions for $\beta \gg 0$ while for $\beta \rightarrow 0$ we obtain a uniform forcing policy $\pi_f(a \mid s) = 1/A$.  
Motivated by \citep{russomulti}, we avoid estimating the MDP connectivity required in \citep{al2021navigating} and instead use
$\varepsilon_t = 1 / \max(1+\nu, N_t(s_t))^\alpha,$
with $N_t(s) = \sum_a N_t(s,a)$ and $\alpha \in (0,1)$, which only depends on the number of state visits.  To ensure convergence of $\omega_t^\star \to \omega^\star$, we only enforce that $\alpha + \beta < 1$ and $\nu \in (0,1)$.

\textbf{Stopping rule. }
To decide when to stop, we employ a Generalized Likelihood Ratio (GLR) test.
Let $b(t,\delta)$ denote the corresponding confidence threshold for transitions (logarithmic in $t$ and $1/\delta$)
\[
b(t,\delta) = \log(1/\delta)
    + (S-1)\sum_{(s,a)\in S\times A} \log\!\left( e\,\left[1 + \frac{N_t(s,a)}{S-1}\right] \right)
\]

The stopping time is defined as
\[
\tau_\delta = \inf\left\{ t \ge 1 :
t \, T\big(N(t)/t,M_t\big)^{-1}
\ge b(t, \delta)
\right\}.
\]
Upon stopping, the algorithm outputs the empirical optimal policy $\widehat{\pi}_{\tau_\delta}^\star$, calculated via policy iteration on the final estimated MDP, as in the last line of Algorithm \ref{algo:nas}. Again, as outlined in \citep{al2021navigating}, this stopping time will ensure that the policy is $\delta$-PC. However, we will go into further details of this later. The $\delta$ subscript is dropped when clear.

\section{Non-Asymptotic Upper Bound for NaS} \label{section_upper_bounds}

\subsection{Sharpness of the Objective Function}\label{main_paper_sharpness_subsection}

Before presenting formal sample-complexity bounds for Algorithm~\ref{algo:nas}, we isolate a key finite-confidence cost: the rate at which the empirical allocations induced by the estimated oracle policies $\pi_t^\star$ approach an optimal allocation. This rate depends on the local geometry of $T(\omega,M)^{-1}$ near its maximizers. When the objective is flat, small model-estimation errors can induce large allocation errors, delaying the stopping rule. We capture this effect through a \emph{sharpness function}.

\textbf{Sharpness Function.}
Let $J_M(\omega) := T(\omega,M)^{-1}$ and define $J^\star_M := \max_{\omega \in \Omega(M)} J_M(\omega)$ with optimal set $\Omega^\star(M) = \arg\max_{\omega \in \Omega(M)} J_M(\omega)$. Let $\mathrm{dist}_\infty(\omega,S(M))
:=
\inf_{\omega'\in S(M)} \|\omega-\omega'\|_\infty$. Let $\varphi$ be a non-decreasing function with $\varphi(0)=0$. Then $\varphi$ is a sharpness function under $M$ if we have $\varphi(J_M^\star-J_M(\omega)) \geq \mathrm{dist}_\infty(\omega,\Omega^\star(M))$. We say that $J_M$ has \emph{polynomial sharpness of order $p>0$} if there exist a constant $c > 0$ such that  $ \varphi(t) =  c \, t^p$. It is sufficient for our analysis that this sharpness condition hold locally in a neighborhood of $\Omega^\star(M)$.

Theorem~\ref{theorem: arbitrary_sharpness} shows that this effect can be arbitrarily severe: for any $m>n\geq 1$, there is an instance with local polynomial sharpness exponent $p=(m-n)/m$. Figure~\ref{fig:main_figure} verifies these rates numerically; details are deferred to Appendix~\ref{appendix_sharpness}.
We visualize the simple family of MDPs used in $\mathcal{M}$ in Figure \ref{fig:MDP} and provide more details of the theorem's proof and construction in Appendix \ref{appendix_sharpness}.

\begin{theorem}\label{theorem: arbitrary_sharpness}
Fix integers $m>n \geq 1$. There exists an MDP $M$ in a family of MDPs
$\mathcal{M}$ and, for all sufficiently small $d>0$, an allocation
$\omega^{(d)} \in \Omega(M)\setminus \Omega^\star(M)$ such that with
$\mathrm{dist}_{\infty}(\omega^{(d)}, \Omega^\star(M)) = d$,
and positive constant $c_{m,n}>0$
\[
J^\star_M-J_M(\omega^{(d)})
=
c_{m,n}d^{\frac{m}{m-n}}
+
o\!\left(d^{\frac{m}{m-n}}\right),
\qquad d\downarrow 0.
\]
\end{theorem}

It may be of practical relevance to restrict the set of MDPs to be within some discrete set, for example due to some limitations of precision. In this case, we can provide some simplifications, showing that the sharpness becomes linear. The proof can be found in Appendix \ref{sec:appendix:discretization}.
\begin{proposition}
       If the set of MDPs $\mathcal{M}$ is finite, then the sharpness function $\varphi$ is linear.
\end{proposition}
We note that the definition of sharpness in this section, which  appears later in the sample complexity upper bound is representative of the worst case sharpness of $J_M(\omega)$ approaching the set of optimal allocations from any direction. However, it would be interesting future work to see if there was an algorithmic formulation which could exploit more approachable paths to the optimal set $\Omega^*(M)$ for improved sample complexity.

\begin{figure}
    \centering
    \includegraphics[width=0.4\linewidth]{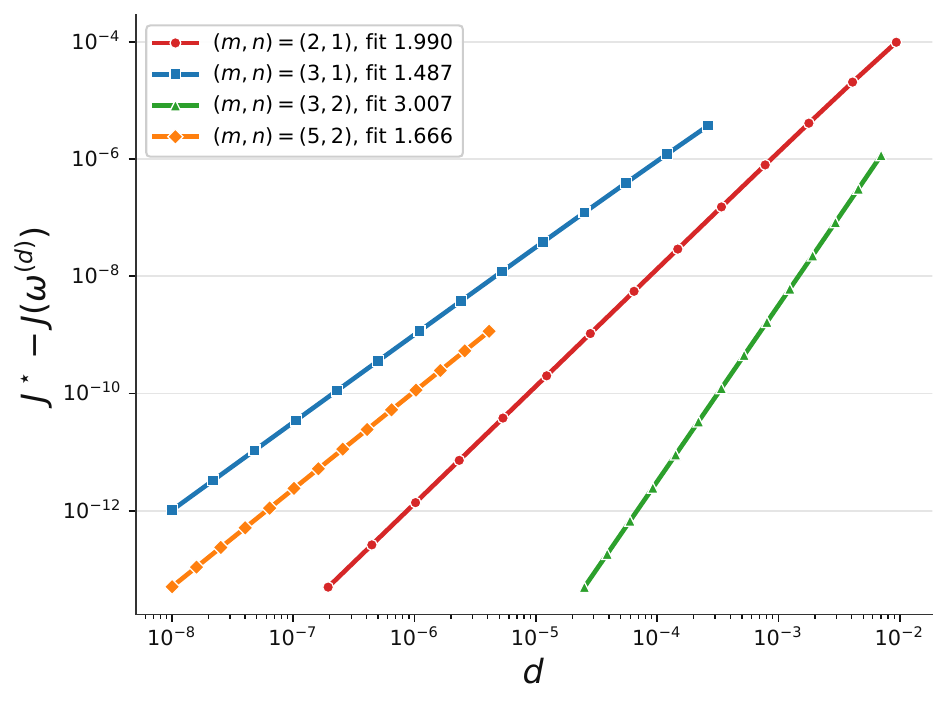}
    \caption{Empirical sharpness scaling for the MDP family of Theorem~\ref{theorem: arbitrary_sharpness}.
For each choice of integers $m>n\geq 1$, we plot the objective gap
$J^\star - J(\omega^{(d)})$ with different values of $d$.
On log--log axes, the slope of each curve estimates the gap exponent $m/(m-n)$, and hence the reciprocal slope estimates the sharpness exponent of the function $\varphi$, namely $(m-n)/m$.
The different choices of $(m,n)$ therefore realize different polynomial sharpness exponents.}
    \label{fig:main_figure}
\end{figure}

\subsection{Implicit and Explicit Upper Bounds}

We can now provide non-asymptotic upper bounds for the sample complexity of the NaS algorithm described in Section \ref{section_algorithm}. We provide an overview of the techniques used for the analysis in Appendix \ref{app_proof_overview} and the full proof is shown in the rest of Appendix \ref{full_app_upper_bound_proof}. Here we display an implicit form of the upper bound and we provide interpretation for the problem dependent terms below\footnote{We also include a table of notation in Appendix \ref{table_notation_appendix} which defines common notation used throughout the paper.}.

\begin{theorem} \label{theorem_non_asymp_upper}
Under Assumptions \ref{assumption_comminicating_aperiodic}-\ref{assumption_w_0_condition},  
NaS in Algorithm \ref{algo:nas}, we can provide the following non-asymptotic upper bound for the sample complexity at confidence requirement $\delta\in(0,1)$.
    \begin{equation*}
        \mathbb{E}[\tau_\delta] \leq \inf \left\{t\in\mathbb{N}: b(t,\delta) < (t-\sqrt{t}) T(M)^{-1} - \ell_0(t,M)\right\} + 7
    \end{equation*}
    where 
        \begin{align*}
            & \ell_0(t,M) = \tilde{\mathcal{O}}\bigg( S^{\frac{5}{2}}A\log(L)m^{-\frac{1}{2}} \eta_m^{-\frac{1}{2}}t^{\frac{1+\alpha+\beta}{2}}\\
            +& t\log(L)(\tilde{L}+\kappa_M)\tilde{L}\left[t^{-\frac{1}{4}} + \frac{A}{\underline{\rho}_M}\varphi_M^{\frac{1}{2}}\left(\frac{S^{\frac{3}{2}}m^{\frac{1}{2}}A\log(L)(S^2+\kappa_\Omega)}{\eta_m^{\frac{1}{2}}t^{\frac{1-(\alpha+\beta)}{4}}}\right) +  \frac{m^\alpha}{t^{\alpha(1-(\alpha+\beta))/2}A^\alpha\eta_m^\alpha}\right]\bigg)
    \end{align*}
\end{theorem}

\begin{remark}
    Extending our analysis to stochastic rewards would primarily add reward-estimation terms to the first component of $\ell_0$ and to the argument of $\varphi$, with dependence on the reward variances. The navigation analysis should remain largely unchanged, so we focus here on the main difficulty: controlling the transition dynamics.
\end{remark}
The error term $\ell_0$ captures two distinct sources of finite-time cost. The first quantity reflects the time required for the forcing policy $\pi_f$ to ensure sufficient coverage, so that the empirical transition estimates concentrate around the true transition kernel (see Lemma~\ref{lemma_bounding_sum_of_M_estimate_errors}). The second quantity measures how quickly the empirical visitation frequencies converge to the stationary distributions induced by the sequence of policies $(\pi_t)_{t \geq 1}$, and arises as the radius term in our main concentration result, Proposition~\ref{prop_concentration_Nt_vs_SumOmega}.
The constant $L$ is from our Assumption \ref{assumption_P_in_G_bounded_ratio} on the boundedness of the transitions, the sharpness function $\varphi$ is  outlined in Section \ref{main_paper_sharpness_subsection}, and $\alpha$ and $\beta$ are hyperparameters of the NaS Algorithm \ref{algo:nas}. The remaining instance-dependent constants admit natural interpretations as follows, although formal definitions are given in the appendix and a table of notation is provided in Appendix \ref{table_notation_appendix} for convenience.
\begin{itemize}
    \item The terms $m$, $\kappa_M$, and $1/\eta_m$ quantify the connectivity of the underlying MDP: smaller values correspond to faster mixing and tighter concentration. 
\item The term $\tilde{L}$
 is a measure of how well the $\omega^*$ policy-induced Markov chain mixes, being bounded tightly when there is sufficient ergodicity and coverage (i.e., the policy stays close enough to uniform and avoids concentrating on a small subset of states/actions, see Appendix~\ref{subsection_bounding_L_tilde}). 
 \item The constant $\kappa_\Omega$ measures the sensitivity of the feasible allocation set $\Omega(M)$ to perturbations of the transition kernel (see Lemma \ref{lem:feasible_set_stability} for details).
 \item Finally, $\underline{\rho}_M$ is the smallest entry of the optimal allocation $\omega^*$ \footnote{A similar multiplicative term appears in the related non-asymptotic lower bounds found in \citep[Theorem 3.3]{sethi2026asymptotically}, however more work is required to identify if such a dependence is unavoidable in our setting as well.}. 
\end{itemize}

\begin{remark}
     Note that the error term $\ell_0$ is sublinear in $t$ and also $b(t,\delta)= \mathcal{O}(\log(1/\delta)+\log(t))$, hence Theorem \ref{theorem_non_asymp_upper} recovers the asymptotic optimality of NaS from \citep{al2021navigating} as desired. In particular, NaS achieves the exactly correct coefficient for $\log(1/\delta)$ in the limit $\delta\rightarrow 0$. This is stated in Theorem \ref{theorem_asymptotic_nas_upper_bound}. Note that we get exact optimality in the limit because we are considering the setting with deterministic rewards, whereas in the original NaS analysis by \citet{al2021navigating} they are suboptimal by a constant of 2. This suboptimality is due to their stopping time taking into account both transition and rewards separately. However, we suspect that this factor of 2 can be avoided regardless with a careful modification of their stopping rule, dealing with these complexities simultaneously.
\end{remark}

\begin{theorem}[Asymptotic Sample Complexity Upper Bound]
\label{theorem_asymptotic_nas_upper_bound}
Under Assumptions \ref{assumption_comminicating_aperiodic}-\ref{assumption_w_0_condition},  
NaS in Algorithm \ref{algo:nas} is $\delta$-PC and satisfies:
\begin{align*}
\limsup_{\delta \to 0} 
\frac{\mathbb{E}_M[\tau_\delta]}{\log(1/\delta)}
&\le \, T(M).
\end{align*}
\end{theorem}

We can now finally provide an explicit upper bound for the sample complexity. The form of this final upper bound is dependent on the problem's sharpness function $\varphi$. We discussed in the previous section that one appropriate form of $\varphi$ is linear. In such  a case, we can simplify the implicit upper bound above to the explicit upper bound here. See Corollary \ref{cor:polynomial_sharpness} in the Appendix for the proof.

\begin{corollary}[Linear Sharpness]
Suppose there exists $C_\varphi>0$ such that
$\varphi_M(x)\le C_\varphi x $ for all relevant $ x\ge 0$. Let $L_\delta:=\log(1/\delta)$ and define
\[
B_\delta
:=
T(M)\Bigl(
L_\delta
+
S^2 A \log\!\bigl(e T(M)(1+L_\delta)\bigr)
+
2S^2 A \log(2e)
\Bigr).
\]
Define constants
\[
A_1:=S^{7/2}A\log(L)m^{-1/2}\eta_m^{-1/2},\quad
A_2:=\log(L)(\widetilde L+\kappa_M)\widetilde L,
\]
\[
A_3:=A_2\,\frac{A}{\underline{\rho}_M}\,C_\varphi\,
\left[S^{3/2}m^{1/2}A\log(L)(S^2+\kappa_\Omega)\right]^{\frac{1}{2}}
\eta_m^{-1/4},\quad
A_4:=A_2\,\frac{m^\alpha}{A^\alpha\eta_m^\alpha}.
\]
Let $
r_1=\tfrac{1+\alpha+\beta}{2},
r_2=\tfrac{3}{4},
r_3=\tfrac{7+\alpha+\beta}{8},
r_4=1-\tfrac{\alpha(1-(\alpha+\beta))}{2}.
$
Then $r_i<1$ for all $i\in\{1,2,3,4\}$, and

\begin{equation}\label{eqn_main_paper_explicit_final}
    \mathbb E_M[\tau_\delta]
\le
B_\delta
+
\widetilde{\mathcal O}\!\left(
\sqrt{B_\delta}
+
T(M)\sum_{i=1}^4 A_i B_\delta^{r_i}
\right)
.
\end{equation}
\end{corollary}

We note that this final upper bound is much richer in information that previous works. In particular, the right hand side of this explicit upper bound does not require asymptotically small $\delta$. Hence, with this non-asymptotic upper bound, we are able to see more sources of complexity before the $T(M)\log(1/\delta)$ term dominates in the limit. For example, from the sum in \eqref{eqn_main_paper_explicit_final}, we see how terms contributing to the connectivity of the MDP (e.g. $m$ and $\kappa_M$) appear in the coefficients $A_i$ with different exponents $r_i$. Indeed, we also see that the hyperparameters used in NaS (i.e. $\alpha$ and $\beta$) affect these exponents and the sample complexity. Overall, this is an informative explicit non-asymptotic upper bound and it would be interesting to analyse this further to understand what terms are unavoidable in the sample complexity and what terms can be improved with different algorithmic choices. We discuss this further in the next section.

\section{Related Work \& Discussion}

Many ideas in BPI are inherited from best-arm identification and sequential testing in bandits, most notably Track-and-Stop \citep{garivier2016optimal}, which combines an allocation-tracking sampling rule with a likelihood-based stopping rule. This perspective has since been extended to settings with multiple correct answers \citep[e.g.][]{degenne2019pure,poiani2025pure}, alternative objectives \citep[e.g.][]{pmlr-v267-lazzaro25a,jourdan2023varepsilon,tuynman2025batch,jourdan2025optimal,garivier2017thresholding}, and richer feedback structures \citep[e.g.][]{russo2025pure}. More recently, non-asymptotic analyses have been developed for game-theoretic variants of Track-and-Stop \citep{degenne2019non} and for the original Track-and-Stop algorithm itself \citep{poiani2025non}, with techniques that motivate our analysis. In RL, much of the classical literature focuses on regret minimization \citep[e.g.][]{boone2025regret,NEURIPS2018_d693d554,06a276ef-3e35-33d0-912a-626bc0841b5e}, whereas BPI is a fixed-confidence sequential testing problem whose goal is to identify an optimal policy as quickly as possible. Existing BPI methods include generative-model approaches \citep[e.g.][]{al2021adaptive,taupin2023best} and online single-trajectory methods, most prominently Navigate-and-Stop \citep{al2021navigating}, with extensions to multi-reward settings \citep{russomulti,russo2025adaptive}. However, in the online NaS setting, existing guarantees are asymptotic (see Appendix \ref{appendix_related_work} for a more detailed overview).

A main takeaway of this work is that asymptotic optimality alone does not fully characterize the difficulty of online BPI. Even when an algorithm attains the optimal leading coefficient $T(M)$, its finite-confidence performance can be governed by additional instance-dependent effects, including the cost of transition coverage, the connectivity and mixing properties of the MDP, and the sharpness of the allocation objective. Our bounds make these effects explicit, suggesting concrete directions for improved exploration strategies that preserve asymptotic optimality while reducing finite-time costs.

A natural next step is to understand whether the dependencies in our non-asymptotic upper bound can be improved, ideally while preserving the asymptotic optimality of the algorithm. More broadly, we believe that our bounds help identify the aspects of problem instances that must be controlled in the development of new BPI methods. To complement these upper bounds, it would also be interesting to study tight, non-asymptotic lower bounds for BPI. Concurrent work has considered such lower bounds for related problems in sequential testing, but significant work would be required to obtain tight, expressive, non-asymptotic lower bounds in our setting 
\citep[Appendix~A.4, ][]{sethi2026asymptotically}. Another important direction for future work is the development of methods that are not only asymptotically optimal, but also computationally efficient and practically deployable. This is a current limitation of NaS-style strategies. A further promising avenue is to extend these non-asymptotic guarantees to the linear reinforcement learning setting, where richer function approximation classes are considered. We conclude by noting that this work fills a fundamental gap in the BPI literature by establishing finite-confidence guarantees for instance-optimal methods, and we hope it serves as a foundation for further advances in both theory and practice.

\begin{ack}
Joseph Lazzaro was supported by the Roth Scholarship through the Department of Mathematics at Imperial College London. 
\end{ack}

\bibliographystyle{abbrvnat}
\bibliography{bibliography}

\newpage
\appendix

\section*{Appendix}
\tableofcontents

\newpage
\section{Further Related Work}\label{appendix_related_work}

\paragraph{Best Arm Identification \& Sequential Testing in Bandits} Methods for BPI are built upon ideas constructed in for optimal Best Arm Identification (BAI) and sequential testing in bandits. In particular, much of the work discussed here is a consequence of the seminal paper by \citet{garivier2016optimal} which introduces the first asymptotically optimal method, Track-and-Stop. The main idea being to track the proportion of actions played across the action space with estimated optimal allocations (and then stop when sufficient evidence has been gathered). This blossomed into research studying problems with multiple correct answers \citep[e.g. ][]{degenne2019pure, poiani2025pure}, different objectives \citep[e.g. ][]{pmlr-v267-lazzaro25a,jourdan2023varepsilon, tuynman2025batch, jourdan2025optimal, garivier2017thresholding}, and different types of feedback \citep[e.g.][]{russo2025pure}.
There has also been a line of work providing non-asymptotic analysis to these optimal methods. Notably \citet{degenne2019non} consider game-theoretic variants of Track-and-Stop and provided complementary non-asymptotic bounds using the optimistic structure of the decisions. More recently, \citet{poiani2025non} provide non-asymptotic analysis to the original Track-and-Stop algorithm, without making any optimistic modifications to the algorithm. Their theoretical techniques are what originally motivated the work in this paper on BPI. 
We note that the potential non-uniqueness of the true optimal allocations are well-handled in the bandit setting (both in the existing non-asymptotic and asymptotic analysis) due to 1-state nature of bandit problems and through the use of C-tracking. Whereas in the BPI setting of this paper, we have to be much more careful, this is discussed further in Section \ref{section_algorithm}. In this work, we study the more broad RL setting of BPI.

\paragraph{Best Policy Identification} In the full reinforcement learning setting we have both a state-space and action-space on which the agent needs to learn the dynamics of the MDP. Hence, although the existing algorithms for optimal BPI are strongly motivated by the existing work in bandits, the analysis can be significantly more involved and nuanced. We first note that there have been some optimal methods proposed in the setting under which  the agent has access to a generative model which can provide observations from \emph{any} state-action in every round \citep[e.g.][]{al2021adaptive,taupin2023best}. An alternative approach, which we consider in this work, is to consider settings in which the agent only has access to a single trajectory through which they must navigate through the MDP. In \citet{al2021navigating}, they provide the first asymptotically optimal algorithm for BPI with a single trajectory, called Navigate-and-Stop \citep[as it is motivated by the Track-and-Stop algorithm][]{pmlr-v132-kaufmann21a}. Subsequently there have been optimal methods proposed for various other sequential testing problems, such as multi-reward settings \citep{russomulti, russo2025adaptive}. However, our work is the first to provide non-asymptotic guarantees for an optimal BPI method. In particular, we provide non-asymptotic guarantees for the original Navigate-and-Stop algorithm. 

\paragraph{Other related works.}
Our work focuses on active sequential testing for BPI in online RL. A related but distinct line of work studies \emph{policy testing}, where the goal is to decide whether a fixed policy has value above a prescribed threshold, while minimizing the sample complexity of the test \citep[e.g.][]{ariu2025policy}. Concurrent work has also established non-asymptotic lower bounds for a broad class of sequential testing problems with Markovian data \citep{sethi2026asymptotically}. Their results, however, are derived for passive data-collection strategies, whereas BPI in our setting requires active exploration under navigation constraints; we discuss possible connections in later sections. Finally, regret-minimization objectives in RL have been extensively studied \citep[e.g.][]{boone2025regret,NEURIPS2018_d693d554,06a276ef-3e35-33d0-912a-626bc0841b5e}, but they address a fundamentally different criterion from fixed-confidence identification, where the objective is to stop as quickly as possible while returning an optimal policy.

\newpage
\section{Table of Notation}\label{table_notation_appendix}

\renewcommand{\arraystretch}{1.15}
\begin{longtable}{|p{0.1\textwidth}|p{0.85\textwidth}|}
  \hline
  \textbf{Notation} & \textbf{Definition}  \\ 
  \hline
$\mathcal{Z}$ & State-action space, $\mathcal{Z} = \mathcal{S} \times \mathcal{A}$. \\
\hline
$P_\pi$ & Transition induced by $\pi$, that is  $P_\pi((s', a') | (s, a)) = P(s'|s, a)\pi(a'|s')$. \\
\hline
$P_u$ & Transition induced by uniform policy. \\
\hline
$P_t$ & Transition induced by the policy $\pi_t$ of MR-NaS. \\
\hline
$M_t$ & Estimated MDP at time step $t$. \\
\hline
$\omega_t$ & Stationary distribution induced by the policy $\pi_t$. \\
\hline
$\omega^{*}_{\pi_t}$ & Stationary distribution induced by the policy $\pi^{*}_{\pi_t}$. \\
\hline
$\bar{\omega}^{*}_{t}$ & Averaged stationary distribution $\bar{\omega}^{*}_{t} := (1/t)\sum_{i=1}^{t}\omega^{*}_{\pi_i}$. \\
\hline
$\omega_u$ & Stationary distribution induced by the uniform policy. \\
\hline
$\omega^*$ & Optimal allocation with minimum norm $\omega^* = \argmin_{\omega \in \Omega^\star(M)} \|\omega\|_2^2$. \\
\hline
$\pi^*(a|s)$ & Oracle exploration policy, defined as $\pi^*(a|s) := \omega^*(a|s)/\sum_b \omega^*(b|s)$. \\
\hline
$\hat{\pi}_t(a|s)$ & Exploration policy at time step $t$, defined as $\hat{\pi}_t(a|s) := \omega^{*}_t(a|s)/\sum_b \omega^{*}_t(b|s)$. \\
\hline
$\mu_t$ & Defined as $\mu_t := 1/N_t(s)^{\alpha + 1}$. \\
\hline
$m_0$ & Maximum number of steps to move between any pair of states $(s, s')$. \\
\hline
$m$ & Maximum number of steps to move between any pair $(s, a), (s', a')$. \\
\hline
$n$ & Equal to $m = m_0 + 1$. \\
\hline
$\eta_k$ & Minimal probability of reaching $z'$ from $z$ in $n \leq k$ transitions using a \\
& uniform random policy, defined as \\
& $\eta_k := \min\{P^n_u(z'|z) | z, z' \in \mathcal{Z}, n \in \{1, \dots, k\}, P^n_u(z'|z) > 0\}$. \\
\hline
$\eta$ & Defined as $\eta := \frac{1}{n} \eta_{mn}$. \\
\hline
$B_\rho(P)$ or $B_\rho(M)$ & Ball of radius $\rho$ around a transition function $P$, defined as \\
& $\{P' : \max_{s, a} |P(\cdot | s, a) - P'(\cdot | s, a)| \leq \rho \}$. \\
\hline
$r_0$ & Ergodicity constant such that $P^r(z) \geq \nu_0$ for all $r \geq r_0$. \\
\hline
$\sigma_u$ & Constant, defined as $\sigma_u := \min_{z',z \in \mathcal{Z}} P^u(z'|z) \nu_u(z')$. \\
\hline
$\sigma(\epsilon, \pi)$ & Defined as $\sigma(\epsilon, \pi) := (1 - \epsilon) A \min_{s,a} \pi(a|s)^{\frac{r_0(n+1)}{\epsilon + r_0}} \sigma_u$. \\
\hline
$\theta(\epsilon, \pi)$ & $\theta(\epsilon, \pi) := 1 - \sigma(\epsilon, \pi)$. \\
\hline
$C(\epsilon, \pi)$ & $C(\epsilon, \pi) := 2/\theta(\epsilon, \pi)$. \\
\hline
$\rho(\epsilon, \pi)$ & $\rho(\epsilon, \pi) := \theta(\epsilon, \pi)^{1/r_0}$. \\
\hline
$L(\epsilon, \pi)$ & $L(\epsilon, \pi) := (1 - \rho(\epsilon, \pi))^{-1} C(\epsilon, \pi)$. \\
\hline
$\Tilde{L}$ & $\Tilde{L} = 2r_0\sigma_u[A\min_{a,s} \pi^*(a|s)]$\\
\hline
$\sigma_t$, $\theta_t$, $C_t$, $\rho_t$, $L_t$ & Respectively defined as $\sigma_t := \sigma(\epsilon_t, \pi_t)$, $\theta_t := \theta(\epsilon_t, \pi_t)$, etc. \\
\hline
$\phi(t,\lambda)$ & $\phi(t,\lambda)= \frac{1}{2}\sum_{k=1}^{\lfloor\frac{t}{m}\rfloor} \frac{\eta_m}{(km)^{\alpha+\beta}} - \log\frac{SA}{\lambda}$\\
  \hline
$\beta_1(n,\lambda)$ &  $\beta_1(n,\lambda) = (S-1) \log\!\left(e\!\left(1 + \frac{n}{S-1}\right)\!\right)
+ \log(SA/\lambda)$\\
\hline
$\psi(t,\lambda)$ & $\psi(t,\lambda)=2\sqrt{\frac{\beta_1(\phi(t,\lambda),\lambda)}{2\phi(t,\lambda)}}$\\
\hline
$\mathcal{E}^1_T(\lambda)$
&
\begin{equation*}
    \mathcal{E}^1_T(\lambda) = \left\{ \forall t \in [1,T], \forall (s,a):  \KL_{\hat{M}_t|M}(s,a) \leq  \frac{\beta_1(N_t(s,a),\lambda)}{N_t(s,a)} \right\}
\end{equation*}\\
\hline
$\mathcal{E}^2_T(\lambda)$
& 
\[
\mathcal{E}^2_T(\lambda)
= \bigg\{
\forall t \in [1,T], (s,a) \in \mathcal{S} \times \mathcal{A}, \,
N_t(s,a) \ge \frac{1}{2} \sum_{k=1}^{\lfloor t/m \rfloor} \frac{\eta_m}{(km)^{\alpha + \beta}}
- \log \left( \frac{SA}{\lambda} \right)
\bigg\}
\]\\
\hline
$\mathcal{E}_t^5(\lambda)$
&
\begin{equation*}
    \mathcal{E}_t^5(\lambda) = \left\{\forall (s,a) \in \mathcal{Z}, \, \bigg|\frac{N_t(s,a) - \sum_{k=1}^t \omega^*_{k-1}(s,a)}{t}\bigg| \leq \ell(t)\right\}.
\end{equation*}\\
\hline
$\underline{\rho}_M$ & $\underline{\rho}_M = \min_{s,a} \omega^*(s,a)$\\
\hline
$\kappa_{\Omega}$ & Stability of the feasible set $d_H\bigl(\Omega(N),\Omega(M)\bigr)
        \le
        \kappa_\Omega(M)\|P_N-P_M\|_{\infty,1} $\\
\hline
$\kappa_M$ & $\kappa_M = \|\left(I-P_{\pi^*} + 1(\omega^*)^\top\right)^{-1}\|_{\infty}$\\
\hline
$\Tilde{L}$ & $\Tilde{L} = 2r_0\sigma_u[A\min_{a,s} \pi^*(a|s)]$\\
\hline
\end{longtable}

\newpage
\section{Proof for non-asymptotic upper bound}\label{full_app_upper_bound_proof}
\subsection{Overview of proof}\label{app_proof_overview}

\paragraph{TL;DR} The main idea of the proof is to show that with high probability our estimates for the transitions are well-behaved and our search of the state-action space is sufficiently well-spread, ensuring that our empirical visitation frequency concentrates well around the (sum of) estimated optimal allocations non-asymptotically due to the navigation rule used. Hence, since we know that the approximated optimal allocation is a solution to an optimization problem involving the estimated transitions, we can carefully manipulate our stopping-time criteria to provide a non-asymptotic high-probability upper bound for the stopping time. This can then be converted to an upper bound on the expected stopping time, as desired, since all well-behaved events hold with high probability.

\paragraph{Good Events:} Sections \ref{subsection_good_events}-\ref{subsection_main_concentration} are focused on showing that three `good' events ($\mathcal{E}^1_T(\lambda)$, $\mathcal{E}^2_T(\lambda)$ and $\mathcal{E}^3_T$) hold with high probability. Our first good event is that the empirical estimates for the MDP are well-behaved. In particular, in Lemma \ref{E1_high_prob_lemma}, we show that the KL-divergence between the estimate for the transitions and the true transitions is bounded from above up to time $T$ with high probability. Namely,
    \begin{equation*}
    \mathcal{E}^1_T(\lambda) = \left\{ \forall t \in [1,T], \forall (s,a):  \KL_{\hat{M}_t|M}(s,a) \leq  \frac{\beta_1(N_t(s,a),\lambda)}{N_t(s,a)} \right\}
\end{equation*}
holds with probability greater than $1-\lambda$.\\

Our second good event is that the empirical visitation frequency $N_t(s,a)$ is sufficiently well-spread, such that the number of times we have played each state-action pair up to time $T$ can be lower bounded. Namely, due to the mixture with our forcing policy $\pi_{f,t}$ in the NaS algorithm, we show in Lemma \ref{E2_high_prob_lemma} that 
\[
\mathcal{E}^2_T(\lambda)
= \left\{
\forall t \in [1,T], \, \forall (s,a) \in \mathcal{S} \times \mathcal{A}, \,
N_t(s,a) \ge \frac{1}{2} \sum_{k=1}^{\lfloor t/m \rfloor} \frac{\eta_m}{(km)^{\alpha + \beta}}
- \log \left( \frac{SA}{\lambda} \right)
\right\}.
\]
holds with probability greater than $1-\lambda$. In Section \ref{subsection_good_events_concequences} we show a variety of useful consequences of the first two good events that will be useful for the remainder of the proof.\\

Section \ref{subsection_main_concentration} is devoted to proving that the empirical visitation frequency $N_t(s,a)$ concentrates well around the sum of approximated optimal allocations so far. This intuitively makes sense, since our sampling rule used in the NaS algorithm is forcing these terms to be close due to the $\pi_t^*$ component of our policy $\pi_t$ which has induced stationary distribution $\sum_{t=1}^T\omega_t^*/T$. In particular we show that the following good event
\begin{equation*}
    \mathcal{E}_t^5(\lambda) = \left\{\forall (s,a) \in \mathcal{Z}, \, \bigg|\frac{N_t(s,a) - \sum_{k=1}^t \omega^*_{k-1}(s,a)}{t}\bigg| \leq \ell(t)\right\}.
\end{equation*}
conditioned on events $\mathcal{E}^1_T(\lambda)$ and $\mathcal{E}^2_T(\lambda)$, holds with high probability $2SA\exp\left(-t^{\frac{1}{2}}\right)$ in Proposition~\ref{prop_concentration_Nt_vs_SumOmega}. Note that, unlike previous BPI works, this concentration is not asymptotic and allows us to provide a non-asymptotic upper bound for the sample complexity in the remainder of the proof.

\paragraph{Bounding the Stopping Time:} In Section \ref{subsection_main_sequence} we show that, when these three good events hold, we can upper bound the stopping time almost surely. To do so, we first observe that by definition of our stopping time, at any time $t$ before stopping we have
\begin{equation*}
    b(t,\delta) \geq tT(M_t,N_t/t)^{-1} =  \inf_{M'\in \text{Alt}(M_t)} \sum_{s,a} N_t(s,a) \KL_{M_t|M'}(s,a)
\end{equation*}
Then, under event $\mathcal{E}^5_t$ we can replace $N_t(s,a)$ with the sum of $\omega^*_j$, up to some error. Namely
\begin{equation*}
    b(t,\delta) \gtrsim \sum_{j=1}^t  \inf_{M'\in \text{Alt}(M_t)} \sum_{s,a} \omega_j^* (s,a) \KL_{M_t|M'}(s,a).
\end{equation*}
Then, since under events $\mathcal{E}^1_t,\mathcal{E}^2_t$ we have that $M_t \approx M$ and since by definition we can write $\omega_j^*(s,a)=\argmax_{\Tilde{\omega}\in \Omega(M_j)}\inf_{M'\in \text{Alt}(M_j)} \sum_{s,a} \Tilde{\omega}(s,a) \KL_{M_j|M'}(s,a)$, we can use a careful sequence of inequalities to manipulate the sums, infimums and supremums appearing to show that 
\begin{align*}
    b(t,\delta) &\gtrsim \sum_{j=1}^t \sup_{\Tilde{\omega}\in \Omega(M)} \inf_{M'\in \text{Alt}(M)} \sum_{s,a} \Tilde{\omega}_j(s,a) KL_{M|M'}(s,a)\\
    &\approx tT(M)^{-1}
\end{align*}
This section of the proof required quite delicate analysis to ensure each step of the approach is valid and that we attain the final bound and to control the errors in the approximations used in each step. This is formalized in Theorem \ref{lemma_main_sequence}. Furthermore, since the three good events hold with high probability, we can bound the expected stopping time to attain our desired non-asymptotic upper bound, which we discuss in Appendix \ref{subsection_putting_everything_together}.

\newpage
\subsection{Good Events}\label{subsection_good_events}

We start the proof by introducing two high probability events $\mathcal{E}^1_T(\lambda)$ and $\mathcal{E}^2_T(\lambda)$. The first event, $\mathcal{E}^1_T(\lambda)$, ensures that the KL-divergence between the empirical transition probabilities and the true transition probabilities is upper bounded up to time $T$ as follows in Lemma \ref{E1_high_prob_lemma} (with probability greater than $1-\lambda$). The form here is motivated by \citep[Lemma 10, ][]{pmlr-v132-kaufmann21a}. The main step of the proof is to modify \citep[Proposition 1, ][]{jonsson2020planning} to use Doob's inequality, as done in Lemma \ref{lemma:modified_empirical_probability_jonsson} in the supplementary results. We write the proof here for completion.
\begin{lemma}\label{E1_high_prob_lemma}
Let 
\begin{equation*}
    \beta_1(n,\lambda) = (S-1) \log\!\left(e\!\left(1 + \frac{n}{S-1}\right)\!\right)
+ \log(SA/\lambda)
\end{equation*}
and let
    \begin{equation*}
    \mathcal{E}^1_T(\lambda) = \left\{ \forall t \in [1,T], \forall (s,a):  \KL_{\hat{M}_t|M}(s,a) \leq  \frac{\beta_1(N_t(s,a),\lambda)}{N_t(s,a)} \right\}
\end{equation*}
Then event $\mathcal{E}^1_T(\lambda)$ holds with probability greater than $1-\lambda$
\end{lemma}
\begin{proof}
    We simply do a union bound and then plug in our bound from Lemma \ref{lemma:modified_empirical_probability_jonsson}.
    \begin{align*}
        &\mathbb{P}(\{\mathcal{E}^1_T(\lambda)\}^C) \\
        &\leq \sum_{s,a} \mathbb{P}\left(\exists t \in [1,T], :  \KL_{\hat{M}_t|M}(s,a) >  \frac{\beta_1(N_t(s,a),\lambda)}{N_t(s,a)} \right)\\
        &=\sum_{s,a} \mathbb{P}\left(\exists t \in [1,T], :  \KL_{\hat{M}_t|M}(s,a)N_t(s,a) >   (S-1) \log\!\left(e\!\left(1 + \frac{N_t(s,a)}{S-1}\right)\!\right)
+ \log(SA/\delta) \right)\\
&\leq \sum_{s,a} \frac{\lambda}{SA}\\
&= \lambda
    \end{align*}
    The proof is then complete.
\end{proof}

We then introduce the second high probability event, $\mathcal{E}^2_T(\lambda)$ as follows in Lemma \ref{E2_high_prob_lemma}. Here $\mathcal{E}^2_T(\lambda)$ ensures that each state action pair has been visited a sufficient number of times with high probability. This is important as we are assuming that the learner does not have access to a generative model, hence the learner must actively navigate through the state-action space. This Lemma \ref{E2_high_prob_lemma} guarantees a level of exploration, which will be useful later. The proof is a modification of \citep[Lemma C.3, ][]{russomulti}, where we modify the Bernoulli result from \citep[Lemma F.4, ][]{dann2017unifying} by again using Doob's inequality instead of Markov's inequality. We write out the full proof for completion and describe the modification of the Bernoulli result in Lemma \ref{lemma:modified_bernoulli_bounded_t_event} in the supplementary results section. We denote here $m_0$ as the maximum number of steps to move between any pair of states $(s, s')$. 

\begin{lemma}\label{E2_high_prob_lemma}
For all $\lambda \in (0,1)$, $m := m_0 + 1$, define the event
\[
\mathcal{E}^2_T(\lambda)
= \left\{
\forall t \in [1,T], \, \forall (s,a) \in \mathcal{S} \times \mathcal{A}, \,
N_t(s,a) \ge \frac{1}{2} \sum_{k=1}^{\lfloor t/m \rfloor} \frac{\eta_m}{(km)^{\alpha + \beta}}
- \log \left( \frac{SA}{\lambda} \right)
\right\}.
\]
Then, for $\text{NaS}$ with $\alpha + \beta \le 1$ we have that event $\mathcal{E}^2_T(\lambda)$ holds with probability greater than $1-\lambda$.
\end{lemma}

\begin{proof}
Fix a state-action pair $z = (s,a)$ and let $Y_t = \mathds{1}_{z_t = z}$ and
$X_k = \mathds{1}_{Y_{(k-1)m+1} + \cdots + Y_{km} > 0}$ for
$k = 1, \ldots, \lfloor t/m \rfloor$.
Define $\tilde{N}_{\lfloor t/m \rfloor}(s,a) = \sum_{k=1}^{\lfloor t/m \rfloor} X_k$, and note that
$N_t(s,a) \ge \tilde{N}_{\lfloor t/m \rfloor}(s,a)$ for all $t \ge 1$.

Let $Q_k = \mathbb{P}(X_k = 1 \mid \mathcal{F}_{k-1})$, where $\mathcal{F}_k = \sigma(Y_1, \ldots, Y_{km})$.
Note that we can also write $Q_k$ as
$Q_k = \mathbb{P}(\exists n \in \{1, \ldots, m\} : Y_{(k-1)m+n} = 1 \mid \mathcal{F}_{k-1})$.
Then
\[
\begin{aligned}
Q_k
&= \mathbb{E}_{(\pi_t)}\big[ \mathbb{P}(\exists n \in \{1, \ldots, m\} : Y_{(k-1)m+n} = 1 \mid \mathcal{F}_{k-1}, (\pi_t)_{t}) \big] \\
&\ge \mathbb{E}_{(\pi_t)} \Big[ \max_{1 \le n \le m} \mathbb{P}(Y_{(k-1)m+n} = 1 \mid \mathcal{F}_{k-1}, (\pi_t)_t) \Big] \\
&\ge \mathbb{E}_{(\pi_t), z_0} \Big[ \max_{1 \le n \le m} \mathbb{P}(z_{(k-1)m+n} = z \mid z_{(k-1)m} = z_0, \mathcal{F}_{k-1}, (\pi_t)_t) \Big] \\
&= \mathbb{E}_{(\pi_t), z_0} \Big[ \max_{1 \le n \le m} \mathbb{E}_{(z_j)_{j=1}^{n-1}} \Big[ \prod_{j=0}^{n-1} \mathbb{P}(z_{(k-1)m+j+1} = z_{j+1} \mid z_{(k-1)m+j} = z_j) \big| z_1, \ldots, z_n = z \Big] \Big| z_0 \Big] \\
&\ge \mathbb{E}_{(\pi_t), z_0} \Big[ \max_{1 \le n \le m} \mathbb{E}_{(z_j)_{j=1}^{n-1}} \Big[ \prod_{j=0}^{n-1} \frac{1}{(km)^{\alpha+\beta}} P^{n}_{u}(z_{(k-1)m+j+1} = z_{j+1} \mid z_{(k-1)m+j} = z_j) \Big| z_1, \ldots, z_n = z \Big] \Big| z_0 \Big] \\
&= \mathbb{E}_{(\pi_t), z_0} \Big[ \max_{1 \le n \le m} \frac{1}{(km)^{n(\alpha+\beta)}} P^{n}_{u}(z_{(k-1)m+n} = z \mid z_{(k-1)m} = z_0) \Big| z_0 \Big] \\
&\geq \mathbb{E}_{(\pi_t), z_0} \Big[ \max_{1 \le n \le m} \frac{1}{(km)^{n(\alpha+\beta)}} \eta_n \Big| z_0 \Big] \\
&\geq \mathbb{E}_{(\pi_t), z_0} \Big[ \max_{1 \le n \le m} \frac{1}{(km)^{n(\alpha+\beta)}} \eta_m \Big| z_0 \Big] \\
&\ge \frac{\eta_m}{(km)^{\alpha+\beta}}.
\end{aligned}
\]

Hence, for $\lambda' \ge 0$ due to Lemma \ref{lemma:modified_bernoulli_bounded_t_event} we also have
\[
\mathbb{P}\left( \exists t \in [1,T] : N_t(s,a) \le \frac{1}{2} \sum_{k=1}^{\lfloor t/m \rfloor} \frac{\eta_m}{(km)^{\alpha+\beta}} - \lambda' \right)
\le \mathbb{P}\left( \exists t \in [1,T]: \tilde{N}_{\lfloor t/m \rfloor}(s,a) \le \frac{1}{2} \sum_{k=1}^{\lfloor t/m \rfloor} Q_k - \lambda' \right)
\le e^{-\lambda'}.
\]

Consequently, choosing $\lambda' = \log(SA / \lambda)$, we get that
\[
\mathbb{P}\big( \overline{\mathcal{E}_{\text{cnt}}(\lambda)} \big)
\le \sum_{s,a} \mathbb{P}\left( \exists t \in [1,T]: N_t(s,a) \le \frac{1}{2} \sum_{k=1}^{\lfloor t/m \rfloor} Q_k(s,a) - \log\left( \frac{SA}{\lambda} \right) \right)
\le \sum_{s,a} \frac{\lambda}{SA} = \lambda.
\]
\end{proof}

Subsequently, unless specified otherwise, we will fix $\lambda=1/T^2$ and drop the $\lambda$ notation in the high probability events throughout the appendix.

\begin{definition}
    Henceforth we will consider $\lambda=1/T^2$ and define $\mathcal{E}^1_T = \mathcal{E}^1_T(1/T^2)$ and $\mathcal{E}^2_T = \mathcal{E}^2_T(1/T^2)$.
\end{definition}

\newpage
\subsection{Consequences of Good Events}\label{subsection_good_events_concequences}

We show that under both good events, we have that the $1$-norm between the true transition and the empirically observed transitions (we denote the MDP with these empirical transitions $P_t$ as $\hat{M}_t$) can be upper bounded. In particular, $\mathcal{E}^3_T(\lambda)$ holds as described in Lemma \ref{E3_high_probability_event} below. Then, by the definition of our projection in equation \eqref{eqn_projection_defintiion} and the triangle inequality that, for all $(s,a)$,
\begin{equation*}
    \|
P_t^\Pi(\cdot\mid s,a)-P(\cdot\mid s,a)
\|_1
\le
2
\|
P_t(\cdot\mid s,a)-P(\cdot\mid s,a)
\|_1 .
\end{equation*}
Hence, if the estimate of the MDP transitions $P_t$ is in some ball of radius $\psi(t)/2$ from the true transitions, then we can guarantee that the projected estimate $P_t^\Pi$ (into $G_\mu(L)$) is in some ball of radius $\psi(t)$ from the true transitions. This is stated in Lemma \ref{lemma_projection_consequence}. Note again that, unless otherwise stated, we will drop the $\lambda$ notation to mean $\mathcal{E}^4_T = \mathcal{E}^4_T(1/T^2)$.

\begin{lemma}\label{E3_high_probability_event}
Consider the event
\begin{equation*}
    \mathcal{E}^3_T(\lambda) = \left\{ \forall t \in [C,T]:  \hat{M}_t \in B_{\psi(t)/2}(M)\right\},
\end{equation*}
Namely, the (projected) estimated MDP is in some $L_1$ ball of radius $\psi(t)/2$ around the true MDP, where we define
\begin{align}
B_{\psi}(M)  &= \{M': \max_{s,a}\|P(s,a)-P'(s,a)\|_1\leq \psi\}\nonumber\\
\psi(t,\lambda)&=2\sqrt{\frac{\beta_1(\phi(t,\lambda),\lambda)}{2\phi(t,\lambda)}}\label{eqn_psi_definition}\\
    \phi(t,\lambda)&= \frac{1}{2}\sum_{k=1}^{\lfloor\frac{t}{m}\rfloor} \frac{\eta_m}{(km)^{\alpha+\beta}} - \log\frac{SA}{\lambda}\nonumber\\
    C &= m\left[2m^{\alpha+\beta}\left(1-(\alpha+\beta)\right)\left(C' + \log\frac{SA}{\lambda}\right)\right]^{\frac{1}{1-(\alpha+\beta)}}+m\nonumber\\
    C' &= \max\left\{1,\frac{S-1}{e}\exp\left(\frac{1}{S} - \frac{1}{S-1}\log(SA/\lambda)\right)-e\right\}.\nonumber
\end{align}
Then, under ${\cal E}^1_T (\lambda)\cap {\cal E}^2_T(\lambda)$, we have that $\mathcal{E}^3_T(\lambda)$ holds.

\end{lemma}
\begin{proof}
By definition ${\rm KL}_{M_t|M}(s,a)={\rm KL}( P_t^\Pi(s,a), P(s,a))$. Therefore,
    under ${\cal E}_T^1$, an application of Pinsker's inequality yields
    \[
    \| P_t^\Pi(s,a)-P(s,a)\|_1^2 \leq \frac{\beta_1(N_t(s,a),\lambda)}{2N_{t}(s,a)}\quad \forall(s,a).
    \]
    Therefore
    \[
     M_t\in B_{\tilde \psi(\lambda,t)}(M),
    \]
    where $\tilde \psi(\lambda,t)$ is defined as 
    \[
    \tilde \psi(\lambda,t)=\sqrt{\frac{\beta_1(N_t(s,a),\lambda)}{2N_{t}(s,a)}}.
    \]
    Now, using also the other event ${\cal E}_{T}^2$, we can also use the fact that $N_t(s,a)$ is lower bounded by $\phi(t,\lambda)= \frac{1}{2}\sum_{k=1}^{\lfloor\frac{t}{m}\rfloor} \frac{\eta_m}{(km)^{\alpha+\beta}} - \log\frac{SA}{\lambda}$. Additionally, we note that $\tilde \psi(\lambda,t)$ is decreasing with respect to $N_t$ for 
    $$N_t(s,a) \geq \max\left\{1,\frac{S-1}{e}\exp\left(\frac{1}{S} - \frac{1}{S-1}\log(SA/\lambda)\right)-e\right\}.$$
    Combining these two facts with Lemma \ref{lemma_phi_upper_bound_lower_bound}, we have that for $t\geq C$
    \[
    \tilde \psi(\lambda,t)\leq \psi(t,\lambda)=\sqrt{\frac{\beta_1(\phi(t,\lambda),\lambda)}{2\phi(t,\lambda)}},
    \]
    and thus $M_t\in B_{\tilde \psi(\lambda,t)}(M)\subset B_{\psi(t,\lambda)}(M)$.
\end{proof}

\begin{lemma}\label{lemma_projection_consequence}
Consider the event
\begin{equation*}
    \mathcal{E}^4_T(\lambda) = \left\{ \forall t \in [C,T]:  M_t \in B_{\psi(t)}(M)\right\},
\end{equation*}

then, under ${\cal E}^1_T (\lambda)\cap {\cal E}^2_T(\lambda)$, we have that $\mathcal{E}^4_T(\lambda)$ holds.

\end{lemma}

We can now provide two results, Lemma \ref{lemma_xi_def} and Lemma \ref{cor_ball_event_concequences}, which allow us to concentrate various important quantities. In Lemma \ref{lemma_xi_def} we upper bound the distance between the exploration policy guided by the estimated optimal allocations and the true optimal policy. In Lemma \ref{cor_ball_event_concequences} we upper bound \emph{(i)} the difference between the stationary distribution over the state-actions of the current policy and the optimal allocation, \emph{(ii)} the distance between the transitions induced by the previous policy and the current policy, and \emph{(iii)} the distance between the exploration policy induced by the average estimated optimal allocations and the optimal policy. These distances are upper bounded in terms of the quantities $\xi_t$ and $\epsilon_t$. Here $\epsilon_t$ is the mixture parameter for the policy used in round $t$ of the NaS algorithm. The quantity $\xi_t$ is defined in \eqref{eqn_xi_definition}. Note that, as $t\rightarrow\infty$, $\xi_t \rightarrow 0$. Note also $\xi_t$ is a decreasing sequence for $t\geq C$. 

Lemma \ref{lemma_xi_def} is quite notation dense, so we provide some intuition here. Recall $\psi(t,\lambda)$ is defined in the definition of the third good event which is being conditioned in this lemma. The rest are MDP-dependent constants/functions which are defined in more detail in Appendix \ref{appendix_sharpness}, namely: $\underline\rho(M)$ is the smallest of the optimal state-action allocations for $M$, $c_{\text{sel}}$ describes the stability of the minimal 2-norm selection from the set of estimated optimal allocations in our algorithm, $\kappa_\Omega(M)$ is a conditioning constant that controls how sensitive the stationary distribution is to changes in the transition kernel, and the function $\varphi_M$ which describes the sharpness of $T(\cdot, M)$ around the set of optimal allocations.

\begin{lemma}\label{lemma_xi_def}
Under event $\mathcal{E}^3_T(\lambda)$, 
there exists $\lambda_0(M)>0$ such that on for all $\lambda\le\lambda_0(M)$
and all $t\le T$,
     \begin{equation*}
        ||\hat{\pi}_t^* - \pi^*||_\infty \leq \xi_t
    \end{equation*}
    where 
\begin{equation}\label{eqn_xi_definition}
\xi_t = 
\frac{(A+1)}{\rho(M)}
    \sqrt{
        2\left[
            \kappa_\Omega(M)\psi(t,\lambda)
            +
            \varphi_M\!\left(C_T(M)\psi(t,\lambda)\right)
        \right]}
\end{equation}
with
\[
C_T(M):=2S^3A(1+2\log(L))+2SA\log(L)\kappa_\Omega(M).
\]
\end{lemma}
\begin{proof}
    For proof and further discussion see Appendix \ref{appendix_sharpness}.
\end{proof}

\begin{lemma}\label{cor_ball_event_concequences}
    For $t\in [C^2,T]$, under event $\mathcal{E}^3_T(\lambda)$ we have
    \begin{align*}
        ||\omega_t - \omega^*||_1 &\leq \kappa_M \left(\frac{1}{\sqrt{t}} + \xi_{\sqrt{t}} + \epsilon_t\right)\\
        ||P_t - P_{t-1}||_\infty &\leq 2 \left(\frac{1}{\sqrt{t-1}} + \xi_{\sqrt{t-1}} + \epsilon_{t-1}\right)\\
        ||\pi^*_t - \pi^*||_\infty &\leq \frac{1}{\sqrt{t}} + \xi_{\sqrt{t}}
    \end{align*}
\end{lemma}
\begin{proof}
    This follows similarly to \citep[Corollary C.2][]{russomulti}. For completion we rewrite the minor modification to their proof here. \\
    
    First, recall the definitions $\pi^\star(a|s) = \omega^\star(a|s)/\sum_b \omega^\star(b|s)$ and $\hat{\pi}^\star_t(a|s) = \omega^\star_t(a|s)/\sum_b \omega^\star_t(b|s)$. Under $\mathcal{E}^{3}_{T}(\lambda)$ using Lemma \ref{lemma_aux_1}, there exists an MDP-specific constant $\kappa_M$ such that
\begin{align*}
    \|\omega_t - \omega^\star\|_1 &\leq \kappa_M \|P_t - P_{\pi^\star}\|_\infty, \\
    &\leq \kappa_M \|\pi_t - \pi^\star\|_\infty, \\
    &\leq \kappa_M \left[\|\pi_t^\star - \pi^\star\|_\infty + \epsilon_t\right].
\end{align*}

Then, for $\sqrt{t}\geq C$, we have
\[
\|\pi_t^\star - \pi^\star\|_\infty \leq \left\|\frac{1}{t} \sum_{k=1}^t (\hat{\pi}_t^\star - \pi^\star)\right\| \leq \frac{\sqrt{t}}{t} + \xi_{\sqrt{t}} = \frac{1}{\sqrt{t}} + \xi_{\sqrt{t}},
\]
which concludes the proof of the first and third inequalities of the Lemma. For the second part, we have that
\begin{align*}
\|P_t - P_{t-1}\|_\infty &\leq \|P_t - P_{\pi^\star}\|_\infty + \|P_{\pi^\star} - P_{t-1}\|_\infty, \\
&\leq 2 \left(\frac{1}{\sqrt{t-1}} + \xi_{\sqrt{t-1}} + \epsilon_{t-1}\right).
\end{align*}
\end{proof}

Finally, by recalling our definition for $\epsilon_t = 1/\max\{1,N_t(s_t)\}^\alpha$ we can plug this into Lemma \ref{E2_high_prob_lemma} to get the following bound on $\epsilon_t$.

\begin{corollary}\label{cor_epsilon_bound}
 For all $t\in [C,T]$, under $\mathcal{E}^2_{T}(\lambda)$, we have
    \begin{equation*}
        \epsilon_t \leq \min(1,\kappa_t(\lambda)), \quad \text{where} \quad \kappa_t(\lambda) = \frac{2^\alpha}{A^\alpha\left(\sum_{k=1}^{\lfloor\frac{t}{m}\rfloor} \frac{\eta_m}{(km)^{\alpha+\beta}} - 2\log\frac{SA}{\lambda}\right)^\alpha}.
    \end{equation*}
\end{corollary}

\newpage
\subsection{Non-asymptotic distance to approximated optimal visitation frequency}\label{subsection_main_concentration}

In this section we provide a key concentration result, which allows us to make the initial steps towards a non-asymptotic upper bound for Nas. In particular, during learning there will be a stationary distribution $\omega_t$ induced by the policy used $\pi_t$ at time $t$. Intuitively, if we are playing policies with these stationary distributions, we would hope that our actual visitation frequencies become close to the sum of these distributions. Namely we hope that $$\bigg|N_t(s,a) - \sum_{k=1}^t \omega_{k-1}(s,a)\bigg|$$ is sufficiently small. In this section, we will formalize this and show that it occurs with high probability, conditioned on our two good events holding $\mathcal{E}^1_T \cap\mathcal{E}^2_T$. We start by providing some more general results in Proposition \ref{prop_main_concentration1} and \ref{prop_main_concentration2}. We then provide the desired bound on the distance in Proposition \ref{prop_concentration_Nt_vs_SumOmega} in terms of the quantity $\ell_t$. Importantly, note that as $t\rightarrow\infty$ we have $\ell(t) \rightarrow 0$ under the event $\mathcal{E}_{T}^2$, since $\epsilon_t\rightarrow 0, \, \xi_t \rightarrow 0$ and $\lim_{t\rightarrow\infty}\Tilde{L}_t < \infty$. 

We begin with the following proposition, inspired by \citep[Proposition C.3]{russomulti}. In contrast to prior results, we introduce a time-dependent concentration radius, which is essential for our non-asymptotic analysis. At this stage, the result is stated in a general form; its role in establishing our final sample complexity guarantees will become clear later, where we relate it to the estimated optimal allocations. For clarity of exposition, Proposition~\ref{prop_main_concentration1} focuses only on the contribution of terms beyond round $\sqrt{t}$. The treatment of the full sum, including the initial rounds, is deferred to Proposition~\ref{prop_main_concentration2}.

\begin{proposition} \label{prop_main_concentration1}
Let $t \in [C^2,T]$, then for all non-negative functions $f$ such that $f:\mathcal{Z}\rightarrow[0,1]$ we have
    \begin{equation*}
        \mathbb{P}\left(\bigg|\frac{\sum_{k=\sqrt{t}+1}^tf(z_k) - \omega_{k-1}(f)}{t}\bigg| \geq \ell'(t) \Bigg| \mathcal{E}_T^1\cap\mathcal{E}^2_T\right) \leq 2\exp\left(-t^{\frac{1}{2}}\right)
    \end{equation*}
where we define 
\begin{align}
    \omega_t(f) &= \mathbb{E}_{z\sim \omega_t}[f(z)]\nonumber\\
    \ell'(t) &= 2\Tilde{L}_t t^{-\frac{1}{4}} + 2\Tilde{L}_{t} (\kappa_M + \Tilde{L}_{t}) \left[ t^{-\frac{1}{4}} + \xi_{t^{\frac{1}{4}}} + \epsilon_{\sqrt{t}} \right] + \frac{ 2\Tilde{L}_{t}}{t}\nonumber\\
    \Tilde{L}_t &= \sup\, \left\{L(\epsilon,\pi) \,:\, |\epsilon|\leq  \min\{1,\kappa_{\sqrt{t}}(\lambda)\}, ||\pi^* - \pi||_\infty\leq t^{-\frac{1}{4}} + \xi_{t^{\frac{1}{4}}}\right\}\label{eqn_definition_L_tilde}
\end{align}
\end{proposition}

\begin{proof} This will follow along with the proof of \citep[Proposition C.3]{russomulti} and we will denote $\omega^*(f) = \mathbb{E}_{z\sim \omega^*}[f(z)]$ and $\omega_t(f) = \mathbb{E}_{z\sim \omega_t}[f(z)]$. Using \citep[Lemma 23][]{al2021navigating} we can rewrite the inside of the absolute value of our expression using $\hat{f}_k$, the solution to the Poisson equation $f(\cdot) - \omega_{k-1}(f) = [\hat{f}_k - P_k\hat{f}_k](\cdot)$.This is well-defined since the sampling policy at each round is a mixture with a forcing policy (e.g., the uniform policy), which ensures that the induced Markov chain is ergodic. Consequently, the Poisson equation admits a solution, allowing us to proceed with the decomposition.. In particular, we can write

\begin{equation*}
    \frac{\sum_{k=\sqrt{t}+1}^tf(z_k) - \omega_{k-1}(f)}{t} = \frac{1}{t}\sum_{k=\sqrt{t}}^t \hat{f}_{k-1}(z_k) - P_{k-1}\hat{f}_{k-1}(z_k) = M_t+C_t+R_t
\end{equation*}
where
\begin{align*}
    M_t &:= \frac{\sum_{k = \sqrt{t} + 1}^{t} \hat{f}_{k-1}(z_k) - P_{k-1} \hat{f}_{k-1}(z_{k-1})}{t},\\
C_t &:= \frac{\sum_{k = \sqrt{t} + 1}^{t} P_k \hat{f}_k(z_k) - P_{k-1} \hat{f}_{k-1}(z_k)}{t},\\
R_t &:= \frac{P_{\sqrt{t}} \hat{f}_{\sqrt{t}}(z_{\sqrt{t}}) - P_t \hat{f}_t(z_t)}{t}.
\end{align*}
Hence, we try to bound these three terms in order to complete the proof.\\

Following similar ideas from \citep[Proposition C.3][]{russomulti}, we can bound the terms in this decomposition and relate them to the quantity $L(\epsilon,\pi)$. This quantity captures the dependence of our bounds on the deviation of the sampling policy from the target allocation, as well as the level of forced exploration. For now, we treat it as an abstract complexity term; its precise definition and interpretation, along with non-asymptotic bounds on its value can be found later in Appendix \ref{subsection_bounding_L_tilde}.

these we can first note that for $k \geq \sqrt{t}$ and using Corollary \ref{cor_epsilon_bound} we have

\begin{equation*}
\epsilon_k \leq \min\{1,\kappa_k(\lambda)\}   \leq \min\{1,\kappa_{\sqrt{t}}(\lambda)\}
\end{equation*}

and from Corollary \ref{cor_ball_event_concequences} we have that for $k\geq \sqrt{t}\geq C$
\begin{equation*}
    ||\pi^*_k - \pi^*||_\infty \leq \frac{1}{\sqrt{k}} + \xi_{\sqrt{k}} \leq t^{-\frac{1}{4}} + \xi_{t^{\frac{1}{4}}}.
\end{equation*}

Then for $k \geq \sqrt{t}$, we can relate this to the quantity $L_k = L(\epsilon_t,\pi_t)$
\begin{equation}\label{eqn_L_tilde_definition}
    L_k \leq \Tilde{L}_t \leq \sup\, \left\{L(\epsilon,\pi) \,:\, |\epsilon|\leq  \min\{1,\kappa_{\sqrt{t}}(\lambda)\}, ||\pi^* - \pi||_\infty\leq t^{-\frac{1}{4}} + \xi_{t^{\frac{1}{4}}}\right\}
\end{equation}
where we define $\Tilde{L}_t := \sup_{k\in [\sqrt{t},t]} L_k$. See section \ref{subsection_bounding_L_tilde} for the relevant definitions and a more thorough analysis on the behaviour of $\Tilde{L}_t$.

\paragraph{Bounding $M_t$:} We can use the fact that $tM_t$ is a martingale such that for $S_k = kM_k$ we have $|S_k - S_{k-1}| \leq 2L_{k-1}\leq 2\Tilde{L}_{k-1}$, and hence for $\sqrt{t}\geq C$ we can apply the Azuma-Hoeffding inequality to attain
\begin{equation*}
    \mathbb{P}(|M_t| \geq 2\Tilde{L}_t g(t)) \leq 2\exp\left(-t(g(t))^2\right).
\end{equation*}
Hence, by setting $g(t) = t^{-\frac{1}{4}}$ we have
\begin{equation*}
    \mathbb{P}(|M_t| \geq 2\Tilde{L}_t t^{-\frac{1}{4}}) \leq 2\exp\left(-t^{\frac{1}{2}}\right).
\end{equation*}

\paragraph{Bounding $C_t$:} For $\sqrt{t}\geq C$ we have
\begin{align*}
C_t &\leq (1/t) \sum_{k = \sqrt{t} + 1}^{t} L_k \left[\|\omega_k - \omega_{k-1}\| + L_{k-1} \|P_k - P_{k-1}\|_\infty \right], \\
&\leq (1/t) \sum_{k = \sqrt{t} + 1}^{t} L_k \left[ \|\omega_k - \omega_{k-1}\| + 2L_{k-1} \left(\frac{1}{\sqrt{k-1}} + \xi_{\sqrt{k-1}} + \epsilon_{k-1}\right) \right], \\
&\leq (1/t) \sum_{k = \sqrt{t} + 1}^{t} L_k \left[ \|\omega_k - \omega^\star\| + \|\omega_{k-1} - \omega^\star\| + 2L_{k-1} \left(\frac{1}{\sqrt{k-1}} + \xi_{\sqrt{k-1}} + \epsilon_{k-1}\right) \right], \\
&\leq (1/t) \sum_{k = \sqrt{t} + 1}^{t} 2L_k \left[ \kappa_M \left(\frac{1}{\sqrt{k-1}} + \xi_{\sqrt{k-1}} + \epsilon_{k-1}\right) + L_{k-1}\left(\frac{1}{\sqrt{k-1}} + \xi_{\sqrt{k-1}} + \epsilon_{k-1}\right) \right], \\
&\leq 2(1/t) \sum_{k = \sqrt{t} + 1}^{t} \Tilde{L}_{t} (\kappa_M + \Tilde{L}_{t}) \left(\frac{1}{\sqrt{k-1}} + \xi_{\sqrt{k-1}} + \epsilon_{k-1}\right), \\
&\leq 2\Tilde{L}_{t} (\kappa_M + \Tilde{L}_{t}) \left[ t^{-\frac{1}{4}} + \xi_{t^{\frac{1}{4}}} + \epsilon_{\sqrt{t}} \right].
\end{align*}
Here the second inequality comes from Corollary \ref{cor_ball_event_concequences}. The third inequality uses the triangle inequality. The fourth inequality uses Corollary \ref{cor_ball_event_concequences} again. The penultimate line comes from the definition of $\Tilde{L}_t$ in \eqref{eqn_L_tilde_definition}.

\paragraph{Bounding $R_t$:} We have, again by similar arguments to \citep[Proposition C.3]{russomulti} and \citep[Corollary 2]{al2021navigating}, that
\begin{align*}
    |R_t| &\leq \frac{\|f\|_\infty (L_{\sqrt{t}}+L_t)}{t}\\
    &\leq \frac{ (L_{\sqrt{t}}+L_t)}{t}\\
    &\leq \frac{ 2\Tilde{L}_{t}}{t}
\end{align*}

\end{proof}

We can now additionally consider the first $\sqrt{t}$ terms that were excluded in the sum of the previous proposition.

\begin{proposition} \label{prop_main_concentration2}
    Let $t \in [C^2,T]$, then for all functions $f$ such that $f:\mathcal{Z}\rightarrow[0,1]$ we have
    \begin{equation*}
        \mathbb{P}\left(\bigg|\frac{\sum_{k=1}^t f(z_k) - \omega_{k-1}(f)}{t}\bigg| \geq \ell(t) \Bigg| \mathcal{E}_T^1\cap\mathcal{E}_{T}^2\right) \leq 2\exp\left(-t^{\frac{1}{2}}\right)
    \end{equation*}
where 
\begin{align*}
    \ell''(t) &= t^{-\frac{1}{2}} + 2\Tilde{L}_t t^{-\frac{1}{4}} + 2\Tilde{L}_{t} (\kappa_M + \Tilde{L}_{t}) \left[ t^{-\frac{1}{4}} + \xi_{t^{\frac{1}{4}}} + \epsilon_{\sqrt{t}} \right]. + \frac{ 2\Tilde{L}_{t}}{t}\\
    \Tilde{L}_t &\leq \sup\, \left\{L(\epsilon,\pi) \,:\, |\epsilon|\leq  \min\{1,\kappa_{\sqrt{t}}(\lambda)\}, ||\pi^* - \pi||_\infty\leq t^{-\frac{1}{4}} + \xi_{t^{\frac{1}{4}}}\right\}
\end{align*}

\end{proposition}
\begin{proof}
    We can first rewrite the term of interest as
    \begin{equation*}
        \frac{\sum_{k=1}^t f(z_k) - \omega_{k-1}(f)}{t} = \frac{\sum_{k=1}^{\sqrt{t}} f(z_k) - \omega_{k-1}(f)}{t} + \frac{\sum_{k=\sqrt{t}+1}^t f(z_k) - \omega_{k-1}(f)}{t}.
    \end{equation*}

    The second term here can be bounded using Proposition \ref{prop_main_concentration1}.  Then, since $f(z_k)\in[0,1]$ and $\omega_{k-1}(f) \in [0,1]$ we can immediately bound the first term

    \begin{equation*}
         \frac{\sum_{k=1}^{\sqrt{t}} f(z_k) - \omega_{k-1}(f)}{t} \leq \frac{\sum_{k=1}^{\sqrt{t}} 1}{t} \leq t^{-\frac{1}{2}}.
    \end{equation*}
    The proof is then complete.    
\end{proof}

We can then write this explicitly in the form that we would like as follows in terms of the visitation frequency, as well as the approximated optimal allocations $\omega_t^*$. Note, here have additionally taken a union bound over the set of all state-action pairs.
\begin{proposition} \label{prop_concentration_Nt_vs_SumOmega}
        Let $t \in [C^2,T]$, then we have
    \begin{equation*}
        \mathbb{P}\left(\exists (s,a) \in \mathcal{Z}, \, \bigg|\frac{N_t(s,a) - \sum_{k=1}^t \omega^*_{k-1}(s,a)}{t}\bigg| \geq \ell(t) \Bigg| \mathcal{E}_T^1\cap\mathcal{E}_{T}^2\right) \leq 2SA\exp\left(-t^{\frac{1}{2}}\right)
    \end{equation*}
where 
\begin{align*}
    \ell(t) &= 2t^{-\frac{1}{2}} + 2\Tilde{L}_t t^{-\frac{1}{4}} + 2\Tilde{L}_{t} (\kappa_M + \Tilde{L}_{t}) \left[ t^{-\frac{1}{4}} + \xi_{t^{\frac{1}{4}}} + \epsilon_{\sqrt{t}} \right]. + \frac{ 2\Tilde{L}_{t}}{t} + \frac{\sum_{k=\sqrt{t}}^{t-2}\epsilon_k}{t}\\
    \Tilde{L}_t &\leq \sup\, \left\{L(\epsilon,\pi) \,:\, |\epsilon|\leq  \min\{1,\kappa_{\sqrt{t}}(\lambda)\}, ||\pi^* - \pi||_\infty\leq t^{-\frac{1}{4}} + \xi_{t^{\frac{1}{4}}}\right\}
\end{align*}

\end{proposition}

\begin{proof}
    We firstly have from the above proposition, by choosing $f$ to be the indicator that the state and action in round $t$, $z_t$, is equal to $s$ and $a$, that 
        \begin{equation}\label{eqn_Nt_vs_wk}
        \mathbb{P}\left(\exists (s,a) \in \mathcal{Z}, \, \bigg|\frac{N_t(s,a) - \sum_{k=1}^t \omega_{k-1}(s,a)}{t}\bigg| \geq \ell''(t) \Bigg| \mathcal{E}_T^1\cap\mathcal{E}_{T}^2\right) \leq 2SA\exp\left(-t^{\frac{1}{2}}\right).
    \end{equation}
    What is left is to compare the empirical visitation frequency with the sum of approximated optimal allocations up to time $t$. To do so, we firstly note that, by Lemma \ref{lemma_aux_1}, we can almost surely bound the differences for all $t \geq 1$
    \begin{align*}
         \|\omega_t - \omega^*_t\|_\infty &\leq \|P_{\pi_t} - P_{\pi^*_t}\|_{\infty} \\
         &\leq \|\pi_t - \pi_t^*\|_{\infty}\\
         &\leq \epsilon_t
    \end{align*}
    Hence we have that 
    \begin{equation}\label{eqn_diff_sum_w_vs_wstar}
        \bigg|\frac{\sum_{k=1}^t \omega_{k-1}(s,a) - \omega^*_{k-1}(s,a)}{t}\bigg| \leq \frac{\sqrt{t}}{t} + \frac{\sum_{k=\sqrt{t}}^{t-2}\epsilon_k}{t}.
    \end{equation}
    As a consequence, we can plug this bound from equation \eqref{eqn_diff_sum_w_vs_wstar} into the concentration from equation \eqref{eqn_Nt_vs_wk} using the triangle inequality to obtain the desired final concentration.
\end{proof}
Note that $\sum_{k=\sqrt{t}}^{t-2}\epsilon_k/t$ can be shown to be sublinear in $t$ using the bounds from Corollary \ref{cor_epsilon_bound} and Lemma \ref{lemma_phi_upper_bound_lower_bound}. We will have a more thorough discussion of the terms in this concentration elsewhere.

\subsubsection{Attributes of the radius term}\label{subsection_bounding_L_tilde}
The radius term $\ell(t)$ in Proposition \ref{prop_concentration_Nt_vs_SumOmega} can be difficult to digest. Hence, we comment on some it's attributes, specifically we show that $\Tilde{L}_t<\infty$ is bounded as $t\rightarrow\infty$ in Lemma \ref{lemma_L_tilde_asymp_bounded}. Then in Lemma \ref{lemma_non_asymp_upper_bound_L_tilde}, we provide a non-asymptotic upper bound for $\tilde{L}_t$. We also provide a looser but more intuitive bound for $\Tilde{L}_t$ in \eqref{eqn_nicer_L_tilde_upper_bound}. 

Firstly, we will revisit the definitions of the following terms here to help with interpretation of the results.

\begin{definition}
     Firstly let $r_0$ be the Ergodicity constant such that
    \begin{equation*}
        P^r_u(z) > 0\quad \forall r \geq r_0
    \end{equation*}
    and let the constant
    $$\sigma_u := \min_{z',z \in \mathcal{Z}} P_u^{r_0}(z'|z)/ \omega_u(z')$$

    Then we define
    
\begin{align*}
 \sigma(\epsilon, \pi) &:= \left([(1 - \epsilon) A \min_{s,a} \pi(a|s)]^{r_0} + \epsilon^{\frac{r_0(\alpha+\beta)}{\alpha}}\right) \sigma_u. \\
\theta(\epsilon, \pi) &:= 1 - \sigma(\epsilon, \pi) \\
C(\epsilon, \pi) &:= 2/\theta(\epsilon, \pi) \\
\rho(\epsilon, \pi) &:= \theta(\epsilon, \pi)^{1/r_0} \\
L(\epsilon, \pi) &:= (1 - \rho(\epsilon, \pi))^{-1} C(\epsilon, \pi) = 2\left(\left(1-\theta(\epsilon, \pi)^{1/r_0}\right)\theta(\epsilon, \pi)\right)^{-1}\\
        L_t &:=L(\epsilon_t,\pi_t)\\
        \Tilde{L}_t &:= \sup_{k\in [\sqrt{t},t]} L_k
 \end{align*}
\end{definition}

Now, we can show the boundedness of $\Tilde{L}_t<\infty$ as $t\rightarrow \infty$. We show in the proof that this quantity will not explode as long as we consider only non-trivial cases where there is a unique optimal policy.

\begin{proposition}
    For all $t \geq 1$, under events $\mathcal{E}_t^1 \cap \mathcal{E}_t^2$, we have boundedness with
    $$L_t <\infty.$$
\end{proposition}
\begin{proof}
    For this, we can simply ensure $\sigma_t \in (0,1)$. We can show this with the following two statements.

    \textbf{Statement 1:} $[(1 - \epsilon_t) A \min_{s,a} \pi_t(a|s)]^{r_0} + \epsilon_t^{\frac{r_0(\alpha+\beta)}{\alpha}} \in (0,1)$. Which can be shown as follows, using that $\epsilon_t \in (0,1)$.
    \begin{align*}
        &[(1 - \epsilon_t) A \min_{s,a} \pi_t(a|s)]^{r_0} + \epsilon_t^{\frac{r_0(\alpha+\beta)}{\alpha}}\\
        & \leq [(1 - \epsilon_t) ]^{r_0} + \epsilon_t^{\frac{r_0(\alpha+\beta)}{\alpha}}\\
        &\leq (1 - \epsilon_t)  + \epsilon_t^{\frac{r_0(\alpha+\beta)}{\alpha}}\\
        &<1
    \end{align*}
    Similarly, we can show the lower bound since
        \begin{align*}
        &[(1 - \epsilon_t) A \min_{s,a} \pi_t(a|s)]^{r_0} + \epsilon_t^{\frac{r_0(\alpha+\beta)}{\alpha}}\\
        & \geq \epsilon_t^{\frac{r_0(\alpha+\beta)}{\alpha}}\\
        &>0 .
    \end{align*}

    \textbf{Statement 2:} $\sigma_u \in (0,1]$. See statement 2 of Lemma \ref{lemma_L_tilde_asymp_bounded}
    
\end{proof}

\begin{lemma}\label{lemma_L_tilde_asymp_bounded}
    Let 
    $$\sigma_0 = [ A \min_{s,a} \pi^*(a|s)]^{r_0}  \sigma_u $$
    and $\theta_0 = 1-\sigma_0$. Then, under events $\cap_{t\geq 1} \mathcal{E}_t^1 \cap \mathcal{E}_t^2$ and Assumption \ref{assumption_w_0_condition} we have 
    $$\limsup_{t\rightarrow\infty} \Tilde{L}_t =  2\left(\left(1-\theta_0^{1/r_0}\right)\theta_0\right)^{-1} < \infty.$$
\end{lemma}
\begin{proof}
    Note that, since $\kappa_{\sqrt{t}}\rightarrow0$ as $t\rightarrow\infty$, 
    we have $\min\{1,\kappa_{\sqrt{t}}(\lambda)\} \rightarrow 0$ as $t\rightarrow\infty$. 
    Hence, in the conditions of $\Tilde{L}_t$ in it's definition, we have $\epsilon\rightarrow0$ as $t\rightarrow\infty$.\\

    Similarly, we have that $t^{-\frac{1}{4}} + \xi_{t^{\frac{1}{4}}} \rightarrow 0$ as $t\rightarrow\infty$, hence in the conditions of $\Tilde{L}_t$ in it's definition, we have $\pi \rightarrow \pi^*$.\\

    Putting these two observations together, we have 
    \begin{align*}
        \limsup_{t\rightarrow\infty} \Tilde{L}_t &= \sup\left\{L(\epsilon,\pi) \,:\, |\epsilon|\leq  0, ||\pi^* - \pi||_\infty\leq 0\right\}\\
        &= L(0,\pi^*)\\
        &= 2\left(\left(1-\theta_0^{1/r_0}\right)\theta_0\right)^{-1}\\
        &< \infty
    \end{align*}

Here the final line of boundedness comes from the fact that $\theta_0 \in (0,1)$, which is a consequence of $\sigma_0 \in (0,1)$. To prove $\sigma_0 \in (0,1)$, we can use the following two statements.

\textbf{Statement 1:} $A\min_{s,a}\pi^*(a|s) \in (0,1)$. Under Assumption \ref{assumption_w_0_condition} and the auxiliary Lemma \ref{lemma_aux_1}, we know that the optimal weights $w^*>0$ are strictly greater than zero. Hence, by the definition of the optimal policy, we have $\min_{s,a}\pi^*(a|s)>0$. Now, the strict upper bound $A\min_{s,a}\pi^*(a|s) < 1$ is also true as due to the uniqueness assumption on our optimal policy $\pi^*(a|s)$. This is because $A\min_{s,a}\pi^*(a|s) = 1$ only occurs when $\pi^*(a|s) = 1/A \quad\forall a,s$. However, when this uniform policy is optimal in a communicating MDP, we have that any policy is optimal and therefore $\pi^*$ is not unique (see Lemma \ref{aux_lemma_pi_u_unique}).

\textbf{Statement 2:} $\sigma_u \in (0,1]$. The upper bound is clear by contradiction, namely suppose that 
$$\sigma_u  = \min_{z',z \in \mathcal{Z}} P_u^{r_0}(z'|z)/ \omega_u(z')>1$$
Then, this tells us that
\begin{align*}
    \forall z,z' \quad P_u^{r_0}(z'|z) &> \omega_u(z')\\
    \Longrightarrow \sum_{z'}P_u^{r_0}(z'|z) &> \sum_{z'}\omega_u(z')\\
    \Longrightarrow 1 &> 1
\end{align*}
which is a contradiction. The lower bound is also clear by the definition of $r_0$ (i.e. the first index in which every entry of $P_u^{r_0}(z'|z)>0$ is strictly positive).
\end{proof}

\begin{lemma}\label{lemma_non_asymp_upper_bound_L_tilde}
    For all $t \geq 1$, under events $\mathcal{E}_t^1 \cap \mathcal{E}_t^2$, we have the upper bound
    \begin{equation*}
        \tilde{L}_t \leq 2r_0 \left[\sigma^{(2)}\left(1-\sigma^{(1)}\right)\right]
    \end{equation*}
    where we define
    \begin{align*}
        \sigma^{(1)} &= 2^{r_0-1}\sigma_u(1-\epsilon_t)^{r_0}A^{r_0}\left(\left[\min_{s,a}\pi^*(a|s)\right]^{r_0} + \left[t^{-\frac{1}{4}}+\xi_{t^{\frac{1}{4}}}\right]^{r_0}\right)+ \sigma_u\epsilon_{\sqrt{t}}^{\frac{r_0(\alpha+\beta)}{\alpha}},\\
        \sigma^{(2)} &= \sigma_u (1-\epsilon_{\sqrt{t}})^{r_0}A^{r_0}\left[\min_{s,a}\pi^*(a|s) - t^{-\frac{1}{4}}-\xi_{t^{\frac{1}{4}}}\right]^{r_0} + \sigma_u\epsilon_{t}^{\frac{r_0(\alpha+\beta)}{\alpha}}.
    \end{align*}
\end{lemma}
\begin{proof}
    Recall that $\tilde{L}_t  = \sup_{k\in[\sqrt{t},t]} L_k = L(\epsilon_k,\pi_k)$. By observing that by definition of $L$ and $\theta$, and since $1-(1-x)^y\geq xy$ for any $x,y\in (0,1)$, we have
    \begin{align}
        L(\epsilon,\pi) &\leq 2\left[\left(1-(1-\sigma)^\frac{1}{r_0}\right)(1-\sigma)\right]^{-1}\nonumber\\
        &\leq 2\left[\frac{\sigma}{r_0}(1-\sigma)\right]^{-1}\label{eqn_L_simplification}
    \end{align}

All that is left is to upper and lower bound $\sigma_k$ from times $\sqrt{t}$ to $t$. Let's start with the upper bound.
\begin{align*}
    \sup_{k\in[\sqrt{t},t]} \sigma(\epsilon_k,\pi_k) &\leq \left(\left[(1-\epsilon_t)A\min_{s,a} (\pi^*(a|s) + t^{-\frac{1}{4}} + \xi_{t^{\frac{1}{4}}})\right]^{r_0} + \epsilon_{\sqrt{t}}^{\frac{r_0(\alpha+\beta)}{\alpha}}\right)\sigma_u\\
    &\leq \sigma_u (1-\epsilon_t)^{r_0}A^{r_0}\left[\min_{s,a}\pi^*(a|s) + t^{-\frac{1}{4}}+\xi_{t^{\frac{1}{4}}}\right] + \sigma_u\epsilon_{\sqrt{t}}^{\frac{r_0(\alpha+\beta)}{\alpha}}\\
    &\leq 2^{r_0-1}\sigma_u(1-\epsilon_t)^{r_0}A^{r_0}\left(\left[\min_{s,a}\pi^*(a|s)\right]^{r_0} + \left[t^{-\frac{1}{4}}+\xi_{t^{\frac{1}{4}}}\right]^{r_0}\right)+ \sigma_u\epsilon_{\sqrt{t}}^{\frac{r_0(\alpha+\beta)}{\alpha}}\\
\end{align*}
where the penultime inequality comes from the fact that for $x,y>0$ and $z\geq 1$ we have $(x+y)^z \leq 2^{z-1}(x^z+y^z)$.\\

Now we provide a lower bound
\begin{align*}
    \inf_{k\in[\sqrt{t},t]} \sigma(\epsilon_k,\pi_k) &\geq \left(\left[(1-\epsilon_{\sqrt{t}})A\min_{s,a} (\pi^*(a|s) - t^{-\frac{1}{4}} - \xi_{t^{\frac{1}{4}}})\right]^{r_0} + \epsilon_{t}^{\frac{r_0(\alpha+\beta)}{\alpha}}\right)\sigma_u\\
    &\leq \sigma_u (1-\epsilon_{\sqrt{t}})^{r_0}A^{r_0}\left[\min_{s,a}\pi^*(a|s) - t^{-\frac{1}{4}}-\xi_{t^{\frac{1}{4}}}\right]^{r_0} + \sigma_u\epsilon_{t}^{\frac{r_0(\alpha+\beta)}{\alpha}}\\
\end{align*}
Plugging the upper and lower bound into equation \eqref{eqn_L_simplification} gives us the desired result.
\end{proof}

\paragraph{Final simple upper bound:} Finally, we can simplify the upper bound in Lemma \ref{lemma_non_asymp_upper_bound_L_tilde} to get the following bound on $\Tilde{L}_t$ shown in equation \eqref{eqn_nicer_L_tilde_upper_bound}
\begin{align}
        \tilde{L}_t &\leq 2r_0 \left[\sigma^{(2)}\left(1-\sigma^{(1)}\right)\right]\nonumber\\
        &\leq 2r_0 \sigma^{(2)}\nonumber\\
        &\leq 2r_0\left[\sigma_u (1-\epsilon_{\sqrt{t}})^{r_0}A^{r_0}\left[\min_{s,a}\pi^*(a|s) - t^{-\frac{1}{4}}-\xi_{t^{\frac{1}{4}}}\right]^{r_0} + \sigma_u\epsilon_{t}^{\frac{r_0(\alpha+\beta)}{\alpha}}\right]\nonumber\\
        &\leq 2r_0\sigma_u\left[\left(A\min_{s,a}\pi^*(a|s)\right)^{r_0} + \epsilon_{t}^{\frac{r_0(\alpha+\beta)}{\alpha}}\right] \label{eqn_nicer_L_tilde_upper_bound}
\end{align}

\newpage
\subsection{Main Sequence of inequalities} \label{subsection_main_sequence}

We can now construct a non-asymptotic, high-probability upper bound for the stopping time. The analysis is dependent on the event that our main concentration event from Proposition \ref{prop_concentration_Nt_vs_SumOmega} holds. Hence we will denote another good event as
\begin{equation*}
    \mathcal{E}_t^5 = \left\{\forall (s,a) \in \mathcal{Z}, \, \bigg|\frac{N_t(s,a) - \sum_{k=1}^t \omega^*_{k-1}(s,a)}{t}\bigg| \leq \ell(t)\right\}.
\end{equation*}

We can then prove the following Theorem using a sequence of inequalities, which will allow us to upper bound the stopping time when all good events hold.

\begin{theorem}\label{lemma_main_sequence}
    When running NaS at any time $t$ in which we have not yet stopped, namely for any $t<\tau_\delta$, we have under event $\mathcal{E}_t^1\cap\mathcal{E}_t^2\cap\mathcal{E}_t^5$  that
\begin{align*}
    b(t,\delta) &\geq (t-\sqrt{t}) T_o(M)^{-1}- \ell_1 - \ell_2 -\ell_3\\
    \intertext{where}
    \ell_1 &= 2SA t \, \ell(t)\log L \\
    \ell_2 &=    2S^2A(1+2\log L)F(t)\\
    \ell_3 &=   S^2A(1+2\log L) F(t)\\
    \end{align*}
    And we define the functions
    \begin{align*}
    F(t)=  & 3\frac{\sqrt{16m^{\alpha+\beta}\left(1-(\alpha+\beta)\right)\cdot \beta_1\left(\frac{\eta_mt}{2m}- \log\frac{SA}{\lambda},\lambda\right)}}{1+(\alpha+\beta)} \left(\frac{t+1}{2m}\right)^{(1+\alpha+\beta)/2}\nonumber\\
    &\quad\qquad+6\max\{J,C\}\\
    J&= 2m\left(2+2m^{\alpha+\beta}\left(1-(\alpha+\beta)\right)\log(SA/\lambda)+1\right)^{\frac{1}{1-(\alpha+\beta)}}\\
    \ell(t) &= t^{-\frac{1}{2}} + 2\Tilde{L}_t t^{-\frac{1}{4}} + 2\Tilde{L}_{t} (\kappa_M + \Tilde{L}_{t}) \left[ t^{-\frac{1}{4}} + \xi_{t^{\frac{1}{4}}} + \epsilon_{\sqrt{t}} \right]. + \frac{ 2\Tilde{L}_{t}}{t}
\end{align*}
\end{theorem}
\begin{proof}
Suppose that the NaS algorithm has not yet stopped at time $t$. Then, by definition of our stopping time, we have
\begin{align}
    \beta(t,\delta) &\geq tT(M_t,N_t/t)^{-1} \\
    &=  t\inf_{M'\in \text{Alt}(M_t)} \sum_{s,a} \frac{N_t(s,a)}{t} \KL_{M_t|M'}(s,a)\\
    &=  \inf_{M'\in \text{Alt}(M_t)} \sum_{s,a} N_t(s,a) KL_{M_t|M'}(s,a)\\
    \intertext{Since $N_t(s,a) \approx \sum_{j=1}^t \omega_j (s,a)$, we have}
    &\geq  \inf_{M'\in \text{Alt}(M_t)} \sum_{s,a} \left(\sum_{j=1}^t \omega_j (s,a)\right) KL_{M_t|M'}(s,a) + \ell_1\\ 
    \intertext{where we formalize this with an explicit selection of $\ell_1$ in Lemma \ref{lemma_bounding_l1}. Since taking an infimum over each term of the sum is smaller than the infimum over the whole sum, we have}
    &\geq \sum_{j=1}^t  \inf_{M'\in \text{Alt}(M_t)} \sum_{s,a} \omega_j (s,a) KL_{M_t|M'}(s,a) + \ell_1\label{eqn_l1_first_appearance}\\
    \intertext{By a supremum argument, we can change the set over which we take the infimum. In particular, if we have that $M_t \in Alt(M_j)$ then the sum in \eqref{eqn_swapping_first_alt_to_Mj} is equal to zero and the inequality holds trivially. If $M_t \notin Alt(M_j)$ then either $Alt(M_t)=Alt(M_j)$ or $M_j$ is not absolutely continuous with respect to $M_t$. By Assumption \ref{assumption_P_in_G_bounded_ratio} we have that $M_t\ll M_j$, therefore we have that $Alt(M_t)=Alt(M_j)$. Hence line \eqref{eqn_swapping_first_alt_to_Mj} holds by equality. (Note, more formally we can apply Lemma \ref{lemma_swapping_alternative_sets_useful} here an select $M_j$ in the supremum to get the lower bound in \eqref{eqn_swapping_first_alt_to_Mj}.)}
    &\geq \sum_{j=1}^t  \inf_{M'\in \text{Alt}(M_j)} \sum_{s,a} \omega_j (s,a) KL_{M_t|M'}(s,a)+\ell_1\label{eqn_swapping_first_alt_to_Mj}\\
    \intertext{Since $M_t \approx M$ and $M_j \approx M$ under our good event, we have $M_t \approx M_j$. Hence 
    }
    &\geq \sum_{j=\sqrt{t}}^t  \inf_{M'\in \text{Alt}(M_j)} \sum_{s,a} \omega_j (s,a) KL_{M_j|M'}(s,a) + \ell_1 + \ell_2\label{eqn_swapping_to_Mj}\\
    \intertext{where we formalize this in Lemma \ref{lemma_l2_bound} with an explicit derivation of $\ell_2$. If we use the definition of $\omega_j(s,a)\in\argmax_{\Tilde{\omega}\in \Omega(M_j)}\inf_{M'\in \text{Alt}(M_j)} \sum_{s,a} \Tilde{\omega}(s,a) KL_{M_j|M'}(s,a)$, then we have}
    & =\sum_{j=\sqrt{t}}^t\sup_{\Tilde{\omega}\in \Omega(M_j)}\inf_{M'\in \text{Alt}(M_j)} \sum_{s,a} \Tilde{\omega}(s,a) KL_{M_j|M'}(s,a) + \ell_1 + \ell_2 \\
    \intertext{Then, again using a similar argument for swapping the alternative sets, namely using Lemma \ref{lemma_swapping_alternative_sets_useful}, we can do the following supremum argument}
    & =\sum_{j=\sqrt{t}}^t \sup_{M''\,:\, M_j \ll M''} \sup_{\Tilde{\omega}\in \Omega(M_j)} \inf_{M'\in \text{Alt}(M'')} \sum_{s,a} \Tilde{\omega}(s,a) KL_{M_j|M'}(s,a) + \ell_1 + \ell_2 \\
    \intertext{We can then choose $M''=M$ in the set over which the supremum is acting, to get}
    & \geq \sum_{j=\sqrt{t}}^t \sup_{\Tilde{\omega}\in \Omega(M_j)} \inf_{M'\in \text{Alt}(M)} \sum_{s,a} \Tilde{\omega}(s,a) KL_{M_j|M'}(s,a) + \ell_1 + \ell_2 \\
    \intertext{Then since $M_j \approx M$, we should also have that the sets (over which the supremum holds) are similar $\Omega(M_j)\approx \Omega(M)$. Hence }
    & \geq \sum_{j=\sqrt{t}}^t \sup_{\Tilde{\omega}\in \Omega(M)} \inf_{M'\in \text{Alt}(M)} \sum_{s,a} \Tilde{\omega}_j(s,a) KL_{M|M'}(s,a) + \ell_1 + \ell_2 + \ell_3 \label{l3_first_appearance_eqn}\\
    \intertext{where this is formalized in Lemma \ref{lemma_l3_bound} and $\ell_3$ is explicitly dervied. This can then, by definition, be related directly to $T_o(M)$ as}
    &= (t-\sqrt{t}) T_o(M)^{-1}+  \ell_1 + \ell_2 +\ell_3 \nonumber
\end{align}

Hence, to complete the proof of the non-asymptotic lower bound, we simply need to bound $\ell_1, \ell_2$ and $\ell_3$ before putting everything together, which is done explicitly in Lemmas \ref{lemma_bounding_l1}, \ref{lemma_l2_bound}, and \ref{lemma_l3_bound} below. The proof is then complete.
\end{proof}

\subsubsection{Bounding the error terms $\ell_1,\ell_2,\ell_3$}

To complete the proof of Theorem \ref{lemma_main_sequence}, we need to bound the error terms $\ell_1, \ell_2, \ell_3$. In particular, we start by bounding $\ell_1$ with Lemma \ref{lemma_bounding_l1} by measuring the effect of replacing $N_t(s,a)$ with $\sum_{j=1}^t \omega_j (s,a)$ in equation \eqref{eqn_l1_first_appearance}. We are able to control this due to our Assumption \ref{assumption_P_in_G_bounded_ratio}. After this, we focus on $\ell_2$ and $\ell_3$ which are the errors caused by swapping the active MDP's in the sums of \eqref{eqn_swapping_to_Mj} and \eqref{l3_first_appearance_eqn}. We know that the difference between our empirical estimates and the true MDP is shrinking over time, hence in Lemmas \ref{lemma_l2_bound} and \ref{lemma_l3_bound} we bound $\ell_2$ and $\ell_3$ in terms of the rate of convergence of our empirical MDP, respectively. Later in Lemma \ref{lemma_bounding_sum_of_M_estimate_errors} we study the rate of convergence of the empirical estimates for the MDP and show that the sum of these errors $\|M_j-M\|_1$ can be bounded in terms of the $F(t)$ quantity which appears in our main Theorem \ref{lemma_main_sequence}. Hence the following three lemmas complete the proof of Theorem \ref{lemma_main_sequence}.

\begin{lemma}\label{lemma_bounding_l1}
Under events $\mathcal{E}_t^1\cap\mathcal{E}_t^2\cap\mathcal{E}_t^5$ we can lower bound the following quantity as
    \begin{equation*}
        \inf_{M'\in \text{Alt}(M_t)} \sum_{s,a} N_t(s,a) KL_{M_t|M'}(s,a) \geq  \inf_{M'\in \text{Alt}(M_t)} \sum_{s,a} \left(\sum_{j=1}^t \omega_j (s,a)\right) KL_{M_t|M'}(s,a) + \ell_1,
    \end{equation*}
    where we define
    \begin{equation*}
        \ell_1 = -2SA t \, \ell(t) \log L.
    \end{equation*}
\end{lemma}
\begin{proof}
    First let $$M^* \in \argmin _{M'\in \text{Alt}(\hat{M}_t)} \sum_{s,a} N_t(s,a) KL_{\hat{M}_t|M'}(s,a)$$ Then we have
\begin{align}
&\inf_{M'\in \text{Alt}(\hat{M}_t)} \sum_{s,a} N_t(s,a) KL_{\hat{M}_t|M'}(s,a)\\
&=\sum_{s,a} N_t(s,a) KL_{\hat{M}_t|M^*}(s,a)\\
     & =  \sum_{s,a} \left[\left(\sum_{j=1}^t \omega_j (s,a)\right) - \left(\sum_{j=1}^t \omega_j (s,a)- N_t(s,a)\right)\right] KL_{\hat{M}_t|M^*}(s,a)\\
     & = \sum_{s,a} \left[\left(\sum_{j=1}^t \omega_j (s,a)\right)\right] KL_{\hat{M}_t|M^*}(s,a) - \sum_{s,a} \left[\left(\sum_{j=1}^t \omega_j (s,a)- N_t(s,a)\right)\right] KL_{\hat{M}_t|M^*}
     (s,a) \nonumber\\
     \intertext{By the definition of the concentration event $\mathcal{E}_t^5$, the second term can be bounded as}
     & \geq \sum_{s,a} \sum_{j=1}^t \omega_j (s,a) KL_{\hat{M}_t|M^*}(s,a) - t\ell(t)\sum_{s,a} KL_{\hat{M}_t|M^*} (s,a)\\
     \intertext{The first term can be lower bounded by choosing instead an $M'$ which minimises the first sum.}
     & \geq \inf_{M'\in \text{Alt}(\hat{M}_t)}\sum_{s,a} \sum_{j=1}^t \omega_j (s,a) KL_{\hat{M}_t|M'}(s,a) - \ell(t)\sum_{s,a} KL_{\hat{M}_t|M^*} (s,a) \label{eqn_l1_bounding_penultimate}\\
     & \geq \inf_{M'\in \text{Alt}(\hat{M}_t)}\sum_{s,a} \sum_{j=1}^t \omega_j (s,a) KL_{\hat{M}_t|M'}(s,a) - 2SAt\,\ell(t) \log(L) \label{eqn_lemma_l1_bound_final}
\end{align}
Here the final line comes from our Assumption \ref{assumption_P_in_G_bounded_ratio} and noticing that for any pair $P, P' \in G_\mu(L)$ that 
\begin{align*}
    \KL(P,P') &= \mathbb{E}_{x \sim P}\left[\log \frac{P(s)}{P'(x)}\right]\\
    &= \mathbb{E}_{x \sim P}\left[\log \frac{P(s)}{P'(x)} \frac{\mu(s)}{\mu(s)}\right]\\
    &= \mathbb{E}_{x \sim P}\left[\log \frac{\mu}{P'} \frac{P}{\mu}(s)\right]\\
    &= \mathbb{E}_{x \sim P}\left[\log \frac{\mu}{P'}+ \log \frac{P}{\mu}(s)\right]\\
    &\leq  2\log(L).
\end{align*}
\end{proof}

We can then bound the final two errors $\ell_2$ and $\ell_3$ in terms of the sum of errors between our empricial estimate for the transitions in the MDP and the true transitions $\|M-M_j\|_1$.

\begin{lemma}\label{lemma_l2_bound} Under events $\mathcal{E}_t^1\cap\mathcal{E}_t^2\cap\mathcal{E}_t^5$ the inequality 
\begin{equation*}
    \sum_{j=1}^t  \inf_{M'\in \text{Alt}(M_j)} \sum_{s,a} \omega_j (s,a) KL_{M_t|M'}(s,a) \geq \sum_{j=1}^t  \inf_{M'\in \text{Alt}(M_j)} \sum_{s,a} \omega_j (s,a) KL_{M_j|M'}(s,a) +\ell_2 
\end{equation*}
holds with 
    \begin{equation*}
        \ell_2 = -\sum_{j=1}^t   \sum_{s,a} S(1+2\log L)\left(\|M_t - M\|_1 + \|M - M_j\|_1\right)
    \end{equation*}
\end{lemma}
\begin{proof}
Consider the difference
    \begin{align*}
     & \sum_{j=1}^t  \inf_{M'\in \text{Alt}(M_j)} \sum_{s,a} \omega_j (s,a) KL_{M_j|M'}(s,a) - \sum_{j=1}^t  \inf_{M'\in \text{Alt}(M_j)} \sum_{s,a} \omega_j (s,a) KL_{M_t|M'}(s,a)\\
    &\leq \sum_{j=1}^t   \sum_{s,a} \omega_j (s,a)\sup_{M'\in \text{Alt}(M_j)} KL_{M_j|M'}(s,a) - KL_{M_t|M'}(s,a)\\
    \intertext{Since the weights are always upper bounded by 1 we have}
    &\leq \sum_{j=1}^t   \sum_{s,a} \sup_{M'\in \text{Alt}(M_j) \cup  \text{Alt}(M_t)}KL_{M_j|M'}(s,a) - KL_{M_t|M'}(s,a)\\
    \intertext{By our Lipschitz Lemma \ref{lemma:KL_on_G_is_lipschitz} we have}
    &\leq \sum_{j=1}^t   \sum_{s,a} S(1+2\log L)\|M_t - M_j\|_1\\
    &\leq \sum_{j=1}^t   \sum_{s,a} S(1+2\log L)\left(\|M_t - M\|_1 + \|M - M_j\|_1\right)
\end{align*}
Hence, we can always select $\ell_2$ to guarantee our desired inequality, as stated.
\end{proof}

\begin{lemma}\label{lemma_l3_bound} Under events $\mathcal{E}_t^1\cap\mathcal{E}_t^2\cap\mathcal{E}_t^5$ the inequality
\begin{align*}
    \sum_{j=\sqrt{t}}^t \sup_{\Tilde{\omega}\in \Omega(M_j)} \inf_{M'\in \text{Alt}(M)} &\sum_{s,a} \Tilde{\omega}_j(s,a) KL_{M_j|M'}(s,a) \geq \\
    &\sum_{j=\sqrt{t}}^t \sup_{\Tilde{\omega}\in \Omega(M)} \inf_{M'\in \text{Alt}(M)} \sum_{s,a} \Tilde{\omega}_j(s,a) KL_{M|M'}(s,a) + \ell_3
\end{align*}
holds with
    \begin{equation}
        \ell_3 = -  \sum_{j=\sqrt{t}}^t \sum_{s,a} S(1+2\log L) \| M-M_j\|_1.
    \end{equation}
\end{lemma}
\begin{proof}
    Consider the quantity
    \begin{align}
        & \sum_{j=\sqrt{t}}^t \sup_{\Tilde{\omega}\in \Omega(M)} \inf_{M'\in \text{Alt}(M)} \sum_{s,a} \Tilde{\omega}_j(s,a) KL_{M|M'}(s,a)\nonumber\\ &\qquad - \sum_{j=\sqrt{t}}^t \sup_{\Tilde{\omega}\in \Omega(M_j)} \inf_{M'\in \text{Alt}(M)} \sum_{s,a} \Tilde{\omega}_j(s,a) KL_{M_j|M'}(s,a)\\
        \intertext{If we denote $\omega^1_j$ and $\omega^2_j$ as the optimizing choice for the first and second supremum, then we can rewrite this as}
        &=\sum_{j=\sqrt{t}}^t  \inf_{M'\in \text{Alt}(M)} \sum_{s,a} \omega^1_j(s,a) KL_{M|M'}(s,a)\nonumber\\ &\qquad - \sum_{j=\sqrt{t}}^t  \inf_{M'\in \text{Alt}(M)} \sum_{s,a} \omega^2_j(s,a) KL_{M_j|M'}(s,a)\\
        &\leq \sum_{j=\sqrt{t}}^t (\omega^1_j(s,a)- \omega^2_j(s,a)) \sup_{M'\in \text{Alt}(M)} \sum_{s,a}  KL_{M|M'}(s,a)-  KL_{M_j|M'}(s,a)\\
        &\leq \sum_{j=\sqrt{t}}^t \sum_{s,a} S(1+2\log L) \| M-M_j\|_1
    \end{align}
    Hence we can always select an $\ell_3$ to gaurantee the desired inequality, as stated. Note that for the final inequality we use that KL is Lipschitz in $G_\mu(L)$, as formalized in Lemma \ref{lemma:KL_on_G_is_lipschitz}.
\end{proof}

\subsubsection{Helper Lemmas} 

Here we provide helpful lemmas for the convergence of the empirical MDPs $M_t$ and the true MDP $M$. In particular, we recall under event $\mathcal{E}^4_t$ (which holds under $\mathcal{E}^1_t \cap \mathcal{E}^2_t$, as in Lemma \ref{lemma_projection_consequence}) that $M_t$ will be in a ball around $M$ with shrinking radius under high probability events. We use this to provide an explicit bound on the sum of terms like $\| M-M_j\|_1$ appearing in previous lemmas. This will complete the section on the main sequence of inequalities and we can start putting together our results in the next section.

\begin{lemma}\label{lemma_phi_upper_bound_lower_bound}
Let $\alpha+\beta <1$ and recall the definition
    $$\phi(t,\lambda):= \frac{1}{2}\sum_{k=1}^{\lfloor\frac{t}{m}\rfloor} \frac{\eta_m}{(km)^{\alpha+\beta}} - \log\frac{SA}{\lambda}$$
    Then we have

    \begin{equation}
        \frac{\eta_m}{2m^{\alpha+\beta}\left(1-(\alpha+\beta)\right)}\left(\bigg\lfloor\frac{t}{m}\bigg\rfloor^{1-(\alpha+\beta)} -1\right) - \log\frac{SA}{\lambda}\leq \phi(t,\lambda) \leq \frac{\eta_m}{2}\lfloor\frac{t}{m}\rfloor- \log\frac{SA}{\lambda}
    \end{equation}
\end{lemma}
\begin{proof}
 Let's prove the lower bound for $\phi(t, \lambda)$ first. Note that the quantity 
 $$\frac{1}{k^{\alpha+\beta}}$$
 is decreasing in $k$. Hence we can lower bound the sum 
\begin{align}
    \sum_{k=1}^{\lfloor\frac{t}{m}\rfloor}\frac{1}{k^{\alpha+\beta}} &\geq \int_{1}^{\lfloor\frac{t}{m}\rfloor} \frac{1}{k^{\alpha+\beta}} dk\\
    &= \frac{1}{1-(\alpha+\beta)}\left(\bigg\lfloor\frac{t}{m}\bigg\rfloor^{1-(\alpha+\beta)} -1\right)
\end{align}
and the proof of the lower bound is then complete. To prove the upper bound, note that since we are assuming $\alpha+\beta\geq 0$ then we have that 
\begin{equation}
    \frac{1}{(km)^{\alpha+\beta}} \leq \frac{1}{(km)^{0}} = 1
\end{equation}
 Hence 
 \begin{equation}
     \frac{1}{2}\sum_{k=1}^{\lfloor\frac{t}{m}\rfloor} \frac{\eta_m}{(km)^{\alpha+\beta}} \leq \frac{1}{2}\sum_{k=1}^{\lfloor\frac{t}{m}\rfloor} \eta_m = \frac{\eta_m}{2}\lfloor\frac{t}{m}\rfloor
 \end{equation}
 as required.
\end{proof}

\begin{lemma}\label{lemma_bounding_sum_of_M_estimate_errors}
Under the event $\mathcal{E}_t^1 \cap \mathcal{E}_t^2  \cap \mathcal{E}_t^5$ we have that 
\begin{align}
    \sum_{j=1}^t  \| M-M_j\|_1 \leq  &\frac{\sqrt{16m^{\alpha+\beta}\left(1-(\alpha+\beta)\right)\cdot \beta_1\left(\frac{\eta_mt}{2m}- \log\frac{SA}{\lambda},\lambda\right)/\eta_m}}{1+(\alpha+\beta)} \left(\frac{t+1}{2m}\right)^{(1+\alpha+\beta)/2}\nonumber\\
    &\quad\qquad+2\max\{J,C\}
\end{align}
Where we define the constant as

\begin{equation}
    J = 2m\left[2+\frac{4m^{\alpha+\beta}\left(1-(\alpha+\beta)\right)\log(SA/\lambda)}{\eta_m}\right]^{\frac{1}{1-(\alpha+\beta)}}
\end{equation}

\end{lemma}
\begin{proof}
    Recall by definition that, under event $\mathcal{E}^4_t$ (which holds under $\mathcal{E}^1_t \cap \mathcal{E}^2_t$, as in Lemma \ref{lemma_projection_consequence}), we have for all $j\geq C$ that
    \begin{equation}
    \| M-M_j\|_1 \leq \psi(j,\lambda):=\sqrt{\frac{\beta_1(\phi(j,\lambda),\lambda)}{2\phi(j,\lambda)}}
\end{equation}

Hence, since $\beta_1$ is increasing in $\phi$ which is increasing in $t$, we have 
\begin{align}
    \sum_{j=C}^t  \| M-M_j\|_1 &\leq 2\max\{J,C\}+\sum_{j=\max\{J,C\}}^t  \| M-M_j\|_1\\
    &\leq 2\max\{J,C\}+\sum_{j=\max\{J,C\}}^t  \| M-M_j\|_1\\
    &\leq 2\max\{J,C\}+\sum_{j=\max\{J,C\}}^t   \sqrt{\frac{\beta_1(\phi(j,\lambda),\lambda)}{2\phi(j,\lambda)}}\\
    &\leq 2\max\{J,C\}+\sqrt{\beta_1(\phi(t,\lambda),\lambda)} \sum_{j=\max\{J,C\}}^t  \sqrt{\frac{1}{2\phi(j,\lambda)}}\label{eqn_M_sum_lemma_to_simplify}
\end{align}

Now, if we focus on the final sum of this equation we can then use our lower bound for $\phi$ in Lemma \ref{lemma_phi_upper_bound_lower_bound} to get
\begin{align}
    \sum_{j=\max\{J,C\}}^t  \sqrt{\frac{1}{2\phi(j,\lambda)}}&\leq  \sum_{j=\max\{J,C\}}^t  \left(\frac{\eta_m}{2m^{\alpha+\beta}\left(1-(\alpha+\beta)\right)}\left(\bigg\lfloor\frac{j}{m}\bigg\rfloor^{1-(\alpha+\beta)} -1\right) - \log\frac{SA}{\lambda}\right)^{-1/2}\\
    &\leq  \sum_{j=\max\{J,C\}}^t  \left(\frac{\eta_m}{2m^{\alpha+\beta}\left(1-(\alpha+\beta)\right)}\left(\left(\frac{j}{2m}\right)^{1-(\alpha+\beta)} -1\right) - \log\frac{SA}{\lambda}\right)^{-1/2}\\
    \intertext{If we isolate the $j$ dependent terms and denote the term in this sum as $(e_1(j)- e_2)^{-1/2}$, then for sufficiently large $j\geq J$ we have that $e_1(j) \geq 2e_2$. Hence, for $j\geq J$ we have $(e_1(j)- e_2)^{-1/2} \leq (e_1(j)/2)^{-1/2}$. Hence}
    &\leq  \sum_{j=\max\{J,C\}}^t  \left(\frac{\eta_m}{4m^{\alpha+\beta}\left(1-(\alpha+\beta)\right)}\left(\frac{j}{2m}\right)^{1-(\alpha+\beta)} \right)^{-1/2}\\
    &= \sum_{j=\max\{J,C\}}^t \sqrt{4m^{\alpha+\beta}\left(1-(\alpha+\beta)\right)/\eta_m} \left(\frac{j}{2m}\right)^{-(1-(\alpha+\beta))/2}\\
    &\leq  \sqrt{4m^{\alpha+\beta}\left(1-(\alpha+\beta)\right)/\eta_m}\int_{j=0}^{t+1}  \left(\frac{j}{2m}\right)^{-(1-(\alpha+\beta))/2}dj\\
    &=  \left[\frac{\sqrt{4m^{\alpha+\beta}\left(1-(\alpha+\beta)\right)/\eta_m}}{1-(1-(\alpha+\beta))/2} \left(\frac{j}{2m}\right)^{1-(1-(\alpha+\beta))/2}\right]^{j=t+1}_{j=0}\\
    &=  \frac{\sqrt{16m^{\alpha+\beta}\left(1-(\alpha+\beta)\right)/\eta_m}}{1+(\alpha+\beta)} \left(\frac{t+1}{2m}\right)^{(1+\alpha+\beta)/2}
\end{align}
By substituting this into equation \eqref{eqn_M_sum_lemma_to_simplify} and by using the upper bound for $\phi$ in Lemma \ref{lemma_phi_upper_bound_lower_bound}, namely that
\begin{equation}
    \phi(t,\lambda) \leq \frac{\eta_mt}{2m}- \log\frac{SA}{\lambda},
\end{equation}
the proof is then complete.
\end{proof}

\subsubsection{Useful Lemma for Swapping Alternative Sets}

\begin{lemma} \label{lemma_swapping_alternative_sets_useful}
Let $M_1 \in G_\mu(L)$ and let $\Tilde{\omega}$ be any valid allocation. Then
    \begin{equation*}
        \inf_{M'\in \text{Alt}(M_1)} \sum_{s,a} \Tilde{\omega}(s,a) KL_{M_1|M'}(s,a) = \sup_{M_2: M_1 \ll M_2}\inf_{M'\in \text{Alt}(M_2)} \sum_{s,a} \Tilde{\omega}(s,a) KL_{M_1|M'}(s,a) 
    \end{equation*}
\end{lemma}
\begin{proof}
There are two cases that occur over the set of $\{M_2: M_1 \ll M_2\}$ in the supremum here. In particular we can either have that $M_1 \in Alt(M_2)$ or $M_1 \notin Alt(M_2)$.

    \paragraph{Case 1:} Suppose first that $M_1 \in Alt(M_2)$. In this case, since the allocation always has all non-negative entries, we have that 
    \begin{equation*}
        \inf_{M'\in \text{Alt}(M_2)} \sum_{s,a} \Tilde{\omega}(s,a) KL_{M_1|M'}(s,a) =  \sum_{s,a} \Tilde{\omega}(s,a) KL_{M_1|M_1}(s,a) = 0.
    \end{equation*}

\paragraph{Case 2:} Now, instead suppose that $M_1 \notin Alt(M_2)$. In this case, by the definition of the alternative sets in \eqref{eqn_alt_set_definition}, we have that either $Alt(M_1) = Alt(M_2)$ or $M_1$ is not absolutely continuous with respect to $M_2$. However, by the definition of the supremum in our lemma statement here we know $M_1 \ll M_2$, hence we know
\begin{equation*}
    M_1 \notin Alt(M_2)\Longrightarrow Alt(M_1) = Alt(M_2).
\end{equation*}
Hence we have 
\begin{equation*}
        \inf_{M'\in \text{Alt}(M_2)} \sum_{s,a} \Tilde{\omega}(s,a) KL_{M_1|M'}(s,a) = \inf_{M'\in \text{Alt}(M_1)} \sum_{s,a} \Tilde{\omega}(s,a) KL_{M_1|M'}(s,a). 
    \end{equation*}

\paragraph{Combining cases:} Then by considering the supremum of $M''$ over $M_{all}$ as simply the supremum over the union of the two sets 
\begin{equation*}
    \{M_2: M_1 \ll M_2\} = \{M_2: M_1 \ll M_2, M_1 \in Alt(M_2)\} \cup \{M_2: M_1 \ll M_2, M_1 \notin Alt(M_2)\},
\end{equation*}
the proof is complete by considering the right hand side of our lemma statement
    \begin{align*}
       \sup_{M_2: M_1 \ll M_2}\inf_{M'\in \text{Alt}(M_2)} \sum_{s,a} \Tilde{\omega}(s,a) KL_{M_1|M'}(s,a)  &=\max \left\{0,\inf_{M'\in \text{Alt}(M_1)} \sum_{s,a} \Tilde{\omega}(s,a) KL_{M_1|M'}(s,a)\right\}\\
       &= \inf_{M'\in \text{Alt}(M_1)} \sum_{s,a} \Tilde{\omega}(s,a) KL_{M_1|M'}(s,a), 
    \end{align*}
since $\inf_{M'\in \text{Alt}(M_1)} \sum_{s,a} \Tilde{\omega}(s,a) KL_{M_1|M'}(s,a) \geq 0$. The proof is then complete.
\end{proof}

\newpage
\subsection{Putting everything together} \label{subsection_putting_everything_together}

Now, using Lemma \ref{lemma_main_sequence} as well as our  high-probability results on the three good events from Sections \ref{subsection_good_events}-\ref{subsection_main_concentration}, we can pull everything together to provide an upper bound on the expected stopping time as follows.
\begin{align}
    \mathbb{E}_M[\tau_\delta] &= \sum_{t=1}^\infty \mathbb{P}_M(\tau_\delta \geq t)\\
    &\leq T_{thresh} + \sum_{t=T_{thresh}+1}^\infty \mathbb{P}_M(\tau_\delta \geq t)\\
    &\leq T_{thresh} + \sum_{t=T_{thresh}+1}^\infty \mathbb{P}_M((\mathcal{E}_t^1\cap\mathcal{E}_t^2\cap\mathcal{E}^3_t)^C)\\
    &\leq T_{thresh} + \sum_{t=T_{thresh}+1}^\infty \mathbb{P}_M((\mathcal{E}_t^3)^C|\mathcal{E}^1_t\cap\mathcal{E}^2_t) + \mathbb{P}_M((\mathcal{E}^1_t)^C) + \mathbb{P}_M((\mathcal{E}^2_t)^C)\\
    &\leq T_{thresh} + \sum_{t=T_{thresh}+1}^\infty \frac{2}{t^2}+2\exp\left(-t^{\frac{1}{2}}\right)
\end{align}
where
\begin{equation}
    T_{thresh} = \inf\{t\in\mathbb{N}: \beta(t,\delta) < (t-\sqrt{t}) T(M)^{-1} - \ell_1 - \ell_2 -\ell_3\}
\end{equation}

Hence we can now provide non-asymptotic upper bounds for the sample complexity of the NaS algorithm described  in Section \ref{section_algorithm} as follows.

\begin{theorem}\label{theorem_appendix_implicit_upper_bound_full_version}
Using the NaS algorithm described above, we can provide the following non-asymptotic upper bound for the sample complexity at confidence requirement $\delta\in(0,1)$.
    \begin{equation}
        \mathbb{E}[\tau_\delta] \leq \inf \left\{t\in\mathbb{N}: b(t,\delta) < (t-\sqrt{t}) T(M)^{-1} - \ell_0\right\} + 7
    \end{equation}
    where we define 
    \begin{align*}
            \ell_0 &= 2t\,\ell(t)\log L + 3S^2A(1+2\log L) F(t)\\
    F(t)&=3\frac{\sqrt{16m^{\alpha+\beta}\left(1-(\alpha+\beta)\right)\cdot \beta_1\left(\frac{\eta_mt}{2m}- \log (SAt),1/t\right)/\eta_m}}{1+(\alpha+\beta)} \left(\frac{t+1}{2m}\right)^{(1+\alpha+\beta)/2}\nonumber\\
    &\quad+6\max\left\{2m\left[2+\frac{4m^{\alpha+\beta}\left(1-(\alpha+\beta)\right)\log(SA/\lambda)}{\eta_m}\right]^{\frac{1}{1-(\alpha+\beta)}},C\right\}\\
    \ell(t) &= 2t^{-\frac{1}{2}} + 2\Tilde{L}_t t^{-\frac{1}{4}} + 2\Tilde{L}_{t} (\kappa_M + \Tilde{L}_{t}) \left[ t^{-\frac{1}{4}} + \xi_{t^{\frac{1}{4}}} + \epsilon_{\sqrt{t}} \right]. + \frac{ 2\Tilde{L}_{t}}{t} + \frac{\sum_{k=\sqrt{t}}^{t-2}\epsilon_k}{t}
    \end{align*}
\end{theorem}

Now, given this non-asymptotic upper bound for the sample complexity and by noting that we have the following orders,

\begin{align*}
    F(t) &= \tilde{O}\left(m^{-\frac{1}{2}}\eta_m^{-\frac{1}{2}}S^{\frac{1}{2}}t^{\frac{1+\alpha+\beta}{2}}\right)\\
    \epsilon_t &= \tilde{O}(\frac{m^\alpha}{t^{\alpha(1-(\alpha+\beta))}A^\alpha\eta_M^\alpha})\\
    \tilde{L}_t &=\tilde{O}(\Tilde{L}) = \Tilde{O}(2r_0\sigma_u[A\min \pi^*(a|s)])\\
    \xi_{t^{\frac{1}{4}}} &= \tilde{O} \left(\frac{A}{\underline{\rho}_M}\varphi_M^{\frac{1}{2}}\left(\frac{S^{\frac{3}{2}}m^{\frac{1}{2}}A\log(L)(S^2+\kappa_M)}{\eta_M^{\frac{1}{2}}\,t^{\frac{1-(\alpha+\beta)}{4}}}\right)\right),
\end{align*}
we can then write the order of the $\ell_0$ term in Theorem \ref{theorem_appendix_implicit_upper_bound_full_version}. We do so in the main implicit upper bound theorem presented in the main paper in Theorem \ref{theorem_non_asymp_upper}.  Note we define $\Tilde{L} = 2r_0\sigma_u[A\min_{a,s} \pi^*(a|s)]$ above, this comes from \eqref{eqn_nicer_L_tilde_upper_bound}.

Now, given a polynomial form of the sharpness, we are able to turn the implicit upper bound provided by Theorem \ref{theorem_non_asymp_upper} into an explicit upper bound on the sample complexity. In particular, we can simplify the upper bound by considering linear sharpness, which is with $p=1/2$ 
in the formulation of the following corollary. Note we derive the constants in \eqref{eqn_stopping_threshold_calculated_coefficients} as follows.

Recall that
\[
b(t,\delta) = \log(1/\delta)
    + (S-1)\sum_{(s,a)\in S\times A} \log\!\left( e\,\left[1 + \frac{N_t(s,a)}{S-1}\right] \right).
\]
Since there are $SA$ state-action pairs, and for each $(s,a)$ we have $N_t(s,a)\leq t$, then
\[
b(t,\delta) \leq \log(1/\delta)
    + (S-1)SA\log\!\left( e\,\left[1 + \frac{t}{S-1}\right] \right).
\]
Using that $1+t/(S-1)\leq  2t$, we obtain
\[
b(t,\delta)\leq \log(1/\delta)
    + S^2A\log\!\left( 2et\right)=\log(1/\delta)
    + S^2A\log(t)+ S^2A\log\left( 2e\right)
\]
Therefore with $c_0=1, c_1=S^2A, c_2=S^2A\log(2e)$, we have
$$
    b(t,\delta)
    \le
    c_0\log(1/\delta)+c_1\log t+c_2 .$$

\begin{corollary}[Polynomial Sharpness]
\label{cor:polynomial_sharpness}
Assume the conditions of Theorem~\ref{theorem_non_asymp_upper}. Suppose additionally that
\[
    \alpha,\beta\in(0,1),
    \qquad
    \alpha+\beta<1.
\]
Assume that the sharpness function is polynomial of order $p>0$, in the sense that there exists
a constant $C_{\varphi,p}>0$ such that
\[
    \varphi_M^{\frac{1}{2}}(x)\le C_{\varphi,p}x^p
    \qquad
    \text{for all relevant } x\ge 0.
\]
Assume that
\begin{equation}\label{eqn_stopping_threshold_calculated_coefficients}
        b(t,\delta)
    \le
    c_0\log(1/\delta)+c_1\log t+c_2
\end{equation}
with $c_0=1, c_1=S^2A, c_2=S^2A\log(2e)$.
Let
\[
    L_\delta:=\log(1/\delta)
\]
and
\[
    B_\delta
    :=
    T(M)\Bigl(
        c_0L_\delta
        +
        c_1\log\bigl(eT(M)(1+L_\delta)\bigr)
        +
        c_2
    \Bigr).
\]
Define the problem-dependent constants
\[
    A_1
    :=
    S^{\frac72}A\log(L)m^{-\frac12}\eta_m^{-\frac12},
\]
\[
    A_2
    :=
    \log(L)(\widetilde L+\kappa_M)\widetilde L,
\]
\[
    A_3^{(p)}
    :=
    A_2
    \frac{A}{\underline{\rho}_M}
    C_{\varphi,p}
    \left[
        S^{\frac32}m^{\frac12}A\log(L)(S^2+\kappa_\Omega)
    \right]^p
    \eta_m^{-\frac{p}{2}},
\]
and
\[
    A_4
    :=
    A_2\frac{m^\alpha}{A^\alpha\eta_m^\alpha}.
\]
Finally, define the exponents
\[
    r_1:=\frac{1+\alpha+\beta}{2},
    \qquad
    r_2:=\frac34,
\]
\[
    r_3^{(p)}
    :=
    1-\frac{p(1-(\alpha+\beta))}{4},
    \qquad
    r_4:=
    1-\frac{\alpha(1-(\alpha+\beta))}{2}.
\]
If $r_3^{(p)}<0$, we set
\[
    \bar r_3^{(p)}:=0,
\]
and otherwise we set
\[
    \bar r_3^{(p)}:=r_3^{(p)}.
\]
Then
\[
    r_1,r_2,\bar r_3^{(p)},r_4<1,
\]
and the expected sample complexity satisfies
\[
\begin{aligned}
    \mathbb E_M[\tau_\delta]
    \le\;&
    B_\delta
    +
    \widetilde{\mathcal O}\Bigl(
        \sqrt{B_\delta}
        +
        T(M)A_1 B_\delta^{r_1}
        +
        T(M)A_2 B_\delta^{r_2}
        +
        T(M)A_3^{(p)} B_\delta^{\bar r_3^{(p)}}
        +
        T(M)A_4 B_\delta^{r_4}
    \Bigr)
    +7 .
\end{aligned}
\]
In particular, for fixed problem-dependent quantities,
\[
    \mathbb E_M[\tau_\delta]
    =
    T(M)c_0\log(1/\delta)
    +
    o(\log(1/\delta))
\]
as $\delta\to0$.
\end{corollary}

\begin{proof}
Let
\[
    \tau_\delta^{\mathrm{ub}}
    :=
    \inf\left\{
        t\in\mathbb N:
        b(t,\delta)
        <
        (t-\sqrt t)T(M)^{-1}-\ell_0(t)
    \right\}.
\]
By Theorem~\ref{theorem_non_asymp_upper},
\[
    \mathbb E_M[\tau_\delta]\le \tau_\delta^{\mathrm{ub}}+7.
\]
It is therefore enough to upper bound $\tau_\delta^{\mathrm{ub}}$.

By definition, for any time $t<\tau_\delta^{\mathrm{ub}}$, the stopping inequality has not yet
been satisfied. Hence
\[
    b(t,\delta)
    \ge
    (t-\sqrt t)T(M)^{-1}-\ell_0(t).
\]
Equivalently,
\[
    t
    \le
    \sqrt t
    +
    T(M)b(t,\delta)
    +
    T(M)\ell_0(t).
\]
Using the assumed upper bound on the threshold gives
\[
    t
    \le
    \sqrt t
    +
    T(M)\Bigl(c_0L_\delta+c_1\log t+c_2\Bigr)
    +
    T(M)\ell_0(t).
\]

We now simplify $\ell_0(t)$ under polynomial sharpness. From Theorem~\ref{theorem_non_asymp_upper},
\[
\begin{aligned}
    \ell_0(t)
    =
    \widetilde{\mathcal O}\bigg(
        &S^{\frac72}A\log(L)m^{-\frac12}\eta_m^{-\frac12}
        t^{\frac{1+\alpha+\beta}{2}}
\\
        &+
        t\log(L)(\widetilde L+\kappa_M)\widetilde L
        \bigg[
            t^{-\frac14}
\\
        &\qquad\qquad
            +
            \frac{A}{\underline{\rho}_M}
            \varphi_M^{\frac{1}{2}}\left(
                \frac{
                    S^{\frac32}m^{\frac12}A\log(L)(S^2+\kappa_\Omega)
                }{
                    \eta_m^{\frac12}t^{\frac{1-(\alpha+\beta)}{4}}
                }
            \right)
\\
        &\qquad\qquad
            +
            \frac{m^\alpha}{
                t^{\alpha(1-(\alpha+\beta))/2}A^\alpha\eta_m^\alpha
            }
        \bigg]
    \bigg).
\end{aligned}
\]
The first term is
\[
    A_1 t^{r_1},
    \qquad
    r_1=\frac{1+\alpha+\beta}{2}.
\]
The first term inside the bracket contributes
\[
    tA_2t^{-1/4}
    =
    A_2t^{3/4}
    =
    A_2t^{r_2},
    \qquad
    r_2=\frac34.
\]
For the sharpness term, polynomial sharpness gives
\[
\begin{aligned}
    &tA_2
    \frac{A}{\underline{\rho}_M}
    \varphi_M^{\frac{1}{2}}\left(
        \frac{
            S^{\frac32}m^{\frac12}A\log(L)(S^2+\kappa_\Omega)
        }{
            \eta_m^{\frac12}t^{\frac{1-(\alpha+\beta)}{4}}
        }
    \right)
\\
    &\le
    tA_2
    \frac{A }{\underline{\rho}_M}
    C_{\varphi,p}
    \left(
        \frac{
            S^{\frac32}m^{\frac12}A\log(L)(S^2+\kappa_\Omega)
        }{
            \eta_m^{\frac12}t^{\frac{1-(\alpha+\beta)}{4}}
        }
    \right)^p
\\
    &=
    A_3^{(p)}
    t^{1-\frac{p(1-(\alpha+\beta))}{4}}
    =
    A_3^{(p)}t^{r_3^{(p)}}.
\end{aligned}
\]
If $r_3^{(p)}<0$, then for all $t\ge1$ we have $t^{r_3^{(p)}}\le1$, and hence this term is
bounded by $A_3^{(p)}=A_3^{(p)}t^0$. This is why we use
\[
    \bar r_3^{(p)}:=\max\{0,r_3^{(p)}\}.
\]
Finally, the last term inside the bracket contributes
\[
    tA_2
    \frac{m^\alpha}{
        t^{\alpha(1-(\alpha+\beta))/2}A^\alpha\eta_m^\alpha
    }
    =
    A_4
    t^{1-\frac{\alpha(1-(\alpha+\beta))}{2}}
    =
    A_4t^{r_4}.
\]
Combining the preceding displays,
\[
    \ell_0(t)
    =
    \widetilde{\mathcal O}\left(
        A_1t^{r_1}
        +
        A_2t^{r_2}
        +
        A_3^{(p)}t^{\bar r_3^{(p)}}
        +
        A_4t^{r_4}
    \right).
\]
Therefore, for every $t<\tau_\delta^{\mathrm{ub}}$,
\[
\begin{aligned}
    t
    \le\;&
    T(M)\Bigl(c_0L_\delta+c_1\log t+c_2\Bigr)
    +
    \sqrt t
\\
    &+
    \widetilde{\mathcal O}\left(
        T(M)A_1t^{r_1}
        +
        T(M)A_2t^{r_2}
        +
        T(M)A_3^{(p)}t^{\bar r_3^{(p)}}
        +
        T(M)A_4t^{r_4}
    \right).
\end{aligned}
\]

We next record that all powers appearing above are strictly sublinear. Since $\alpha+\beta<1$,
\[
    r_1=\frac{1+\alpha+\beta}{2}<1,
    \qquad
    r_2=\frac34<1.
\]
Also,
\[
    r_4
    =
    1-\frac{\alpha(1-(\alpha+\beta))}{2}
    <1,
\]
because $\alpha>0$ and $1-(\alpha+\beta)>0$. Finally, since $p>0$,
\[
    r_3^{(p)}
    =
    1-\frac{p(1-(\alpha+\beta))}{4}
    <1.
\]
Thus also $\bar r_3^{(p)}<1$.

It remains to invert the preceding inequality. Since every exponent is strictly smaller than one,
the sublinear terms can be evaluated at the leading scale
\[
    t\asymp T(M)\log(1/\delta).
\]
More explicitly, the logarithmic term contributes only
\[
    T(M)c_1\log\bigl(eT(M)(1+L_\delta)\bigr)
\]
at this scale, and the remaining terms are sublinear perturbations. Hence the maximal value of
$t$ satisfying the preceding display is at most
\[
\begin{aligned}
    B_\delta
    +
    \widetilde{\mathcal O}\Bigl(
        \sqrt{B_\delta}
        +
        T(M)A_1B_\delta^{r_1}
        +
        T(M)A_2B_\delta^{r_2}
        +
        T(M)A_3^{(p)}B_\delta^{\bar r_3^{(p)}}
        +
        T(M)A_4B_\delta^{r_4}
    \Bigr).
\end{aligned}
\]
Consequently,
\[
\begin{aligned}
    \tau_\delta^{\mathrm{ub}}
    \le\;&
    B_\delta
    +
    \widetilde{\mathcal O}\Bigl(
        \sqrt{B_\delta}
        +
        T(M)A_1B_\delta^{r_1}
        +
        T(M)A_2B_\delta^{r_2}
        +
        T(M)A_3^{(p)}B_\delta^{\bar r_3^{(p)}}
        +
        T(M)A_4B_\delta^{r_4}
    \Bigr).
\end{aligned}
\]
Plugging this into the bound
\[
    \mathbb E_M[\tau_\delta]\le \tau_\delta^{\mathrm{ub}}+7
\]
proves the claimed finite-confidence upper bound.

Finally, since
\[
    B_\delta
    =
    T(M)c_0\log(1/\delta)
    +
    \mathcal O\!\left(T(M)\log\log(1/\delta)\right),
\]
and since
\[
    r_1,r_2,\bar r_3^{(p)},r_4<1,
\]
all correction terms are $o(\log(1/\delta))$ for fixed problem-dependent quantities. Hence
\[
    \mathbb E_M[\tau_\delta]
    =
    T(M)c_0\log(1/\delta)
    +
    o(\log(1/\delta)),
\]
as $\delta\to0$.
\end{proof}

\newpage
\subsection{Supplementary Results}

In this section we provide some supplementary results useful throughout the paper. We start by proving two results that are consequences of our Assumption \ref{assumption_P_in_G_bounded_ratio}. In particular, we first note that $G_\mu(L)$ is a convex set. Secondly, we note that the function $KL(\cdot, M')$ is Lipschitz for $M' \in G_\mu(L)$ in the domain of $G_\mu(L)$.

\begin{lemma}\label{lemma_G_mu_is_convex}
    The set of environments $G_\mu(L)$ is convex.
\end{lemma}
\begin{proof}
Consider any two environments $M_,M_2 \in G_\mu(L)$ and let $\alpha\in[0,1]$. Let $M_3 = \alpha M_1 + (1-\alpha)M_2$.\\

Since we have by assumption that $M_1,M_2 \in \Delta(S)$ we have their convex combination is also a probability distribution $M_3 \in \Delta(S)$. Furthermore, since $\mu \ll M_1$ and $\mu \ll M_2$ we have that $\mu \ll M_3$. Finally, we have by assumption that
\begin{equation}
    \frac{M_1(s)}{\mu(s)}, \frac{M_2(s)}{\mu(s)} \in [1/L,L] \quad \forall s \in S.
\end{equation}
Hence we have
\begin{equation}
    \frac{M_3(s)}{\mu(s)} = \frac{\alpha M_1(s) + (1-\alpha)M_2(s)}{\mu(s)} = \alpha\frac{ M_1(s)}{\mu(s)} + (1-\alpha)\frac{M_2(s)}{\mu(s)} \in [1/L,L]
\end{equation}
since the set $[1/L,L]$ is convex. Hence, the proof is complete.
\end{proof}

\begin{lemma} \label{lemma:KL_on_G_is_lipschitz}
    Under Assumption \ref{assumption_P_in_G_bounded_ratio}, and with any transition probabilities $P_1,P_2,P_3 \in G_\mu(L)$ we have that 
    \begin{equation}
        |\KL(P_1,P_3) - \KL(P_2,P_3)| \leq  S(1+2\log L) \|P_1-P_2\|_1
    \end{equation}
\end{lemma}
\begin{proof}
    In order to prove this, we can simply use the mean-value theorem. In particular, since $G_\mu(L)$ is convex, we have that 
    \begin{equation}
        |\KL(P_1,P_3) - \KL(P_2,P_3)| \leq  \sup_{P\in G_\mu(L)} \|\nabla_P \KL(P,P_3)\|_\infty \quad \cdot \quad\|P_1-P_2\|_1.
    \end{equation}
     Where we denote $\nabla_P(\cdot) = (d/dP(1), \dots, d/dP(S))^\top $. Hence, to complete the proof, we need only bound the norm of gradient of the KL of $\nabla_P \KL(P,P_3)$. Note that

    \begin{equation}
        \frac{d}{dP(s')} \KL(P,P_3) = \frac{d}{dP(s')} \sum_s P(s) \log \frac{P(s)}{P_3(s)} = 1 + \log \frac{P(s')}{P_3(s')} \leq 1+2\log L
    \end{equation}

    Therefore we have that 
    \begin{equation}
        \sup_{P\in G_\mu(L)} \|\nabla_P \KL(P,P_3)\|_1  \leq S(1+2\log L)
    \end{equation}
    and the proof is complete.
\end{proof}

We additionally modify the proof of \citep[Lemma F.4, ][]{dann2017unifying} by using Doob's inequality instead of Markov's inequality to provide the following bounded time event that holds with high probability.
\begin{lemma}\label{lemma:modified_bernoulli_bounded_t_event}
Let $\mathcal{F}_i$ for $i = 1, \ldots, n$ be a filtration and $X_1, \ldots, X_n$ be a sequence of Bernoulli random variables with 
$\mathbb{P}(X_i = 1 \mid \mathcal{F}_{i-1}) = P_i$ with $P_i$ being $\mathcal{F}_{i-1}$-measurable and $X_i$ being $\mathcal{F}_i$-measurable.
It holds for any $W>0$ that
\begin{equation}
    \mathbb{P} \left( \exists n \in [1,t]: \sum_{i=1}^{n} X_i < \sum_{i=1}^{n} \frac{P_i}{2} - W \right) \le e^{-W}.
\end{equation}
\end{lemma}
\begin{proof}
$P_i - X_i$ is a martingale difference sequence with respect to the filtration $\mathcal{F}_i$.
Since $X_i$ is nonnegative and has finite second moment, we have for any $\lambda > 0$ that 
\[
\mathbb{E}\left[e^{-\lambda(X_i - P_i)} \mid \mathcal{F}_{i-1}\right] \le e^{\lambda^2 P_t / 2}
\]
\citep[][Exercise 2.9]{boucheron:hal-00942704}).
Hence, we have
\[
\mathbb{E}\left[ e^{\lambda (P_i - X_i) - \lambda^2 P_i / 2} \mid \mathcal{F}_{i-1} \right] \le 1.
\]

By setting $\lambda = 1$, we see that
\[
M_n = e^{\sum_{i=1}^{n} (-X_i + P_i / 2)}
\]
is a 
supermartingale. It hence holds by the definition of $M_n$ and Doob's inequality that
\[
\mathbb{P}\left(\exists n \in [1,t]: \sum_{i=1}^{n} (-X_i + P_i / 2) \ge W \right)
= \mathbb{P}\left( \exists n \in [1,t]: M_n \ge e^{W} \right)
\le e^{-W} \mathbb{E}[M_1] \le e^{-W},
\]
which gives us the desired result.
\end{proof}

We additionally modify the proof of \citep[Proposition 1, ][]{jonsson2020planning} to use Doob's inequality at the end instead as follows.

\begin{lemma}\label{lemma:modified_empirical_probability_jonsson}
Let $X_1, X_2, \ldots, X_n, \ldots$ be i.i.d.\ samples from a distribution supported over 
$\{1, \ldots, m\}$, of probabilities given by $p \in \Sigma_m$, where $\Sigma_m$ is the 
probability simplex of dimension $m-1$. We denote by $\hat{p}_n$ the empirical vector 
of probabilities, i.e.\ for all $k \in \{1, \ldots, m\}$,
\[
\hat{p}_{n,k} = \frac{1}{n} \sum_{\ell=1}^{n} \mathds{1}(X_\ell = k) .
\]

For all $p \in \Sigma_m$, for all $\delta \in [0,1]$,
\[
\mathbb{P}\!\left(
\exists n \in [1,t],\,
n\, \mathrm{KL}(\hat{p}_n, p) > 
\log(1/\delta) + (m - 1)\log\!\big(e(1 + n/(m - 1))\big)
\right) \le \delta .
\]
\end{lemma}

\begin{proof}
Consider a Dirichlet prior on the mean parameter
of the exponential family formed by the set of categorical distributions on $\{1,\dots,m\}$.
Let
\[
\phi_p(\lambda) = \log \mathbb{E}_{X \sim p}\!\left[e^{\lambda X}\right]
= \log\!\left(p_m + \sum_{k=1}^{m-1} p_k e^{\lambda_k}\right)
\]
be the log-partition function. The following quantity is a martingale:
\[
M^\lambda_n = \exp\!\big(n \langle \lambda, \hat{p}_n \rangle - n \phi_p(\lambda)\big).
\]

We set a Dirichlet prior $q \sim \mathrm{Dir}(\alpha)$ with $\alpha \in \mathbb{R}^{\ast +}_m$,
and for $\lambda_q = (\nabla \phi_p)^{-1}(q)$ consider the integrated martingale
\[
\begin{aligned}
M_n
&= \int M^{\lambda_q}_n
\frac{\Gamma\!\left(\sum_{k=1}^m \alpha_k\right)}
{\prod_{k=1}^m \Gamma(\alpha_k)}
\prod_{k=1}^m q_k^{\alpha_k - 1} \, dq \\
&= \int e^{n(\mathrm{KL}(\hat{p}_n,p) - \mathrm{KL}(\hat{p}_n,q))}
\frac{\Gamma\!\left(\sum_{k=1}^m \alpha_k\right)}
{\prod_{k=1}^m \Gamma(\alpha_k)}
\prod_{k=1}^m q_k^{\alpha_k - 1} \, dq \\
&= e^{n\mathrm{KL}(\hat{p}_n,p) + n H(\hat{p}_n)}
\frac{\Gamma\!\left(\sum_{k=1}^m \alpha_k\right)}
{\prod_{k=1}^m \Gamma(\alpha_k)}
\int \prod_{k=1}^m q_k^{n \hat{p}_{n,k} + \alpha_k - 1} \, dq \\
&= e^{n\mathrm{KL}(\hat{p}_n,p) + n H(\hat{p}_n)}
\frac{\Gamma\!\left(\sum_{k=1}^m \alpha_k\right)}
{\prod_{k=1}^m \Gamma(\alpha_k)}
\frac{\prod_{k=1}^m \Gamma(\alpha_k + n \hat{p}_{n,k})}
{\Gamma\!\left(\sum_{k=1}^m \alpha_k + n\right)} .
\end{aligned}
\]

Now choose the uniform prior $\alpha = (1,\dots,1)$. Then
\[
\begin{aligned}
M_n
&= e^{n\mathrm{KL}(\hat{p}_n,p) + n H(\hat{p}_n)} (m-1)!
\frac{\prod_{k=1}^m \Gamma(1 + n \hat{p}_{n,k})}{\Gamma(m + n)} \\
&= e^{n\mathrm{KL}(\hat{p}_n,p) + n H(\hat{p}_n)} (m-1)!
\frac{\prod_{k=1}^m (n \hat{p}_{n,k})!}{n!}
\frac{n!}{(m + n - 1)!} \\
&= e^{n\mathrm{KL}(\hat{p}_n,p) + n H(\hat{p}_n)}
\frac{\binom{n}{n \hat{p}_n}^{-1}}{\binom{m + n - 1}{m - 1}^{-1}} \\
&= e^{n\mathrm{KL}(\hat{p}_n,p) + n H(\hat{p}_n)}
\frac{1}{\binom{n}{n \hat{p}_n} \binom{m + n - 1}{m - 1}} .
\end{aligned}
\]

Thanks to \citep[][Theorem 11.1.3]{elements_of_information_theory} , for any $M \in \mathbb{N}^\ast$ and
$x \in \{0,\dots,M\}^m$ such that $\sum_{k=1}^m x_k = M$, it holds that
\[
\binom{M}{x} = \frac{M!}{\prod_{k=1}^m x_k!} \le e^{M H(x/M)} .
\]
Hence
\[
M_n \ge e^{n\mathrm{KL}(\hat{p}_n,p) + n H(\hat{p}_n)
- n H(\hat{p}_n) - (m + n - 1) H\!\left(\frac{m-1}{m + n - 1}\right)} 
= e^{n\mathrm{KL}(\hat{p}_n,p)
- (m + n - 1) H\!\left(\frac{m-1}{m + n - 1}\right)} .
\]

We upper bound the entropic term:
\[
\begin{aligned}
(m + n - 1) H\!\left(\frac{m - 1}{m + n - 1}\right)
&= (m - 1) \log\!\frac{m + n - 1}{m - 1}
+ n \log\!\frac{m + n - 1}{n} \\
&\le (m - 1) \log\!\left(1 + \frac{n}{m - 1}\right)
+ n \log\!\left(1 + \frac{m - 1}{n}\right) \\
&\le (m - 1) \log\!\left(1 + \frac{n}{m - 1}\right) + (m - 1).
\end{aligned}
\]
Thus,
\[
M_n \ge e^{n\mathrm{KL}(\hat{p}_n,p)}
\left(e\!\left(1 + \frac{n}{m - 1}\right)\right)^{-(m - 1)} .
\]

Finally, using the fact that for any supermartingale $M_n$, with Doob's inequality we have
\[
\mathbb{P}\!\left(\exists n \in [1,t] : M_n > \frac{1}{\delta}\right)
\le \delta \, \mathbb{E}[M_1],
\]
we conclude that
\[
\mathbb{P}\!\left(
\exists n \in [1,t], \,
n \, \mathrm{KL}(\hat{p}_n, p)
> (m-1) \log\!\left(e\!\left(1 + \frac{n}{m-1}\right)\!\right)
+ \log(1/\delta)
\right)
\le \delta .
\]
\end{proof}

We recall the following fact for perturbation bounds for the stationary distribution of a Markov chain \citep{cho2001Comparison}.
\begin{lemma}\label{lemma_aux_1}
Consider a communicating MDP $M=(S,A,P)$ with transition function $P$ and finite state-action spaces $S$ and $A$. Consider  two policies $\pi_1,\pi_2$ such that the induced transitions $P_{\pi_1}$ and $P_{\pi_2}$ are irreducible and homogeneous (with $P_\pi(s'|s)=\sum_aP(s'|s,a)\pi(a|s)$). Then, the two induced Markov chains $P_{\pi_1},P_{\pi_2} $admit unique stationary distributions, respectively $\mu_1$ and $\mu_2$, satisfying
\begin{equation}
\|\mu_1-\mu_2\|_1\leq \min(\kappa_1,\kappa_2) \|P_{\pi_1}-P_{\pi_2}\|_\infty,
\end{equation}
where $\kappa_i=\|(I-P_{\pi_i}+\mathbf{1}\mu_i^\top)^{-1}\|_\infty$ and $Z_i=(I-P_{\pi_i}+\mathbf{1}\mu_i^\top)^{-1}$ is also called the fundamental matrix \citep{puterman2014markov}
\end{lemma}
\begin{proof}
    The fact that the two stationary distributions are unique follow from standard MDP arguments \citep{puterman2014markov}. The proof then follows as in \citep{schweitzer1968Perturbation} (see also \citep{cho2001Comparison} for a comparison between perturbation bounds).
\end{proof}

We now prove that the optimal allocation $\omega^\star$ induces an ergodic chain.
\begin{lemma}\label{lemma_auxiliary_positive_omegas}
    Consider a communicating MDP $\calm=(S,A,P,r)$ with discount factor $\gamma\in (0,1)$ and unique optimal policy $\pi^\star$. Assume the rewards to be deterministic. Consider the optimization problem
    \[
    \max_{\omega \in \Omega(\calm)} \inf_{\calmalt \in \alt(\calm)} T_{\calm,\calm'}(\omega) \,\hbox{ where }\, T_{\calm,\calm'}(\omega)\coloneqq \sum_{s,a}\omega(s,a){\rm KL}(P(s,a),P'(s,a)),
    \]
    and $\omega(\calm)=\{\omega\in \Delta(S\times A): \sum_a \omega(s,a)=\sum_{s',a'}P(s|s',a')\omega(s',a')\}$ is the set of possible stationary distributions. 

    If  for all $s,a\neq \pi^\star(s)$ we have  \[
   \gamma\left(\|V_P^{\pi^\star}\|_\infty- \mathbb{E}_{s'\sim P(s,a)}[V^{\pi^\star}(s')]\right) > \Delta(s,a),
    \]
    and for all $s$ we have that
    \[
     \gamma (1-\gamma)\left( \mathbb{E}_{s'\sim P(s,\pi^\star(s))}[V^{\pi^\star}(s')]- \min_{s''}V^{\pi^\star}(s'')\right)> \min_{a\neq\pi^\star(s)}\Delta(s,a),
    \]
    then we have that any optimal solution $\omega^\star\in \arg\max_{\omega\in \Omega(\calm)}\inf_{\calm'\in \alt(\calm)}T_{\calm,\calm'}(\omega)$ satisfies $\omega^\star(s,a)>0$ for all $(s,a)$. Hence, if $\calm$ admits a self-loop, then the chain induced by $\omega$ is ergodic.
\end{lemma}
\begin{proof}
    \noindent{\bf Preamble.} We prove it by contradiction. 
    First, as in \citep{al2021adaptive} note that we can always write the set of confusing models as \[\alt(\calm)=\cup_{s,a\neq \pi^\star(s)} \alt_{s,a}(\calm),\] where $\alt_{s,a}(\calm)=\{\calm': (P(s',a')\ll P'(s',a') \;\forall (s',a') )\,\wedge\,Q_{P'}^{\pi^\star}(s,a)>V_{P'}^{\pi^\star}(s)\}$ with $\calm'=(S,A,P',r)$. Hence, the optimization problem can also be written as
    \[
    \max_{\omega \in \Omega(\calm)} \min_{s',a'\neq\pi^\star(s_0)}\inf_{\calm' \in \alt_{s',a'}(\calm)} \sum_{s,a}\omega(s,a){\rm KL}(P(s,a),P'(s,a)).
    \]
    
    \noindent {\bf Proof strategy.} Assume \[\omega^\star \in \arg\max_{\omega\in \Omega(\calm)}\inf_{\calm'\in\alt(\calm)}\sum_{s,a}\omega(s,a){\rm KL}(P(s,a),P'(s,a))\] is such that  $\exists (s_0,a_0): \omega^\star(s_0,a_0)=0$. Then, to prove the statement, we can construct an alternative MDP $\calm'$ such that $P(s,a)\ll P'(s,a)$ in all $(s,a)$, satisfying $\pi^\star\notin \Pi^\star(\calm')$, where $\Pi^\star(\calm')$ is the set of optimal policies for $\calm'$. If $P'$ is built such that only the transition in $(s_0,a_0)$ is different, and all the other transitions satisfy $P'(s,a)=P(s,a)$,  then we have that \[\sum_{s,a}\omega^\star(s,a){\rm KL}(P(s,a),P'(s,a))= \omega^\star(s_0,a_0){\rm KL}(P(s_0,a_0),P'(s_0,a_0))=0.\]
    Since the minimum value is $0$, then such $P'$ attains the minimum value, implying that the sample complexity is infinite. But since we know that at the optimum the value is strictly better than $0$,
    then $\omega^\star$ is not an optimal solution, and thus we have a contradiction, and we must have $\omega^\star(s_0,a_0)>0$.

    In the proof we also use the following fact: by uniqueness of the optimal policy we have that $V_P^{\pi^\star}(s)>Q_P^{\pi^\star}(s,a)$ for all $s,a\neq \pi^\star(s)$. We also
    let $s_{\rm max}\in \arg\max_s V_P^{\pi^\star}(s)$ be any state with maximum value and $s_{\rm min}\in \arg\min V_P^{\pi^\star}(s)$ the minimum value.
    \\
    
    \noindent{\bf Construction of an alternative model for $s_0,a_0\neq \pi^\star(s_0)$.} We construct an alternative transition function using a parameter $\lambda\in (0,1)$:
    \[
    P'(s'|s,a)=\begin{cases}
        (1-\lambda)P(s'|s,a)+\lambda \mathbf{1}_{(s'=s_{\rm max})} & (s,a)=(s_0,a_0),\\
        P(s'|s,a) & \hbox{otherwise}.
    \end{cases}
    \]
    We then clearly have that $P(s,a)\ll P'(s,a)$ for all $(s,a)$. Furthermore, we also have that $P'$ is a communicating MDP and $\sum_{s,a}\omega^\star(s,a){\rm KL}(P(s,a),P'(s,a))=0$. We are left with proving that the optimal policy in $\calm'=(S,A,P',r)$ is not the same optimal policy in $\calm$.
    \\
    
    Let $\pi'$ be an optimal policy in $\calm'$. Then, consider $\Delta(s,a)=V_{P}^{\pi^\star}(s)-Q_{P}^{\pi^\star}(s,a)$. We now $\calm'$ is an alternative model if $Q_{P'}^{\pi^\star}(s,a)> V_{P'}^{\pi^\star}(s)$ in some pair $s,a\neq\pi^\star(s)$. This condition is also equivalent to
    \[
    Q_{P'}^{\pi^\star}(s,a)-V_P^{\pi^\star}(s)>V_{P'}^{\pi^\star}(s)-V_P^{\pi^\star}(s).
    \]
    However, note that for  all $s$ we have that $V_{P'}^{\pi^\star}(s)-V_P^{\pi^\star}(s)=\gamma P(s,\pi^\star(s))^\top [V_{P'}^{\pi^\star}-V_P^{\pi^\star}]$. Since $\gamma\in (0,1)$, the only value $V_{P'}^{\pi^\star}$ satisfying the contraction is $V_{P'}^{\pi^\star}-V_P^{\pi^\star}=0$. Therefore, the condition reduces to 
    \[ Q_{P'}^{\pi^\star}(s,a)-V_P^{\pi^\star}(s)>0.\]

    Let $s_0,a_0\neq\pi^\star(s_0)$ be the state-action pair in which we impose $Q_{P'}^{\pi^\star}(s_0,a_0)-V_P^{\pi^\star}(s_0)>0$. Using that $V_{P'}^{\pi^\star}=V_P^{\pi^\star}$, the condition is rewritten as follows
    \[
    r(s_0,a_0)+\gamma(1-\lambda)P_{s_0,a_0}^\top V_P^{\pi^\star} + \gamma \lambda \|V_P^{\pi^\star}\|_\infty -V_P^{\pi^\star} > 0.
    \]
    In other words, we have
    \[
    \gamma \lambda(\|V_P^{\pi^\star}\|_\infty- P_{s_0,a_0}^\top V_P^{\pi^\star}) \geq \Delta(s_0,a_0),
    \]
    where $\Delta(s_0,a_0)\coloneqq V_P^{\pi^\star}(s_0)-r(s_0,a_0)-\gamma P_{s_0,a_0}^\top V_{P}^{\pi^\star}=V_P^{\pi^\star}(s_0)- Q_{P}^{\pi^\star}(s_0,a_0)$.  Since $a_0\neq \pi^\star(s_0)$, we have $\Delta(s_0,a_0)>0$. Therefore, if
    \[
   \gamma(\|V_P^{\pi^\star}\|_\infty- P_{s_0,a_0}^\top V_P^{\pi^\star}) > \Delta(s_0,a_0),
    \]
    we have that there exists $\lambda \in (0,1)$ such that $P'$ is an alternative model. Thus, we have proven the statement for all $(s,a\neq \pi^\star(s))$
  \\
    
    \noindent{\bf Construction of an alternative model for $s_0,a_0=\pi^\star(s_0)$.} The proof follows similarly as before, but now we change $P'$ so that the value in $(s_0,a_0)$ decreases sufficiently (with $a_0=\pi^\star(s_0)$).   Hence, to guarantee that $P'$ is an alternative model, we need to show that there exists $(s_0,a\neq\pi^\star(s_0))$ such that $Q_{P'}^{\pi^\star}(s_0,a)> V_{P'}^{\pi^\star}(s_0)$. 
    
    Define $P'$ as follows with $\lambda\in (0,1)$
     \[
    P'(s'|s,a)=\begin{cases}
        (1-\lambda)P(s'|s,a)+\lambda \mathbf{1}_{(s'=s_{\rm min})} & (s,a)=(s_0,a_0),\\
        P(s'|s,a) & \hbox{otherwise}
        \end{cases}
        \]
We now show that this choice of $P'$ yields a worse value of $\pi^\star$ in $P'$. 
Let $P_{\pi^\star}=(P_{\pi^\star}(s'|s))_{s,s'}$ be the transition matrix induced by $\pi^\star$ in $P$ and similarly define $P_{\pi^\star}'$. Observe that $P_{\pi^\star}'=P_{\pi^\star}+\lambda A$, where $(A)_{s,s'}=[  \mathbf{1}_{(s'=s_{\rm min})}-(P_{\pi^\star})_{s,s'} ]\mathbf{1}_{(s=s_0)}$. Since $(I-\gamma P_{\pi^\star}) V_P^{\pi^\star}=r $ and $(I-\gamma P_{\pi^\star}') V_{P'}^{\pi^\star}=r $ we have that
\[
(I-\gamma P_{\pi^\star}) V_P^{\pi^\star}=(I-\gamma P_{\pi^\star}') V_{P'}^{\pi^\star}.
\]
Using that $P_{\pi^\star}=P_{\pi^\star}'-\lambda A$, we have $
(I-\gamma P_{\pi^\star}'+\gamma \lambda A) V_P^{\pi^\star}=(I-\gamma P_{\pi^\star}') V_{P'}^{\pi^\star}$, and therefore
\[
(V_{P'}^{\pi^\star}-V_{P}^{\pi^\star})= \gamma \lambda (I-\gamma P_{\pi^\star}')^{-1}A V_{P}^{\pi^\star}.
\]
By construction we have $(A V_{P}^{\pi^\star})_s=0$ for $s\neq s_0$, and $(A V_{P}^{\pi^\star})_{s_0}= V_{P}^{\pi^\star}(s_{\rm min})-\sum_{s'} P_{\pi^\star}(s'|s_0) V_{P}^{\pi^\star}(s') \leq 0$. Therefore, since $(I-\gamma P_{\pi^\star}^{-1})=\sum_k (\gamma P_{\pi^\star}')^k$ has nonnegative entries, we see that $(V_{P'}^{\pi^\star}-V_{P}^{\pi^\star})\leq 0$ holds element-wise. In fact, we also write $A V_{P}^{\pi^\star}= \kappa e_{s_0}$, where $\kappa \coloneqq V_{P}^{\pi^\star}(s_{\rm min})-\sum_{s'} P_{\pi^\star}(s'|s_0) V_{P}^{\pi^\star}(s')$ and $(e_{s_0})_s=\mathbf{1}_{s=s_0}$ is an $S$-dimensional unit vector.

Now, to conclude the proof we need to guarantee the following condition $Q_{P'}^{\pi^\star}(s_0,a)> V_{P'}^{\pi^\star}(s_0)$, which can be rewritten as follows 
\begin{align*}
r(s_0,a) +\gamma P_{s_0,a}^\top V_{P'}^{\pi^\star} \pm \gamma P_{s_0,a}^\top V_{P}^{\pi^\star} &> V_{P'}^{\pi^\star}(s_0),\\
Q_{P}^{\pi^\star}(s_0,a) + \gamma P_{s_0,a}^\top (V_{P'}^{\pi^\star}-V_{P}^{\pi^\star})&>  V_{P'}^{\pi^\star}(s_0) \pm V_P^{\pi^\star}(s_0),\\
-\Delta(s_0,a) &>  [e_{s_0}-\gamma P_{s_0,a}]^\top(V_{P'}^{\pi^\star}-V_{P}^{\pi^\star}).
\end{align*}
  Then, swapping the two sides, the condition to verify becomes
\[
[\gamma P_{s_0,a}-e_{s_0}]^\top(V_{P'}^{\pi^\star}-V_{P}^{\pi^\star})=\gamma \lambda|\kappa|[e_{s_0}-\gamma P_{s_0,a}]^\top   (I-\gamma P_{\pi^\star}')^{-1} e_{s_0}>\Delta(s_0,a)   .
\]
To lower bound $[e_{s_0}-\gamma P_{s_0,a}]^\top   (I-\gamma P_{\pi^\star}')^{-1} e_{s_0}$ let $W$ be the solution to the Bellman equation $W= e_{s_0}+\gamma (P_{\pi^\star}')^\top W$. Clearly $W(s_0)\geq W(s)$ for $s\neq s_0$ by monotonicity, and $W(s_0)\geq 1$ (since $W(s_0)= \sum_{k\geq 0}(\gamma P_{\pi^\star}')_{s_0,s_0}^k\geq 1$).  Therefore
\[
[e_{s_0}-\gamma P_{s_0,a}]^\top   (I-\gamma P_{\pi^\star}')^{-1} e_{s_0}= W(s_0)-\gamma  P_{s_0,a}^\top  W \geq W(s_0)-\gamma  W(s_0)P_{s_0,a}^\top  {\bf 1}\geq 1-\gamma.
\]

Thus the proof concludes by verifying
\[
 \lambda \geq  \frac{\Delta(s_0,a)}{\gamma |\kappa|(1-\gamma)},
\]
which allows to pick a valid value of $\lambda$ if  $\gamma |\kappa|(1-\gamma)> \Delta(s_0,a)$. This concludes the proof.
\end{proof}

\begin{theorem}[Berge's Maximum Theorem \citep{berge1997Topological,sundaram1996first}]\label{thm:berge_maximum}
Consider two sets ${\cal X}\subset\mathbb{R}^n$ and ${\cal Y}\subset\mathbb{R}^m$. Let $f:{\cal X}\times {\cal Y}\to\mathbb{R}$ be a continuous function and
let $\psi: {\cal X}\to {\cal P}({\cal Y})$ be a non-empty compact-valued continuous correspondence. Define
\begin{enumerate}
    \item $f^\star(x)=\max\{f(x,y)\,|\, y\in \psi(x)\}$.
    \item $\psi^\star(x)=\{y\in \psi(x)\,|\,f(x,y)=f^\star(x)\}$.
\end{enumerate}
Then
\begin{enumerate}
    \item The function $f^\star$ is a continuous function on ${\cal X}$ and $\psi^\star$ is a compact-valued upper-hemi-continuous correspondence on ${\cal X}$.
    \item If $f$ is concave on ${\cal X}\times {\cal Y}$, and $\psi$ has a convex graph, then $f^\star$ is a concave function and $\psi^\star$ is a convex-valued upper-hemi-continuous correspondence.
\end{enumerate}
\end{theorem}

\begin{lemma}\label{aux_lemma_optimising_set_convex}
    Let \[\Omega^\star({\cal M})\coloneqq \left\{\omega \in \Omega(M)\,\Bigg|\, \inf_{ \calm'\in \alt(\calm)} \sum_{s,a}\omega(s,a){\rm KL}(P(s,a),P'(s,a))= (T^\star(\calm))^{-1}\right\}.\]
    Then the set of optimizers $\Omega^\star({\cal M})$ is  convex. 
\end{lemma}
\begin{proof}
    The idea of the proof is to apply twice Berge's maximum Theorem \ref{thm:berge_maximum}. First, assume $\calm$ to be fixed, and  consider the function
    \[
    f(\omega, \calm')= \sum_{s,a}\omega(s,a){\rm KL}(P(s,a),P'(s,a)).
    \]
    Let ${\cal X}=\Omega(\calm)$ and $\psi(\omega)\equiv\psi\coloneqq \overline\alt(\calm)$. Clearly, $f$ is continuous in $(\omega,\calm')$ and the set $\psi$ is constant, non-empty and compact. Then, by Theorem \ref{thm:berge_maximum} the function $f^\star(\omega)=\min\{f(\omega,\calm')\,|\,\calm'\in \psi\}$ is continuous  and $\psi^\star=\{\calm' \in \psi\,|\, f(\omega,\calm')=f^\star(\omega)\}$ is a compact-valued upper-hemi-continuous correspondence.

    At this point, note that for any $\omega_1,\omega_2\in \Omega(\calm)$  and $\lambda\in [0,1]$ we have
    \begin{align*}
    f^\star((1-\lambda)\omega_1+\lambda \omega_2) &= \min_{\calm'\in\overline\alt(\calm)} f((1-\lambda)\omega_1+\lambda \omega_2,\calm'),\\
    &= \min_{\calm'\in\overline\alt(\calm)} (1-\lambda)f(\omega_1,\calm') + \lambda f(\omega_2,\calm') \tag{since $f$ is linear in $\omega$},\\
    &\geq \min_{\calm'\in\overline\alt(\calm)} (1-\lambda)f(\omega_1,\calm') + \min_{\calm'\in\overline\alt(\calm)}\lambda f(\omega_2,\calm') ,\\
    &= (1-\lambda)f^\star(\omega_1)+\lambda f^\star(\omega_2).
    \end{align*}
    Hence $f^\star$ is concave in $\omega$.
    At this point we can apply the theorem again to $f^\star(\omega)$, so that we consider the values $\max\{f^\star(\omega)\,|\,\omega\in \Omega(\calm)\}$. We let ${\cal X}=\{\calm\}$ and ${\cal Y}=\Delta(S\times A)$, with $\psi(\calm)=\Omega(\calm)$ (which is compact by the finiteness of the MDP). Since $f^\star$ is continuous and concave, and $\psi$ is a convex set, then $\{\omega \in \Omega(M)\,|\, f^\star(\omega)=(T^\star(\calm))^{-1}\}$ is a convex-valued set  by Lemma \ref{thm:berge_maximum}. 
\end{proof}

\begin{lemma}
    Let
    \[
    f(\omega, \calm')= \sum_{s,a}\omega(s,a){\rm KL}(P(s,a),P'(s,a)).
    \]
    and $f^\star(\omega)=\min\{f(\omega,\calm')\,|\,\calm'\in \alt(\cal M)\}$. Then $f^*(\omega)$ is continuous and concave in $\omega$.
\end{lemma}
\begin{proof}
    This is an immediate consequence from the proof of Lemma \ref{aux_lemma_optimising_set_convex}.
\end{proof}

We also add the result here that if the uniform policy is optimal in a communcating MDP, then we have that it is not unique.

\begin{lemma}\label{aux_lemma_pi_u_unique}
Let $(S, A, P, r, \gamma)$ be a finite discounted Markov decision process (MDP) with 
$0 < \gamma < 1$ and $|A| \ge 2$. 
Suppose that the \emph{uniform policy}
\[
\pi_u(a|s) = \frac{1}{|A|}, \quad \forall s \in S, \, a \in A
\]
is optimal. Then $\pi_u$ is not unique; in fact, every policy is optimal.
\end{lemma}

\begin{proof}
Let $V^\pi$ denote the value function of policy $\pi$:
\[
V^\pi(s) = \mathbb{E}_\pi \!\left[ \sum_{t=0}^{\infty} 
\gamma^t r(s_t, a_t) \,\middle|\, s_0 = s \right].
\]
Define the optimal value and $Q$-functions by
\[
V^*(s) = \max_{a \in A} Q^*(s,a), \qquad
Q^*(s,a) = r(s,a) + \gamma \sum_{s'} P(s'|s,a) V^*(s').
\]

\paragraph{Step 1. Characterization of optimal policies.}
A standard result in dynamic programming (see, e.g., 
Puterman, \emph{Markov Decision Processes}, 1994, Proposition~6.2.3)
states that a policy $\pi$ is optimal if and only if, for all $s \in S$,
\begin{equation}
\pi(a|s) > 0 \; \Rightarrow \; Q^*(s,a) = V^*(s).
\label{eq:opt}
\end{equation}
In words: a policy may assign positive probability only to actions that
achieve the maximal $Q^*(s,a)$ value in each state.

\paragraph{Step 2. Apply~\eqref{eq:opt} to the uniform policy.}
Assume that the uniform policy $\pi_u$ is optimal. 
Since $\pi_u(a|s) = 1/|A| > 0$ for every $s$ and $a$, 
substituting into~\eqref{eq:opt} gives
\begin{equation}
Q^*(s,a) = V^*(s), \quad \forall s \in S, \, a \in A.
\label{eq:qequal}
\end{equation}
Thus, in each state, all actions have identical optimal $Q$-values.

\paragraph{Step 3. The set of optimal actions.}
Define the set of optimal actions in state $s$ as
\[
A^*(s) = \{ a \in A : Q^*(s,a) = V^*(s) \}.
\]
Equation~\eqref{eq:qequal} implies
\[
A^*(s) = A, \quad \forall s \in S.
\]
Hence, every action is optimal in every state.

\paragraph{Step 4. The set of optimal policies.}
Let $\Pi^*$ denote the set of optimal policies:
\[
\Pi^* = \bigl\{ \pi : \pi(a|s) > 0 \Rightarrow a \in A^*(s),
\ \forall s \in S \bigr\}.
\]
Since $A^*(s) = A$ for all $s$, we have
\[
\Pi^* = \{ \pi : \pi(\cdot|s) \in \Delta(A), \ \forall s \in S \},
\]
where $\Delta(A)$ is the simplex over $A$. 
That is, \emph{every possible policy} (deterministic or stochastic) is optimal.

\paragraph{Step 5. Non-uniqueness.}
Because $\Pi^*$ contains all policies, the uniform policy $\pi_u$ is not unique.
Indeed, any policy $\pi'$ differing from $\pi_u$ on at least one $(s,a)$ is also optimal.

\paragraph{Step 6. Degenerate exception.}
If $|A|=1$, the MDP admits only one policy. 
In that trivial case the uniform policy is uniquely optimal, but this is vacuous.

\paragraph{Conclusion.}
For any finite discounted MDP with $|A|>1$, 
if the uniform policy is optimal, it cannot be the unique optimal policy. 
\end{proof}

\newpage
\section{Sharpness of the Objective Function} \label{appendix_sharpness}

In this section we discuss in more detail the formal sharpness results used earlier in the appendix and within the main paper. In particular, we will present Lemma \ref{lem:corrected_selector_stability} which, under some `assumptions' (A2)-(A5), illustrates that for MDPs with similar transitions, the minimum 2-norm selected optimal allocations are similar and therefore their respective induced policies are similar too. After this, we show that (A2)-(A5) hold in our problem setting and so the result is valid and can be used in our analysis. After this, we discuss some particular examples of MDP problem settings and their respective sharpnesses.

\begin{remark}
    The sharpness condition used in our analysis is stated in Assumption (A4) of 
    Lemma~\ref{lem:corrected_selector_stability}. For the purposes of our sample 
    complexity bounds, it is sufficient to impose this condition only locally, namely 
    for suboptimal allocations lying within a fixed neighbourhood of the optimal 
    allocation set. Such a localization affects the analysis only through constant 
    factors, since it merely restricts the range of perturbations for which the 
    selector-stability argument is applied.

    Thus, although Lemma~\ref{lem:corrected_selector_stability} is stated using a 
    global sharpness condition, a local version is enough for our application. This is 
    also the most natural notion for studying the objective $T$: the relevant quantity 
    is the local curvature of the objective on paths near the optimal allocation set. For this 
    reason, the examples later in the appendix focus on local sharpness.
\end{remark}

\subsection{Main sharpness lemma}
Before stating the main sharpness lemma, we introduce the following notation. Define the distance to a set
\[
\mathrm{dist}_\infty(\omega,S(M))
:=
\inf_{\omega'\in S(M)} \|\omega-\omega'\|_\infty .
\]

\begin{lemma}[Local stability of the minimum-norm optimal allocation and induced navigating policy]
\label{lem:corrected_selector_stability}
Let \(M=(\mathcal S,\mathcal A,P,r,\gamma)\) be a finite discounted MDP with deterministic rewards,
and let \(M'=(\mathcal S,\mathcal A,P',r,\gamma)\) be another MDP with the same reward function.
Write
\[
    \varepsilon
    :=
    \|P'-P\|_{\infty,1}
    :=
    \max_{s,a}
    \|P'(\cdot\mid s,a)-P(\cdot\mid s,a)\|_1 .
\]
For an MDP \(N=(\mathcal S,\mathcal A,P_N,r,\gamma)\), define
\[
    J_N(\omega)
    :=
    \inf_{\widetilde N\in \alt(N)}
    \sum_{s,a}
    \omega(s,a)
    \KL\!\left(
        P_N(\cdot\mid s,a),
        P_{\widetilde N}(\cdot\mid s,a)
    \right),
\]
and
\[
    J_N^\star
    :=
    \max_{\omega\in\Omega(N)}J_N(\omega),
    \qquad
    S(N)
    :=
    \argmax_{\omega\in\Omega(N)}J_N(\omega).
\]
Assume that \(S(N)\) is non-empty, compact, and convex for every \(N\) in the local neighbourhood
considered below. Define the minimum-norm optimal allocation selector by
\[
    \omega^\dagger(N)
    :=
    \argmin_{\omega\in S(N)}\|\omega\|_2^2 .
\]
Since \(S(N)\) is non-empty, compact, and convex, and since
\(\omega\mapsto\|\omega\|_2^2\) is strictly convex, \(\omega^\dagger(N)\) is uniquely defined.

Assume that 
the following conditions hold for every
MDP \(N=(\mathcal S,\mathcal A,P_N,r,\gamma)\) with the same reward function as \(M\) 

\begin{enumerate}

    \item[\textup{(A2)}] \textbf{Objective regularity.}
    There exist constants \(L_M,L_\omega<\infty\), depending on \(M\) but not on \(N\), such that
    for all admissible allocations \(\omega,\omega_1,\omega_2\),
    \[
        |J_N(\omega)-J_M(\omega)|
        \le
        L_M\|P_N-P\|_{\infty,1},
    \]
    and
    \[
        |J_M(\omega_1)-J_M(\omega_2)|
        \le
        L_\omega\|\omega_1-\omega_2\|_\infty .
    \]
    We also assume the same objective-regularity bound holds locally with \(M\) and \(N\)
    interchanged, with the same constants up to enlarging them.

    \item[\textup{(A3)}] \textbf{Feasible-set stability.}
    There exists \(\kappa_\Omega(M)<\infty\), depending on \(M\) but not on \(N\), such that
    \[
        d_H\bigl(\Omega(N),\Omega(M)\bigr)
        \le
        \kappa_\Omega(M)\|P_N-P\|_{\infty,1},
    \]
    where \(d_H\) denotes Hausdorff distance with respect to \(\|\cdot\|_\infty\).

    \item[\textup{(A4)}] \textbf{Uniform sharpness.}
    There exists a non-decreasing function
    \[
        \varphi_M:[0,\infty)\to[0,\infty),
        \qquad
        \varphi_M(0)=0,
    \]
    such that, for 
    every
    \(\omega\in\Omega(N)\),
    \[
        \mathrm{dist}_\infty(\omega,S(N))
        \le
        \varphi_M\!\left(J_N^\star-J_N(\omega)\right).
    \]

    \item[\textup{(A5)}] \textbf{Positive optimal state marginals.}
    The optimal allocations at \(M\) have uniformly positive state marginals:
    \[
        \rho(M)
        :=
        \min_{\omega\in S(M)}
        \min_{s\in\mathcal S}
        \omega(s)
        >0,
        \qquad
        \omega(s):=\sum_{a\in\mathcal A}\omega(s,a).
    \]
\end{enumerate}

Define
\[
    C_T(M)
    :=
    2L_M+2L_\omega\kappa_\Omega(M).
\]
Then
we have
\[
    \|\omega^\dagger(M')-\omega^\dagger(M)\|_\infty
    \le
    \sqrt{
        2\left[
            \kappa_\Omega(M)\|P'-P\|_{\infty,1}
            +
            \varphi_M\!\left(C_T(M)\|P'-P\|_{\infty,1}\right)
        \right]
    } .
\]

Furthermore, the induced policies
\[
    \pi^\dagger_N(a\mid s)
    :=
    \frac{\omega^\dagger(N)(s,a)}
    {\sum_{b\in\mathcal A}\omega^\dagger(N)(s,b)}
\]
satisfy
\[
    \|\pi^\dagger_{M'}-\pi^\dagger_M\|_\infty
    \le
    \frac{(A+1)}{\rho(M)}
    \sqrt{
        2\left[
            \kappa_\Omega(M)\|P'-P\|_{\infty,1}
            +
            \varphi_M\!\left(C_T(M)\|P'-P\|_{\infty,1}\right)
        \right]
    } .
\]

In particular, on any event of the form
\[
    \mathcal E_T(\lambda)
    :=
    \left\{
        \forall t\le T:
        \|\widehat P_t-P\|_{\infty,1}
        \le
        \psi(t,\lambda)
    \right\},
\]
we have
\[
    \|\omega^\dagger(\widehat M_t)-\omega^\dagger(M)\|_\infty
    \le
    \sqrt{
        2\left[
            \kappa_\Omega(M)\psi(t,\lambda)
            +
            \varphi_M\!\left(C_T(M)\psi(t,\lambda)\right)
        \right]
    } 
\]
and thus
\[
    \|\pi^\dagger_{\widehat M_t}-\pi^\dagger_M\|_\infty
    \le
    \frac{A+1}{\rho(M)}
    \sqrt{
        2\left[
            \kappa_\Omega(M)\psi(t,\lambda)
            +
            \varphi_M\!\left(C_T(M)\psi(t,\lambda)\right)
        \right]
    } .
\]
\end{lemma}

\begin{proof}
Fix \(M'\in \mathcal{M}\).
We first prove a Hausdorff stability bound between \(S(M')\) and \(S(M)\).

Let \(\omega'\in S(M')\). By Assumption~\textup{(A3)}, there exists
\(\bar\omega'\in\Omega(M)\) such that
\[
    \|\omega'-\bar\omega'\|_\infty
    \le
    \kappa_\Omega(M)\varepsilon .
\]
Let \(\omega^\star\in S(M)\). Again by Assumption~\textup{(A3)}, there exists
\(\bar\omega^\star\in\Omega(M')\) such that
\[
    \|\omega^\star-\bar\omega^\star\|_\infty
    \le
    \kappa_\Omega(M)\varepsilon .
\]
Since \(\omega'\in S(M')\), it is optimal for \(J_{M'}\) over \(\Omega(M')\). Therefore,
\[
    J_{M'}(\omega')
    \ge
    J_{M'}(\bar\omega^\star).
\]
Using Assumption~\textup{(A2)}, we obtain
\[
\begin{aligned}
    J_M(\bar\omega')
    &\ge
    J_M(\omega') - L_\omega\|\bar\omega'-\omega'\|_\infty
    \\
    &\ge
    J_{M'}(\omega') - L_M\varepsilon
       - L_\omega\kappa_\Omega(M)\varepsilon
    \\
    &\ge
    J_{M'}(\bar\omega^\star) - L_M\varepsilon
       - L_\omega\kappa_\Omega(M)\varepsilon
    \\
    &\ge
    J_M(\bar\omega^\star) - 2L_M\varepsilon
       - L_\omega\kappa_\Omega(M)\varepsilon
    \\
    &\ge
    J_M(\omega^\star)
       -2L_M\varepsilon
       -2L_\omega\kappa_\Omega(M)\varepsilon .
\end{aligned}
\]
Since \(\omega^\star\in S(M)\), \(J_M(\omega^\star)=J_M^\star\). Thus
\[
    J_M^\star-J_M(\bar\omega')
    \le
    \left(2L_M+2L_\omega\kappa_\Omega(M)\right)\varepsilon
    =
    C_T(M)\varepsilon .
\]
By Assumption~\textup{(A4)} applied with \(N=M\),
\[
    \mathrm{dist}_\infty(\bar\omega',S(M))
    \le
    \varphi_M(C_T(M)\varepsilon).
\]
Therefore,
\[
\begin{aligned}
    \mathrm{dist}_\infty(\omega',S(M))
    &\le
    \|\omega'-\bar\omega'\|_\infty
    +
    \mathrm{dist}_\infty(\bar\omega',S(M))
    \\
    &\le
    \kappa_\Omega(M)\varepsilon
    +
    \varphi_M(C_T(M)\varepsilon).
\end{aligned}
\]
Taking the supremum over \(\omega'\in S(M')\), we obtain
\[
    \sup_{\omega'\in S(M')}
    \mathrm{dist}_\infty(\omega',S(M))
    \le
    \kappa_\Omega(M)\varepsilon
    +
    \varphi_M(C_T(M)\varepsilon).
\]

We now prove the reverse inclusion bound. Let \(\omega\in S(M)\). By Assumption~\textup{(A3)},
there exists \(\bar\omega\in\Omega(M')\) such that
\[
    \|\omega-\bar\omega\|_\infty
    \le
    \kappa_\Omega(M)\varepsilon .
\]
Let \(\omega'^\star\in S(M')\). Again by Assumption~\textup{(A3)}, there exists
\(\bar\omega'^\star\in\Omega(M)\) such that
\[
    \|\omega'^\star-\bar\omega'^\star\|_\infty
    \le
    \kappa_\Omega(M)\varepsilon .
\]
Since \(\omega'^\star\in S(M')\),
\[
    J_{M'}^\star=J_{M'}(\omega'^\star).
\]
Using Assumption~\textup{(A2)} and the optimality of \(\omega\in S(M)\), we have
\[
\begin{aligned}
    J_{M'}^\star
    &=
    J_{M'}(\omega'^\star)
    \\
    &\le
    J_M(\omega'^\star)+L_M\varepsilon
    \\
    &\le
    J_M(\bar\omega'^\star)
    +
    L_\omega\|\omega'^\star-\bar\omega'^\star\|_\infty
    +
    L_M\varepsilon
    \\
    &\le
    J_M(\bar\omega'^\star)
    +
    L_\omega\kappa_\Omega(M)\varepsilon
    +
    L_M\varepsilon
    \\
    &\le
    J_M^\star
    +
    L_\omega\kappa_\Omega(M)\varepsilon
    +
    L_M\varepsilon
    \\
    &=
    J_M(\omega)
    +
    L_\omega\kappa_\Omega(M)\varepsilon
    +
    L_M\varepsilon .
\end{aligned}
\]
On the other hand,
\[
\begin{aligned}
    J_{M'}(\bar\omega)
    &\ge
    J_M(\bar\omega)-L_M\varepsilon
    \\
    &\ge
    J_M(\omega)
    -
    L_\omega\|\bar\omega-\omega\|_\infty
    -
    L_M\varepsilon
    \\
    &\ge
    J_M(\omega)
    -
    L_\omega\kappa_\Omega(M)\varepsilon
    -
    L_M\varepsilon .
\end{aligned}
\]
Combining the two displays gives
\[
    J_{M'}^\star-J_{M'}(\bar\omega)
    \le
    2L_M\varepsilon+2L_\omega\kappa_\Omega(M)\varepsilon
    =
    C_T(M)\varepsilon .
\]
By Assumption~\textup{(A4)} applied with \(N=M'\),
\[
    \mathrm{dist}_\infty(\bar\omega,S(M'))
    \le
    \varphi_M(C_T(M)\varepsilon).
\]
Therefore,
\[
\begin{aligned}
    \mathrm{dist}_\infty(\omega,S(M'))
    &\le
    \|\omega-\bar\omega\|_\infty
    +
    \mathrm{dist}_\infty(\bar\omega,S(M'))
    \\
    &\le
    \kappa_\Omega(M)\varepsilon
    +
    \varphi_M(C_T(M)\varepsilon).
\end{aligned}
\]
Taking the supremum over \(\omega\in S(M)\), we obtain
\[
    \sup_{\omega\in S(M)}
    \mathrm{dist}_\infty(\omega,S(M'))
    \le
    \kappa_\Omega(M)\varepsilon
    +
    \varphi_M(C_T(M)\varepsilon).
\]
Consequently,
\[
    d_H(S(M'),S(M))
    \le
    \kappa_\Omega(M)\varepsilon
    +
    \varphi_M(C_T(M)\varepsilon).
\]
Set
\[
    h
    :=
    d_H(S(M'),S(M)).
\]
Then
\[
    h
    \le
    \kappa_\Omega(M)\varepsilon
    +
    \varphi_M(C_T(M)\varepsilon).
\]

We now control the minimum-norm selectors. Let
\[
    x:=\omega^\dagger(M),
    \qquad
    y:=\omega^\dagger(M').
\]
Because \(S(M)\) and \(S(M')\) are non-empty, compact, and convex, \(x\) and \(y\) are the
Euclidean projections of \(0\) onto \(S(M)\) and \(S(M')\), respectively.

Since \(d_H(S(M'),S(M))=h\), there exists \(\bar y\in S(M)\) such that
\[
    \|y-\bar y\|_2\le h,
\]
and there exists \(\bar x\in S(M')\) such that
\[
    \|x-\bar x\|_2\le h.
\]
The projection variational inequality for \(x=\Pi_{S(M)}(0) = \argmin_{\omega\in S(M)}\|\omega\|_2^2\) gives
\[
    \langle x,\bar y-x\rangle\ge 0.
\]
Hence
\[
\begin{aligned}
    \langle x,y-x\rangle
    &=
    \langle x,\bar y-x\rangle
    +
    \langle x,y-\bar y\rangle
    \\
    &\ge
    -\|x\|_2\|y-\bar y\|_2
    \\
    &\ge
    -h,
\end{aligned}
\]
where we have used Cauchy-Schwarz and \(\|x\|_2\le 1\), since \(x\) is a probability vector. Similarly, the projection
variational inequality for \(y=\Pi_{S(M')}(0)\) gives
\[
    \langle y,x-y\rangle\ge -h.
\]
Therefore,
\[
\begin{aligned}
    \|x-y\|_2^2
    &=
    \langle x-y,x-y\rangle
    \\
    &=
    \langle x,x-y\rangle
    +
    \langle y,y-x\rangle
    \\
    &\le
    h+h
    =
    2h .
\end{aligned}
\]
Thus,
\[
    \|x-y\|_\infty
    \le
    \|x-y\|_2
    \le
    \sqrt{2h}.
\]
Substituting the bound on \(h\) gives
\[
    \|\omega^\dagger(M')-\omega^\dagger(M)\|_\infty
    \le
    \sqrt{
        2\left[
            \kappa_\Omega(M)\varepsilon
            +
            \varphi_M(C_T(M)\varepsilon)
        \right]
    } .
\]

It remains to prove the policy bound. For any $\omega_1,\omega_2$, we have
\[
\|\pi_{\omega_1}-\pi_{\omega_2}\|_\infty
\le
\frac{A+1}{\rho(M)}\|\omega_1-\omega_2\|_\infty
\]
Combining this with the earlier allocation bound gives
\[
    \|\pi^\dagger_{M'}-\pi^\dagger_M\|_\infty
    \le
    \frac{(A+1)}{\rho(M)}
    \sqrt{
        2\left[
            \kappa_\Omega(M)\varepsilon
            +
            \varphi_M(C_T(M)\varepsilon)
        \right]
    } .
\]

Finally, the empirical statements follow by setting \(M'=\widehat M_t\). On
\(\mathcal E_T(\lambda)\), if \(\psi(t,\lambda)\le r_M\), then
\[
    \|\widehat P_t-P\|_{\infty,1}\le \psi(t,\lambda),
\]
so the deterministic bounds above apply with
\[
    \varepsilon=\|\widehat P_t-P\|_{\infty,1}\le \psi(t,\lambda).
\]
\end{proof}

\subsection{Assumptions in Lemma~\ref{lem:corrected_selector_stability} hold in our setting}

We now discuss that all assumptions for Lemma \ref{lem:corrected_selector_stability} hold and therefore our the lemma itself holds. 

\paragraph{Assumption (A2).} The first part of Assumption A3 can be proven with Lemma \ref{lemma:KL_on_G_is_lipschitz}, with $L_M = S^3A(1+\log(L))$. The second part of the assumption can be shown directly by our $G_\mu(L)$ assumption (see Assumption \ref{assumption_P_in_G_bounded_ratio}) where we can find $L_\omega = 2SA\log(L)$.

\paragraph{Assumption (A3).} We can illustrate that this bound holds using Hoffman's error bound. We do so through the following lemma. 

\begin{lemma}[Justification of Assumption~\textup{(A3)}: feasible-set stability]
\label{lem:feasible_set_stability}
Let
\[
    \Omega(M)
    =
    \left\{
    \omega\in\Delta(\mathcal S\times\mathcal A):
    \sum_a\omega(s,a)
    =
    \sum_{s',a'}P(s\mid s',a')\omega(s',a')
    \quad
    \forall s\in\mathcal S
    \right\}.
\]
Then there exist constant \(\kappa_\Omega(M)<\infty\) such that
we have
\[
    d_H(\Omega(N),\Omega(M))
    \le
    \kappa_\Omega(M)\|P_N-P\|_{\infty,1}.
\]
\end{lemma}

\begin{proof}

Recall that, for an MDP \(N=(\mathcal S,\mathcal A,P_N,r,\gamma)\),
\[
    \Omega(N)
    =
    \left\{
        \omega\in\Delta(\mathcal S\times\mathcal A):
        \sum_a \omega(s,a)
        =
        \sum_{s',a'} P_N(s\mid s',a')\omega(s',a')
        \quad
        \forall s\in\mathcal S
    \right\}.
\]
Equivalently,
\[
    \Omega(N)
    =
    \left\{
        \omega\in\mathbb R^{SA}:
        \omega\ge 0,\ 
        \mathbf 1^\top \omega=1,\ 
        B(P_N)\omega=0
    \right\},
\]
where \(B(P_N)\omega=0\) encodes the stationarity equations
\[
    \sum_a \omega(s,a)
    -
    \sum_{s',a'} P_N(s\mid s',a')\omega(s',a')
    =
    0,
    \qquad s\in\mathcal S .
\]
Thus \(\Omega(N)\) is a polyhedron of the form
\[
    \{x\in\mathbb R^{SA}: Ax\le b,\ C_Nx=d\},
\]
where we may take
\[
    A=-I,
    \qquad
    b=0,
\]
to encode the non-negativity constraints \(\omega\ge 0\), and
\[
    C_N
    =
    \begin{pmatrix}
        \mathbf 1^\top \\
        B(P_N)
    \end{pmatrix},
    \qquad
    d
    =
    \begin{pmatrix}
        1\\
        0
    \end{pmatrix}.
\]

Fix \(\omega\in\Omega(M)\). We want to bound
\[
    \mathrm{dist}_\infty(\omega,\Omega(N)).
\]
Since \(\omega\in\Omega(M)\), we have
\[
    \omega\ge 0,
    \qquad
    \mathbf 1^\top\omega=1,
    \qquad
    B(P)\omega=0.
\]
When \(\omega\) is tested against the constraints defining \(\Omega(N)\), the inequality
residual is therefore zero:
\[
    (-I\omega-0)_+
    =
    (-\omega)_+
    =
    0,
\]
and the normalisation residual is also zero:
\[
    \mathbf 1^\top\omega-1=0.
\]
The only non-zero residual is the stationarity residual:
\[
    B(P_N)\omega
    =
    B(P_N)\omega-B(P)\omega
    =
    \bigl(B(P_N)-B(P)\bigr)\omega.
\]
Since \(B(P)\) depends linearly on \(P\), there exists a finite dimension-dependent constant
\(c_B>0\) such that
\[
    \|B(P_N)-B(P)\|_\infty
    \le
    c_B\|P_N-P\|_{\infty,1}.
\]
Moreover, since \(\omega\in\Delta(\mathcal S\times\mathcal A)\), we have
\[
    \|\omega\|_1=1.
\]
Hence
\[
    \|B(P_N)\omega\|_\infty
    \le
    c_B\|P_N-P\|_{\infty,1}.
\]

Applying Hoffman's lemma (Lemma \ref{hoffman_lemma_appendix}) to the fixed polyhedron \(\Omega(N)\), there exists a constant
\(H_N<\infty\) such that
\[
    \mathrm{dist}_\infty(\omega,\Omega(N))
    \le
    H_N
    \left(
        \|(-\omega)_+\|_\infty
        +
        |\mathbf 1^\top\omega-1|
        +
        \|B(P_N)\omega\|_\infty
    \right).
\]
The first two terms are zero, and the third term is bounded as above. Therefore
\[
    \mathrm{dist}_\infty(\omega,\Omega(N))
    \le
    H_Nc_B\|P_N-P\|_{\infty,1}.
\]
Thus, for each fixed \(N\), the one-sided deviation satisfies
\[
    \sup_{\omega\in\Omega(M)}
    \mathrm{dist}_\infty(\omega,\Omega(N))
    \le
    H_Nc_B\|P_N-P\|_{\infty,1}.
\]
\end{proof}

\paragraph{Assumption (A4).} 
Assumption (A4) is quite weak and it does not impose
any linear/polynomial rate. Such a modulus function $\varphi$ always exists when $\Omega(M)$ is compact and
$T(M,\cdot)$ is continuous (which holds in our setting, see Berge's theorem and the proof of Lemma \ref{aux_lemma_optimising_set_convex}). Namely, define
\[
g_M(r)
:=
\min_{\omega:\,\mathrm{dist}_\infty(\omega,S(M))\ge r}
J^*_M-J_M(\omega),
\]
which is strictly positive for $r>0$,
and let
\[
\varphi_M(u)
:=
\sup\{r:\,g_M(r)\le u\}.
\]
Then (A4) holds with this $\varphi_M$. Although we discuss later (and in the main body of the paper) particular case-studies $M$ in which we are able to provide explicit forms for this sharpness function $\varphi_M$.

\paragraph{Assumption (A5).} This statement follows quickly from Assumption \ref{assumption_w_0_condition} which, as discussed in the main paper, enforces all optimal allocations for all state-action pairs to be strictly positive by Lemma \ref{lemma_auxiliary_positive_omegas}.

\begin{lemma}[Justification of Assumption~\textup{(A5)}: positive optimal state marginals]
\label{lem:positive_state_marginals}
Suppose every optimal allocation assigns positive mass to every state-action pair:
\[
    \omega(s,a)>0
    \qquad
    \forall \omega\in S(M),\ \forall(s,a)\in\mathcal S\times\mathcal A.
\]
Then
\[
    \rho(M)
    :=
    \min_{\omega\in S(M)}
    \min_{s\in\mathcal S}
    \sum_a\omega(s,a)
    >0.
\]
\end{lemma}

\begin{proof}
For every \(\omega\in S(M)\) and every \(s\in\mathcal S\),
\[
    \omega(s)
    =
    \sum_a\omega(s,a)>0.
\]
The map
\[
    \omega\mapsto \min_{s\in\mathcal S}\omega(s)
\]
is continuous. Since \(S(M)\) is compact, this map attains its minimum over \(S(M)\). By the
strict positivity assumption, that minimum cannot be zero. Hence
\[
    \rho(M)
    =
    \min_{\omega\in S(M)}
    \min_{s\in\mathcal S}
    \omega(s)
    >0.
\]
\end{proof}

\paragraph{Useful results for the above proofs.} Within the above proof, we used a consequence of Hoffman's Lemma which we state here.

\begin{lemma}[Hoffman error bound]\label{hoffman_lemma_appendix}
Let
\[
S = \{x \in \mathbb{R}^n : Ax \le b,\; Cx = d\}
\]
be a nonempty polyhedron. Then, for any choice of norms, there exists
a constant $\kappa > 0$ such that for all $x$,
\[
\mathrm{dist}(x,S)
\;\le\;
\kappa \Big( \|(Ax - b)_+\| + \|Cx - d\| \Big).
\]
\end{lemma}
\begin{proof}
This is Hoffman's lemma; see \citep{Hoffman1952OnAS} and \citep[][equation 1.3]{hoffman_modern}.
\end{proof}

The following consequence of Hoffman's lemma will also be useful later.

\begin{lemma}[LP error bound from Hoffman]\label{lem:lp-error-bound}
Let \(P\subset\mathbb{R}^{d+1}\) be a nonempty polyhedron, and let
\[
t^\star=\max\{t:(\omega,t)\in P\}.
\]
Assume the maximum is attained, and define
\[
S^\star:=P\cap\{t=t^\star\}.
\]
Then there exists a constant \(\kappa_P>0\) such that for all
\((\omega,t)\in P\),
\[
\mathrm{dist}_\infty\big((\omega,t),S^\star\big)
\le
\kappa_P\,(t^\star-t).
\]
\end{lemma}

\begin{proof}
Since \(P\) is a polyhedron, there exist matrices \(A,C\) and vectors \(b,d\)
such that
\[
P=\{x\in\mathbb{R}^{d+1}: Ax\le b,\; Cx=d\}.
\]
Hence
\[
S^\star
=
\{x\in\mathbb{R}^{d+1}: Ax\le b,\; Cx=d,\; t\le t^\star,\; -t\le -t^\star\}
\]
is also a nonempty polyhedron. Applying Hoffman's error bound to this system,
there exists \(\kappa_P>0\) such that for all \(x=(\omega,t)\),
\[
\mathrm{dist}_\infty(x,S^\star)
\le
\kappa_P\Big(\|(Ax-b)_+\|_\infty+\|Cx-d\|_\infty+(t-t^\star)_+ + (t^\star-t)_+\Big).
\]
Now let \(x=(\omega,t)\in P\). Then \((Ax-b)_+=0\) and \(Cx-d=0\). Moreover,
since \(t^\star\) is the optimal value over \(P\), we have \(t\le t^\star\),
so \((t-t^\star)_+=0\). Therefore
\[
\mathrm{dist}_\infty\big((\omega,t),S^\star\big)
\le
\kappa_P\,(t^\star-t),
\]
as claimed.
\end{proof}

\subsection{Linear Sharpness Under a Finite Discretisation of the MDP Space}
\label{sec:appendix:discretization}

In many settings it is natural to approximate the alternative family by a finite discretization. 
For example, if the model class is parametrized by a compact parameter space, 
one may construct a finite grid approximation of the alternatives. 
Such discretizations are also common in algorithmic implementations and in lower bound arguments, 
where a finite collection of alternative models suffices to capture the intrinsic difficulty of the identification problem.

In this section we show that discretizing the alternative set restores linear sharpness. 
The key observation is that once the set of alternatives is finite, 
the function $T(M,\omega)$ becomes a concave piecewise-linear function of $\omega$. 
Maximizing such a function over the polyhedral feasible set $\Omega(M)$ can be written as a linear program in epigraph form. 
Classical error bounds for linear programs then imply a linear relationship between the suboptimality gap and the distance to the optimal set. 

We write this into a lemma and proof below. Informally, when the alternative family is discretized, the objective becomes piecewise linear and the optimization problem reduces to a linear program. 
In this case the suboptimality gap grows linearly with the distance to the optimal set, yielding the linear modulus in Assumption~A4.

\paragraph{Discretized objective.}
Let $\alt_\Delta(M) = \{M^1,\ldots,M^N\} \subset \alt(M)$ be a finite subset of alternatives. 
Define the discretized objective
\[
T_\Delta(M,\omega)
    := \min_{1 \le i \le N} \langle \omega, k(M^i) \rangle,
\]
where $k(M^i) = \KL(P_M,P_{M^i})$. Let
\[
T_\Delta^\star(M)
    := \max_{\omega \in \Omega(M)} T_\Delta(M,\omega),
\qquad
S_\Delta(M)
    := \arg\max_{\omega \in \Omega(M)} T_\Delta(M,\omega).
\]

\begin{proposition}[Linear sharpness under discretization]\label{prop:linear-sharp-discrete}
Assume that \(\Omega(M)\subset\mathbb{R}^d\) is a nonempty compact polytope.
Let
\[
\alt_\Delta(M)=\{M_1,\dots,M_N\}\subset \alt(M)
\]
be a finite discretization of the alternative set, and define
\[
T_\Delta(M,\omega)
:=
\min_{1\le i\le N}\langle \omega, k(M_i)\rangle,
\qquad \omega\in\Omega(M).
\]
Let
\[
T_\Delta^\star(M)
:=
\max_{\omega\in\Omega(M)} T_\Delta(M,\omega),
\qquad
S_\Delta(M)
:=
\argmax_{\omega\in\Omega(M)} T_\Delta(M,\omega).
\]
Then there exists a constant \(c_\Delta(M)>0\) such that for all
\(\omega\in\Omega(M)\),
\[
\mathrm{dist}_\infty\big(\omega,S_\Delta(M)\big)
\le
c_\Delta(M)\Big(T_\Delta^\star(M)-T_\Delta(M,\omega)\Big).
\]
Consequently, Assumption~(A4) in Lemma \ref{lem:corrected_selector_stability} holds for the discretized objective with linear
modulus
\[
\varphi_M(u)=c_\Delta(M)\,u.
\]
\end{proposition}

\begin{proof}
Write \(k_i:=k(M_i)\in\mathbb{R}^d\) for \(i=1,\dots,N\). Since the alternative
set is finite,
\[
T_\Delta(M,\omega)=\min_{1\le i\le N}\langle \omega,k_i\rangle
\]
is a concave piecewise-linear function of \(\omega\). Now, introduce the lifted feasible set
\[
P_\Delta(M)
:=
\Bigl\{
(\omega,t)\in\mathbb{R}^{d+1}
:
\omega\in\Omega(M),\;
t\le \langle \omega,k_i\rangle \ \forall i=1,\dots,N
\Bigr\}.
\]
Then maximizing \(T_\Delta(M,\omega)\) over \(\Omega(M)\) is equivalent to the
linear program
\[
T_\Delta^\star(M)=\max\{t:(\omega,t)\in P_\Delta(M)\}.
\]
Because \(\Omega(M)\) is a polytope and the additional constraints
\(t\le \langle \omega,k_i\rangle\) are linear, \(P_\Delta(M)\) is a nonempty
polyhedron. Since \(\Omega(M)\) is compact and each map
\(\omega \mapsto \langle \omega,k_i\rangle\) is continuous,
the function \(T_\Delta(M,\omega)=\min_i \langle \omega,k_i\rangle\)
is continuous on \(\Omega(M)\). Hence the maximum
\(T_\Delta^\star(M)\) is attained. Define the lifted optimal
set
\[
S_\Delta^{\mathrm{lift}}(M)
:=
P_\Delta(M)\cap\{t=T_\Delta^\star(M)\}.
\]
By Lemma~\ref{lem:lp-error-bound}, there exists a constant
\(\kappa_\Delta(M)>0\) such that for all \((\omega,t)\in P_\Delta(M)\),
\[
\mathrm{dist}_\infty\big((\omega,t),S_\Delta^{\mathrm{lift}}(M)\big)
\le
\kappa_\Delta(M)\big(T_\Delta^\star(M)-t\big).
\]

Now fix \(\omega\in\Omega(M)\), and set \(t:=T_\Delta(M,\omega)\). By
definition of \(T_\Delta\), we have \(t\le \langle \omega,k_i\rangle\) for all
\(i\), hence \((\omega,t)\in P_\Delta(M)\). Therefore
\[
\mathrm{dist}_\infty\big((\omega,T_\Delta(M,\omega)),S_\Delta^{\mathrm{lift}}(M)\big)
\le
\kappa_\Delta(M)\Big(T_\Delta^\star(M)-T_\Delta(M,\omega)\Big).
\]
Finally, projection onto the \(\omega\)-coordinates cannot increase
\(\ell_\infty\)-distance, so
\[
\mathrm{dist}_\infty\big(\omega,S_\Delta(M)\big)
\le
\kappa_\Delta(M)\Big(T_\Delta^\star(M)-T_\Delta(M,\omega)\Big).
\]
Setting \(c_\Delta(M):=\kappa_\Delta(M)\) completes the proof.
\end{proof}

\subsection{An explicit family realizing arbitrary rational exponents for polynomial sharpness}\label{appendix_subsection_arbitrary_sharpness}

\begin{figure}
    \centering
    \includegraphics[width=0.5\linewidth]{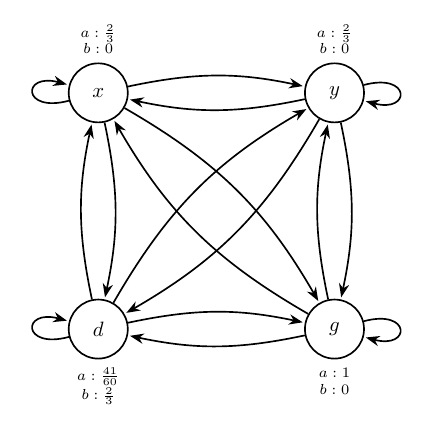}
    \caption{Visualization of the family of MDPs used in the construction in Proposition \ref{prop:counterexample}. We display the deterministic rewards from the two actions $a,b$ next to each of the states. The MDP is fully connected and the transitions vary parametrically through the family of MDPs. See Proposition \ref{prop:counterexample} for more details.}
    \label{fig:MDP}
\end{figure}

We now construct a concrete family of finite discounted MDPs for which the sharpness exponent along a feasible transverse direction can equal \((m-n)/m\) for any $m,n \geq 1$. In particular, we show that from any point on the boundary of $\Omega^*(M)$, there always exists a direction outwards from the optimal set $\Omega^*(M)$ in which the function $J_M(\cdot)$ is arbitrarily flat.  In this particular example constructed for Proposition \ref{prop:counterexample}, we note that the infinity norm distance between $\omega_t = \omega^*+th$ and $\Omega^*(M)$ is $t$ in this case. Namely $\mathrm{dist}_{\infty}(\omega_t,\Omega^*(M)) = t$. Hence, the final bounds on the curvature can be related to our definition of sharpness in Assumption (A4) of Lemma \ref{lem:corrected_selector_stability}. Hence, we have the statement in the main body of the paper, Theorem \ref{theorem: arbitrary_sharpness}.

\begin{proposition}[Explicit communicating and aperiodic family realizing every rational exponent \(>1\)]\label{prop:counterexample}
Fix integers \(m>n\ge 1\). Then there exist:
\begin{itemize}
\item a finite discounted communicating and aperiodic MDP \(M\) with \(|\mathcal{S}|=4\), \(|\mathcal{A}|=2\), rewards in \([0,1]\), and unique deterministic optimal policy \(\pi^\star\),
\item a compact semialgebraic feasible family \(\Theta\) of communicating and aperiodic models,
\item a known reference kernel \(\mu\) such that \(P_M\in G_\mu(L)\) for every \(L>1\), and in fact \(\Theta\subseteq G_\mu(4)\) with $\mu$ uniform,
\item a boundary point \(\omega^\star\in\Omega^\star(M)\),
\item and a feasible transverse direction \(h\), i.e., such that $\omega^\star + t h \in \Omega(M)$ for all sufficiently small $t > 0$
\end{itemize}
such that
\begin{equation}\label{eqn_appendix_dierctional_sharpness_arbitrary}
    J^\star-J_M(\omega^\star+t h)
=
c_{m,n}\, t^{m/(m-n)}+o\bigl(t^{m/(m-n)}\bigr)
\qquad (t\downarrow0),
\end{equation}

where
\[
c_{m,n}
=
\frac43\Bigl(1-\frac{n}{m}\Bigr)
\Bigl(\frac{4n}{m}\Bigr)^{\!\frac{n}{m-n}}
>0.
\]
In particular, every rational exponent \(p>1\) is achievable even within the communicating and aperiodic discounted setting.
\end{proposition}

\begin{proof}
We divide the proof into several steps.

\smallskip
\noindent\textbf{Step 1: the true MDP and the reward normalization.}
Let
\[
\mathcal{S}=\{x,y,d,g\},
\qquad
\mathcal{A}=\{a,b\},
\qquad
\gamma=\frac12.
\]
Define the true transition kernel by
\[
P_M(\cdot\mid s,a)=P_M(\cdot\mid s,b)=p
\qquad \forall s\in\mathcal{S},
\]
where
\[
p:=\Bigl(\frac14,\frac14,\frac14,\frac14\Bigr)
\]
in the state order \((x,y,d,g)\). Thus every row of every transition matrix equals \(p\). In particular, under every deterministic policy the induced state chain has strictly positive transition probabilities, so \(M\) is communicating and aperiodic.

The actual reward function used in the proposition is
\[
r(x,a)=r(y,a)=\frac23,
\qquad
r(d,a)=\frac{41}{60},
\qquad
r(g,a)=1,
\]
and
\[
r(x,b)=r(y,b)=r(g,b)=0,
\qquad
r(d,b)=\frac23.
\]
Hence \(r(s,\alpha)\in[0,1]\) for every state-action pair. For algebraic convenience, introduce the auxiliary reward
\[
\bar r:=3r-2,
\]
so that
\[
\bar r(x,a)=\bar r(y,a)=0,
\qquad
\bar r(d,a)=\frac{1}{20},
\qquad
\bar r(g,a)=1,
\]
and
\[
\bar r(x,b)=\bar r(y,b)=\bar r(g,b)=-2,
\qquad
\bar r(d,b)=0.
\]
The deterministic optimal-policy sets under \(r\) and \(\bar r\) coincide for the true model and for every alternative considered below. Therefore all optimality calculations may be carried out with \(\bar r\).

Under \(M\), the transition kernel is independent of the action, so action comparisons are determined entirely by immediate rewards. Since
\[
\bar r(s,a)-\bar r(s,b)>0
\qquad \forall s\in\mathcal{S},
\]
it follows that the policy
\[
\pi^\star(s)\equiv a
\qquad (s\in\mathcal{S})
\]
is the unique deterministic optimal policy of \(M\). The same is true for the actual reward function \(r\in[0,1]\).

\smallskip
\noindent\textbf{Step 2: the alternative family and the \(G_\mu(L)\) condition.}
Fix \(\bar u\in(0,1/16]\). For \(u\in[0,\bar u]\), define two probability vectors
\[
q_a(u):=\Bigl(\frac38+u^m,\frac14,\frac14,\frac18-u^m\Bigr),
\qquad
q_b(u):=\Bigl(\frac18+u^n,\frac14,\frac14,\frac38-u^n\Bigr),
\]
again in the state order \((x,y,d,g)\). Because \(\bar u\le 1/16\), all entries of \(q_a(u)\) and \(q_b(u)\) belong to \([1/16,7/16]\), in particular they are strictly positive.

Now define \(M_u\) by modifying only the two transitions at state \(d\):
\[
P_{M_u}(\cdot\mid d,a)=q_a(u),
\qquad
P_{M_u}(\cdot\mid d,b)=q_b(u),
\]
while for every other state-action pair \((s,\alpha)\neq(d,a),(d,b)\),
\[
P_{M_u}(\cdot\mid s,\alpha)=p.
\]
Set
\[
\Theta:=\{M\}\cup\{M_u:0\le u\le \bar u\}.
\]
 The family \(\Theta\) is compact, and because it is parameterized polynomially by \(u\), it is semialgebraic as a subset of the product of simplices. Since every transition probability of every \(M_u\) is strictly positive, every \(M_u\) is also communicating and aperiodic.

Define the reference kernel \(\mu\) by
\[
\mu(\cdot\mid s,\alpha)=p
\qquad \forall (s,\alpha)\in\mathcal{S}\times\mathcal{A}.
\]
Then \(P_M=\mu\), so \(P_M\in G_\mu(L)\) for every \(L>1\). Moreover, because every entry of every transition row in the family belongs to \([1/16,7/16]\), one has for every \(M_u\in\Theta\) and every state-action-state triple \((s,\alpha,s')\),
\[
\frac{\mu(s'\mid s,\alpha)}{P_{M_u}(s'\mid s,\alpha)}
\in
\left[\frac{1/4}{7/16},\frac{1/4}{1/16}\right]
=
\left[\frac47,4\right]
\subseteq [1/4,4].
\]
Hence \(\Theta\subseteq G_\mu(4)\).

\smallskip
\noindent\textbf{Step 3: the alternative policy is strictly optimal for all small \(u\).}
Define the deterministic policy
\[
\pi_u(x)=a,\qquad
\pi_u(y)=a,\qquad
\pi_u(d)=b,\qquad
\pi_u(g)=a.
\]
We first verify strict optimality at \(u=0\). Under \(M_0\), the policy \(\pi_0\) has value vector
\[
V_{M_0}^{\pi_0}
=
\Bigl(\frac{17}{64},\frac{17}{64},\frac{21}{64},\frac{81}{64}\Bigr),
\]
which is obtained by solving the Bellman system
\[
V(x)=\frac12\cdot\frac{V(x)+V(y)+V(d)+V(g)}{4},
\]
\[
V(y)=\frac12\cdot\frac{V(x)+V(y)+V(d)+V(g)}{4},
\]
\[
V(d)=\frac12\Bigl(\frac18V(x)+\frac14V(y)+\frac14V(d)+\frac38V(g)\Bigr),
\]
\[
V(g)=1+\frac12\cdot\frac{V(x)+V(y)+V(d)+V(g)}{4}.
\]
Now compare actions state by state.

At \(x\) and \(y\), action \(b\) has the same transition kernel as action \(a\) but a reward smaller by \(2\), so
\[
Q_{M_0}^{\pi_0}(x,a)-Q_{M_0}^{\pi_0}(x,b)=2,
\qquad
Q_{M_0}^{\pi_0}(y,a)-Q_{M_0}^{\pi_0}(y,b)=2.
\]
At \(g\), action \(b\) has the same transition kernel as \(a\) but a reward smaller by \(3\), so
\[
Q_{M_0}^{\pi_0}(g,a)-Q_{M_0}^{\pi_0}(g,b)=3.
\]
At \(d\), one computes
\[
Q_{M_0}^{\pi_0}(d,b)
=
\frac12\Bigl(\frac18\cdot\frac{17}{64}+\frac14\cdot\frac{17}{64}
+\frac14\cdot\frac{21}{64}+\frac38\cdot\frac{81}{64}\Bigr)
=
\frac{21}{64},
\]
whereas
\[
Q_{M_0}^{\pi_0}(d,a)
=
\frac{1}{20}
+\frac12\Bigl(\frac38\cdot\frac{17}{64}+\frac14\cdot\frac{17}{64}
+\frac14\cdot\frac{21}{64}+\frac18\cdot\frac{81}{64}\Bigr)
=
\frac{81}{320}.
\]
Hence
\[
Q_{M_0}^{\pi_0}(d,b)-Q_{M_0}^{\pi_0}(d,a)=\frac{3}{40}>0.
\]
Thus \(\pi_0\) is the unique deterministic optimal policy of \(M_0\).

All Bellman action values for the auxiliary reward \(\bar r\) depend continuously on \(u\). Therefore, after possibly shrinking \(\bar u>0\), the same strict inequalities remain valid for every \(u\in[0,\bar u]\). Consequently, for all \(u\in[0,\bar u]\),
\[
\Pi^\star(M_u)=\{\pi_u\},
\qquad
\pi^\star\notin\Pi^\star(M_u),
\]
first for the auxiliary reward \(\bar r\), and therefore also for the actual reward \(r\in[0,1]\). 
In particular,
\[
\alt(M)=\{M_u:0\le u\le \bar u\}
\]
after replacing \(\Theta\) by the smaller compact family \(\{M\}\cup\{M_u:0\le u\le \bar u\}\).

\smallskip
\noindent\textbf{Step 4: the navigation polytope.}
Because every row of \(P_M\) equals \(p\), the balance equations become
\[
\omega(s)=\sum_{s',\alpha} p(s)\,\omega(s',\alpha)=p(s)\sum_{s',\alpha}\omega(s',\alpha)=p(s)
\qquad \forall s\in\mathcal{S}.
\]
Since \(\omega\in\Delta(\mathcal{S}\times\mathcal{A})\), this is equivalent to
\[
\omega(x)=\omega(y)=\omega(d)=\omega(g)=\frac14.
\]
Hence
\[
\Omega(M)
=
\Bigl\{\omega\in\Delta(\mathcal{S}\times\mathcal{A}):
\omega(x)=\omega(y)=\omega(d)=\omega(g)=\frac14\Bigr\}.
\]

\smallskip
\noindent\textbf{Step 5: the objective and the optimal-allocation set.}
For \(u\in[0,\bar u]\), define
\[
K_m(u):=\KL\bigl(p\,\|\,q_a(u)\bigr),
\qquad
L_n(u):=\KL\bigl(p\,\|\,q_b(u)\bigr).
\]
Since every state-action pair other than \((d,a)\) and \((d,b)\) has the same transition under \(M\) and \(M_u\), one has, for every \(\omega\in\Omega(M)\),
\[
J_M(\omega)
=
\inf_{u\in[0,\bar u]}
\Bigl[
\omega(d,a)K_m(u)+\omega(d,b)L_n(u)
\Bigr].
\]
Because \(\omega(d,a)+\omega(d,b)=1/4\) for every \(\omega\in\Omega(M)\), this is a one-dimensional action-allocation problem at state \(d\).

A direct computation gives
\[
K_m(0)=L_n(0)=:A=\frac14\log\frac{4}{3}.
\]
Moreover, for \(u>0\),
\[
K_m(u)>A>L_n(u).
\]
If \(\omega(d,b)=0\), then \(\omega(d,a)=1/4\) and
\[
J_M(\omega)=\inf_{u\in[0,\bar u]}\frac14 K_m(u)=\frac14 A.
\]
If \(\omega(d,b)>0\), then choosing \(u>0\) small enough yields
\[
\omega(d,a)K_m(u)+\omega(d,b)L_n(u)<\frac14 A,
\]
because the negative \(u^n\) variation in \(L_n(u)\) dominates the positive \(u^m\) variation in \(K_m(u)\) as \(u\downarrow0\). Hence
\[
J_M(\omega)<\frac14 A
\qquad \text{whenever }\omega(d,b)>0.
\]
Therefore
\[
J^\star=\frac14 A
\]
and
\[
\Omega^\star(M)
=
\{\omega\in\Omega(M):\omega(d,b)=0,\ \omega(d,a)=1/4\}.
\]
Choose
\[
\omega^\star(s,a)=\frac14,\qquad
\omega^\star(s,b)=0
\qquad (s\in\mathcal{S}).
\]
Then \(\omega^\star\in\Omega^\star(M)\), and it is a boundary point of \(\Omega^\star(M)\).

\smallskip
\noindent\textbf{Step 6: the transverse direction.}
Let
\[
h:=e_{(d,b)}-e_{(d,a)}.
\]
Because \(h\) preserves the state marginal at \(d\) and leaves all other state marginals unchanged, it is a feasible direction at \(\omega^\star\). It is not strongly tangent to \(\Omega^\star(M)\), since
\[
(\omega^\star+t h)(d,b)=t>0
\qquad (t>0),
\]
whereas every point of \(\Omega^\star(M)\) satisfies \(\omega(d,b)=0\). Thus \(h\) is transverse.

For \(t\in(0,1/4)\), write
\[
\omega_t:=\omega^\star+t h.
\]
Then
\[
\omega_t(d,a)=\frac14-t,
\qquad
\omega_t(d,b)=t,
\]
and therefore
\[
J_M(\omega_t)=\inf_{u\in[0,\bar u]}F_t(u),
\qquad
F_t(u):=\Bigl(\frac14-t\Bigr)K_m(u)+tL_n(u).
\]

Note also that the infinity norm distance between $\omega_t$ and $\Omega^*(M)$ is $t$ in this case. Namely $\mathrm{dist}_{\infty}(\omega_t,\Omega^*(M)) = t$. Hence, the final bounds on the curvature can be related to our definition of sharpness in Assumption (A4) of Lemma \ref{lem:corrected_selector_stability}.

\smallskip
\noindent\textbf{Step 7: asymptotic expansions of the branch costs.}
Because
\[
q_a(u)=\Bigl(\frac38+u^m,\frac14,\frac14,\frac18-u^m\Bigr),
\qquad
q_b(u)=\Bigl(\frac18+u^n,\frac14,\frac14,\frac38-u^n\Bigr),
\]
one computes
\[
K_m(u)
=
\frac14\log\frac{(1/4)^2}{(3/8+u^m)(1/8-u^m)},
\qquad
L_n(u)
=
\frac14\log\frac{(1/4)^2}{(1/8+u^n)(3/8-u^n)}.
\]
Expanding at \(u=0\) gives
\[
K_m(u)=A+\frac43 u^m + O(u^{2m}),
\qquad
L_n(u)=A-\frac43 u^n + O(u^{2n}).
\]
Hence
\[
F_t(u)
=
\frac{A}{4}
+\frac13 u^m
-\frac43 t u^n
+R_t(u),
\]
where
\[
|R_t(u)|
\le
C\bigl(tu^m+u^{2m}+t u^{2n}\bigr)
\]
for all sufficiently small \(u,t\).

\smallskip
\noindent\textbf{Step 8: localization of the minimizer.}
Set
\[
s:=\frac{1}{m-n},
\qquad
p^\sharp:=\frac{m}{m-n}.
\]
We claim that any minimizer \(u_t\) of \(F_t\) over \([0,\bar u]\) satisfies
\[
u_t=O(t^s).
\]
To see this, choose a constant \(M>0\) such that \(M^{m-n}>4\). If \(u\ge M t^s\), then
\[
\frac13 u^m-\frac43 t u^n
=
u^n\Bigl(\frac13 u^{m-n}-\frac43 t\Bigr)
\ge 0.
\]
For such \(u\), the remainder \(R_t(u)\) is of strictly smaller order than \(u^m\), so in fact
\[
F_t(u)\ge \frac{A}{4}
\]
for all sufficiently small \(t\). On the other hand, at the trial point \(u=t^s\),
\[
F_t(t^s)=\frac{A}{4}+t^{p^\sharp}\Bigl(\frac13-\frac43+o(1)\Bigr)
<\frac{A}{4}
\]
for small \(t\). Hence no minimizer can satisfy \(u_t\ge M t^s\), which proves the claim.

\smallskip
\noindent\textbf{Step 9: scaling and the sharpness exponent.}
Set \(u=t^s v\). Since \(u_t=O(t^s)\), it suffices to analyze bounded \(v\). Uniformly for \(v\) in bounded sets,
\[
F_t(t^s v)
=
\frac{A}{4}
+t^{p^\sharp}
\Bigl[
\frac13 v^m-\frac43 v^n+o(1)
\Bigr].
\]
Define
\[
\psi(v):=\frac13 v^m-\frac43 v^n.
\]
This polynomial has a unique minimizer at
\[
v_\star=\Bigl(\frac{4n}{m}\Bigr)^{1/(m-n)},
\]
and
\[
\psi(v_\star)
=
-\frac43\Bigl(1-\frac{n}{m}\Bigr)
\Bigl(\frac{4n}{m}\Bigr)^{\frac{n}{m-n}}
=:-c_{m,n}<0.
\]
Therefore,
\[
\inf_{u\in[0,\bar u]}F_t(u)
=
\frac{A}{4}
-c_{m,n}\, t^{p^\sharp}
+o(t^{p^\sharp}).
\]
Since \(J^\star=A/4\), this gives
\[
J^\star-J_M(\omega^\star+t h)
=
c_{m,n}\, t^{m/(m-n)}
+o\bigl(t^{m/(m-n)}\bigr),
\]
where
\[
c_{m,n}
=
\frac43\Bigl(1-\frac{n}{m}\Bigr)
\Bigl(\frac{4n}{m}\Bigr)^{\frac{n}{m-n}}
>0.
\]

\end{proof}

\subsection{Empirical Verification of Sharpness Results.}

\paragraph{Construction of Figure~\ref{fig:main_figure}.}
For each pair of integers $(m,n)$ shown in the figure, we instantiate the MDP family from Theorem~\ref{theorem: arbitrary_sharpness}. 
Let $\Omega^\star(M)$ denote the set of optimal allocations and let $\omega^\star\in\Omega^\star(M)$ be the reference optimal allocation used in the construction. 
We then consider the one-dimensional transverse path
\[
    \omega_t = \omega^\star + t\bigl(e_{(d,b)} - e_{(d,a)}\bigr),
\]
which moves mass from the optimal state-action pair $(d,a)$ to the suboptimal state-action pair $(d,b)$ while remaining feasible for sufficiently small $t>0$. 
Ass described in Appendix \ref{appendix_subsection_arbitrary_sharpness}For each value of $t$, we compute the objective value
\[
    J(\omega_t)
    =
    \inf_{u\in[0,1/16]}
    \sum_{s,a} \omega_t(s,a)
    \mathrm{KL}\!\left(P(\cdot\mid s,a)\,\|\,P_u(\cdot\mid s,a)\right),
\]
where $P_u$ ranges over the one-parameter family of alternatives used in the proof. 
The scalar minimization over $u$ is solved numerically using the reduced one-dimensional objective from the construction. 
We then plot the resulting gaps $J^\star - J(\omega_t)$ against $t$ on log--log axes and fit a line to each curve in log space. 
The fitted slope estimates the exponent in
\[
    J^\star - J(\omega_t)
    \asymp
    t^{m/(m-n)} ,
\]
so the reciprocal fitted slope estimates the polynomial sharpness exponent of $\varphi$, equal to $(m-n)/m$.

All experiments were conducted on a cluster equipped with NVIDIA RTX 6000 Ada Generation GPUs (48 GB VRAM each).

\newpage
\section{Other results}\label{lebesgue_measure_zero_appendix_section}

Another important property is that the set of MDPs with multiple optimal policies has null Lebesgue measure.
To prove this result, we employ  the following lemma.
\begin{lemma}{Rational structure of policy values}{}
Consider the set of tabular MDPs with fixed reward vector $r\in [0,1]^{\mathcal S\times \mathcal A}$ and $\gamma\in [0,1)$. Let $M$ be an MDP identified by $P\in \Delta(\mathcal S)^{\mathcal S\times \mathcal A}$. For any two policies $(\pi,\pi')$, for any $s\in \mathcal S$ 
the map
\[
g_M^{\pi,\pi'}(s)\coloneqq V_M^\pi(s)-V_M^{\pi'}(s),
\]
is a real rational function.
\end{lemma}
\begin{proof}
Since $V_M^\pi=(I-\gamma P_\pi)^{-1} r_\pi$, by Cramer’s rule,
$(I-\gamma P_\pi)^{-1}=\operatorname{adj}(I-\gamma P_\pi)/\det(I-\gamma P_\pi)$, with $\det(I-\gamma P_\pi)\neq 0$ since $\gamma\in [0,1)$.
Entries of $\operatorname{adj}(\cdot)$ and the determinant are polynomials in the entries of $P_\pi$,
hence $V_M^\pi$ is componentwise a ratio of polynomials in $(P,r)$ for every $s$.
\end{proof}
We also need the following useful result, showing that optimal policies do not yield constant reward vectors if $\max_a r(s,a)$ varies across states.
\begin{lemma}{Constant-reward policies cannot be optimal}{}
    Let $m(s):=\max_a r(s,a)$ and assume $\exists s,s': m(s) \neq m(s')$. If a deterministic policy $\pi$ has $r_\pi(s)\equiv c$ for all $s\in \mathcal S$, then $\pi$ is not optimal for any transition kernel $P$.
\end{lemma}
\begin{proof}
    If $r_\pi =c\mathbf{1}$ for some $c\in [0,1]$, then for any $P$ we have $V_M^\pi=\frac{c}{1-\gamma}\mathbf 1$.
    For any $s$ and $b$,
    \[
    Q_M^\pi(s,b)=r(s,b)+\gamma \frac{c}{1-\gamma},\qquad
    V_M^\pi(s)=\frac{c}{1-\gamma}.
    \]
    Optimality of $\pi$ would require $V_M^\pi(s)\geq Q_M^\pi(s,b)$ for all $(s,b)$, i.e.
    \[
    0\leq \frac{c}{1-\gamma}-r(s,b)-\gamma \frac{c}{1-\gamma} =  c -r(s,b) \Rightarrow c\geq r(s,b)
    \]
    for all $(s,b)$, implying $c\geq m(s)$ for all $s$. But $c=r_\pi(s) =r(s,\pi(s)) \leq m(s)$ for all $s$, implying  we must have $c=m(s)$ for all $s$, which is however a contradiction since there exist a pair $(s,s')$ s.t. $m(s)\neq m(s')$.
\end{proof}
Next, we show that $g^\pi$ is non-trivial under some assumptions.
\begin{lemma}{Non-triviality of $g$}\label{lem:switch-zero-iff}
    Fix deterministic $\pi$, a state $s$, and $a\neq \pi(s)$. Define
    \[
    g_{s,a}^\pi(M):=Q^\pi_M(s,\pi(s))-Q^\pi_M(s,a).
    \]
    Then $g_{s,a}^\pi\equiv 0$ on $\Delta(\mathcal S)^{\mathcal S\times \mathcal A}$ if and only if $r(s,\pi(s))=r(s,a)$ and $r_\pi$ is constant across states.
\end{lemma}
\begin{proof}
    Write 
     \[g_{s,a}^\pi(M)= [r(s,\pi(s))-r(s,a)] + \gamma(p-q)^\top V^\pi(P)\]
      for some MDP with transition $P$ and we write
    $p=P(\cdot|s,\pi(s))$, $q=P(\cdot|s,a)$.
    
    \paragraph{Sufficiency.} If $r(s,\pi(s))=r(s,a)$ and $r_\pi=c\mathbf{1}$ for some $c\in [0,1]$, then
    \[
    g_{s,a}^\pi(M)= c\gamma(p-q)^\top (I-\gamma P_\pi)^{-1} \mathbf{1}
    \]
    Since $\mathbf{1}$ is a right eigenvector of $(I-\gamma P_\pi)^{-1}$ satisfying $(I-\gamma P_\pi)^{-1}\mathbf{1}=\frac{1}{1-\gamma}\mathbf{1}$, we have that
    \[
    g_{s,a}^\pi(M)= c\gamma(p-q)^\top \mathbf{1}=0.
    \]
    which proves sufficiency.
    
    \paragraph{Necessity.} Suppose $g_{s,a}^\pi(M)\equiv 0$ on $P\in \Delta(\mathcal S)^{\mathcal S\times \mathcal A}$. For each $P$ we obtain $P_\pi$. Treat $P_\pi$ as a fixed variable, while $q$ is a degree of freedom. Then, we must have that
    \[0 = \sup_q \left| [r(s,\pi(s))-r(s,a)] + \gamma(p-q)^\top V^\pi(P)\right|.\]
    Since for any $q$ we obtain $0$, consider $q=e_i$ and $q=e_j$ for $i,j$ taking value in an enumeration of $\mathcal S$. Then
    \[
     [r(s,\pi(s))-r(s,a)] + \gamma(p-e_i)^\top V^\pi(P)= [r(s,\pi(s))-r(s,a)] + \gamma(p-e_j)^\top V^\pi(P)
    \]
    implying that $(e_j-e_i)^\top V^\pi(P)=0$. Since $i,j$ are allowed to vary, this implies that $V_P^\pi = c\mathbf{1}$ for some $c\in \mathbb{R}$.
    Therefore
     \[g_{s,a}^\pi(M)= [r(s,\pi(s))-r(s,a)] + \gamma(p-q)^\top c\mathbf{1}= r(s,\pi(s))-r(s,a),\]
     implying that $r(s,\pi(s))=r(s,a)$ if $g_{s,a}^\pi(M)= 0$.
     
     Now, for any fixed $q$ consider varying $P$ so that $P_\pi$ varies. Since $V_P^\pi = c\mathbf{1}=(I-\gamma P_\pi)^{-1}r_\pi$, and $\mathbf{1}$ is the right eigenvector of $(I-\gamma P_\pi)^{-1}$, we have that $r_\pi$ itself is proportional to $\mathbf{1}$. Thus we conclude that $r_\pi$ must be constant across states.
\end{proof}
\begin{lemma}{Switch hypersurfaces have null measure}{}
    Let $m(s)=\max_a r(s,a)$, $Z_{s,a}^\pi\coloneqq\{P\in \Delta(\mathcal S)^{\mathcal S\times \mathcal A}: g_{s,a}^\pi(P)=0\}$ and 
    $\mathcal U^\pi\coloneqq \{P\in \Delta(\mathcal S)^{\mathcal S\times \mathcal A}: \pi \text{ is optimal at } (P,r,\gamma)\}$.
    
 If $\exists s,s': m(s)\neq m(s')$ , then for every deterministic $\pi$ and every $a\neq \pi(s)$, the set
    $\mathcal U_\pi\cap Z_{s,a}^\pi$ has Lebesgue measure $0$.
\end{lemma}
\begin{proof}
    By the previous lemma we know that $g_{s,a}^\pi$ is a real rational function, thus real analytic.
    For a nontrivial real-analytic function, its zero set has Lebesgue measure zero.
    
    If $g_{s,a}^\pi\equiv 0$, then by the previous lemma we have that $r_\pi$ is constant. But optimal policies do not admit constant $r_\pi$ if there exists a pair $(s,s')$ such that $m(s)\neq m(s')$, so $\mathcal U^\pi=\varnothing$ and 
    $\mathcal U^\pi\cap Z_{s,a}^\pi=\varnothing$ (null).
    Alternatively, if $g_{s,a}^\pi\not\equiv 0$, then $Z_{s,a}^\pi$ is the zero set of a nontrivial real-analytic function, hence it has Lebesgue measure $0$; intersecting with the measurable set $\mathcal U^\pi$ preserves nullity.
\end{proof}
\begin{lemma}[Policy-switch set has Lebesgue measure $0$]\label{prop:measure0-A1}
    Let $m(s)=\max_a r(s,a)$. If $\exists s,s': m(s)\neq m(s')$  then the set
    \[
    \mathcal T:=\{P\in \Delta(\mathcal S)^{\mathcal S\times \mathcal A}: \exists s,\ a\neq b\ \text{ with } Q^\star_M(s,a)=Q^\star_M(s,b)=V^\star_M(s)\}
    \]
     has Lebesgue measure $0$.
\end{lemma}
\begin{proof}
    Optimal policies are w.l.o.g. deterministic. If $P\in\mathcal T$, there exists an optimal deterministic $\pi^\star$, a state $s$, and $a\neq \pi^\star(s)$ with $Q^{\pi^\star}_M(s,\pi^\star(s))=Q^{\pi^\star}_M(s,a)=V_M^\star(s)$, hence
    $P\in \mathcal U^{\pi^\star}\cap Z_{s,a}^{\pi^\star}$.
    Therefore
    \[
    \mathcal T \subseteq \bigcup_{\pi} \bigcup_{s} \bigcup_{a\neq \pi(s)} \left(\mathcal U^\pi\cap Z_{s,a}^\pi \right).
    \]
    Since we have a countable union of $0$-measure sets, we have that $\mathcal T$ has measure $0$.
\end{proof}

\end{document}

%% file: pramble.tex
\DeclareMathOperator*{\argmax}{arg\,max}
\DeclareMathOperator*{\argmin}{arg\,min}
\DeclareMathOperator*{\arginf}{arg\,inf}

\newcommand{\KL}{{\rm KL}}

\newcommand{\statespace}{{\cal S}}
\newcommand{\actionspace}{{\cal A}}

\newcommand{\refalgbyname}[2]{\hyperref[#1]{\texttt{\textbf{#2}}}}
\newcommand{\nas}{\refalgbyname{algo:nas}{NaS}}

\newcommand{\calm}{{\cal M}}
\newcommand{\calmalt}{{\cal M}'}
\newcommand{\alt}{{\rm Alt}}

\newtheorem{assumption}{Assumption}
\newtheorem{theorem}{Theorem}
\newtheorem{lemma}[theorem]{Lemma}
\newtheorem{proposition}[theorem]{Proposition}
\newtheorem{definition}[theorem]{Definition}
\newtheorem{remark}[theorem]{Remark}
\newtheorem{corollary}[theorem]{Corollary}